\documentclass[twoside,11pt]{article}

\usepackage{amsmath}
\usepackage[preprint]{jmlr2e}
\usepackage{bm}
\usepackage{algorithm}
\usepackage{algpseudocode}
\usepackage{titletoc}

\hypersetup{hidelinks}

\newcommand{\stchlm}{\mathrm{STCH}^{\mu}_{\boldsymbol{\lambda}}}
\DeclareMathOperator*{\argmin}{arg\,min}

\newtheorem{assumption}[theorem]{Assumption}

\newfloat{subroutine}{tbp}{los}
\floatname{subroutine}{Subroutine}

\newcommand{\bw}{\bm{w}}
\DeclareMathOperator*{\argmax}{arg\,max}
\DeclareMathOperator*{\minimize}{minimize}

\usepackage{lastpage}
\jmlrheading{}{}{}{ Revised}{}{}{2026 Chenxi Li, Yao Shu, Siyang Gao, Shuang Qiu, Guangxin Jiang}

\ShortHeadings{Safe Multi-Objective Reinforcement Learning}{Li, Shu, Gao, Qiu and Jiang}
\firstpageno{1}

\begin{document}

\title{Safe Multi-Objective Reinforcement Learning: Zero Constraint Violation and Primal-Averaging Convergence}

\author{\name Chenxi Li \email ling112358@gmail.com \\
\addr 
Harbin Institute of Technology\\
Harbin, Heilongjiang, China
\AND
\name Yao Shu \email yaoshu@hkust-gz.edu.cn \\
\addr 
The Hong Kong University of Science and Technology (Guangzhou)\\
Guangzhou, Guangdong, China
\AND
\name Siyang Gao \email siyangao@cityu.edu.hk \\
\addr 
City University of Hong Kong\\
Hong Kong, China
\AND
\name Shuang Qiu$^\dagger$ \email shuanqiu@cityu.edu.hk \\
\addr 
City University of Hong Kong\\
Hong Kong, China
\AND
\name Guangxin Jiang$^\dagger$ \email gxjiang@hit.edu.cn \\
\addr 
Harbin Institute of Technology\\
Harbin, Heilongjiang, China
}

\editor{}

\maketitle

\begingroup
\renewcommand{\thefootnote}{\textdagger}
\footnotetext[0]{Corresponding authors}
\endgroup

\begin{abstract}
Many applications of multi-objective reinforcement learning (MORL) require policies that balance conflicting objectives while satisfying safety constraints.
We study safe MORL with zero constraint violation throughout learning, covering both cumulative and stepwise constraints.
A key challenge is to explore sufficiently under unknown dynamics and constraint costs without violating the constraints.
For a given preference over objectives, we develop an algorithm to learn a safe Pareto optimal policy by combining optimistic reward estimates with pessimistic cost bounds.
With high probability, every policy executed over $T$ episodes incurs zero constraint violation.
The algorithm returns a safe Markov policy induced by the averaged occupancy measure and achieves an $\widetilde{\mathcal O}\!(1/\sqrt T)$ suboptimality gap under smooth Tchebycheff scalarization.
We further develop a preference-free algorithm that first explores safely without a prespecified preference and then computes a safe policy for any given preference without further interaction.
Over $T$ exploration episodes and $K$ planning iterations, it retains the same safety guarantee and achieves an $\widetilde{\mathcal O}\!(1/\sqrt T+1/\sqrt K)$ suboptimality gap.
Both algorithms require knowledge of a policy that strictly satisfies the constraints and its costs.
We then relax these requirements, incurring at most constant cumulative violation.
To our knowledge, we provide the first sample-efficient MORL algorithms that guarantee safety for every executed policy.
\end{abstract}


\section{Introduction}

Reinforcement learning (RL) studies sequential decision-making through interaction with an environment to maximize cumulative rewards~\citep{sutton1998reinforcement}.
In many real-world applications, however, the learner must simultaneously optimize multiple, potentially conflicting objectives such as efficiency, energy consumption, and fairness.
This motivates \emph{multi-objective reinforcement learning} (MORL)~\citep{roijers2013survey,liu2014multiobjective}, where the learner seeks policies that perform well under different trade-offs among objectives.
Since a single policy that maximizes all objectives rarely exists, MORL commonly targets \emph{Pareto optimal} policies, i.e., policies for which no objective can be improved without degrading another~\citep{van2014multi,qiu2024traversing}.
For a given preference over the objectives, the goal is to learn a Pareto optimal policy that reflects the desired trade-off.
In \emph{preference-based} learning, the preference is fixed before exploration.
Learning under a different preference may require additional environment interaction.
\emph{Preference-free} learning removes this need by exploring the environment once without a prespecified preference and then computing a policy for any given preference from the collected data~\citep{wu2021accommodating}.

In safety-critical applications of MORL, the learner must balance multiple objectives while satisfying explicit safety requirements, such as avoiding collisions in autonomous driving, respecting resource and actuation budgets in robotics, or keeping operating temperatures within prescribed limits in industrial control~\citep{he2023toward,tang2023constrained}.
Existing safe MORL methods can be broadly classified into two categories.
The first treats safety as an additional objective, or a set of objectives, augmenting the reward vector with costs and learning trade-offs via scalarization or preference learning~\citep{huang2022constrained}.
This strategy does not in general guarantee safety, since whether the learned policy satisfies the safety requirements depends on the chosen preference or scalarization~\citep{achiam2017constrained}.
The second category imposes safety as explicit constraints while optimizing multi-objective performance, leading to \emph{constrained MORL} frameworks~\citep{kim2025conflict,park2026constrained}.
Under cumulative safety, these constraints limit the expected total cost over an episode, whereas under stepwise safety, they limit the cost at each step.
A \emph{safe Pareto optimal policy} is Pareto optimal within the class of policies satisfying these constraints.
The objective value vectors of such policies form the \emph{safe Pareto front}.
However, existing constrained MORL methods may still violate the constraints during learning~\citep{gu2025safe}.
In constrained RL, prior work has established sublinear or constant bounds on cumulative constraint violation~\citep{qiu2020upper,liu2021learning,stradi2025learning}.
Some methods further assume knowledge of a strictly safe policy, which satisfies every constraint with strict inequality.
Under this assumption, they ensure that every executed policy satisfies the constraints with high probability, yielding zero constraint violation throughout learning~\citep{liu2021learning,bura2022dope,stradi2025learning}.
However, the literature on constrained MORL still lacks a unified, sample-efficient framework that guarantees zero constraint violation throughout learning under either cumulative or stepwise safety constraints.
This leads to the following question:

{\centering\emph{
Can we design provably sample-efficient algorithms for safe MORL that traverse the safe Pareto front while guaranteeing zero constraint violation?} \par}

This problem involves three main challenges.
First, safety requirements conflict with the need for exploration when costs and dynamics are unknown.
The learner may not know whether an exploratory policy is safe before execution, yet every policy executed during learning must satisfy the safety constraints.
If the learner executes only policies already known to be safe, the resulting data may be insufficient to identify safe Pareto optimal policies.
Second, the set of safe Pareto optimal policies cannot generally be obtained by removing unsafe policies from the unconstrained Pareto optimal set.
A policy can be Pareto optimal among safe policies even when it is dominated by an unsafe policy.
To obtain a safe Pareto optimal policy for a given preference, we adopt the smooth Tchebycheff scalarization~\citep{lin2024smooth,qiu2024traversing}, which measures each objective relative to its optimal value over the safe policy class.
These optimal values are unknown, and estimating them requires solving an auxiliary constrained RL problem for each objective while maintaining safety throughout learning.
Finally, it is even more challenging to explore the environment only once and then compute a safe Pareto optimal policy for any subsequently specified preference.
Since no preference is available during exploration, the collected data must suffice for every preference.
At the same time, these data should be collected only by safe policies, which restrict the state-action pairs that the learner can visit.

To address the above challenges, we propose \texttt{CSTRL} (Constrained Smooth Tchebycheff Reinforcement Learning) for preference-based learning.
Its mirror-ascent updates use optimistic reward estimates to encourage exploration, while pessimistic cost bounds are used to maintain safety.
We further develop \texttt{PF-CSTRL} (Preference-Free \texttt{CSTRL}), which uses reward and transition estimation uncertainty to guide safe exploration and collect data for planning under any subsequently specified preference.
Both algorithms require prior knowledge of a strictly safe policy and its constraint costs.
We also develop extensions that relax these requirements.

\vspace{5pt}
\noindent\textbf{Contributions.} 
Our main contributions are as follows.
\textbf{(i)} We propose \texttt{CSTRL} and \texttt{PF-CSTRL} for preference-based and preference-free safe MORL, respectively.
To our knowledge, these are the first sample-efficient MORL algorithms that guarantee zero constraint violation for every policy executed during learning.
Both algorithms combine optimistic mirror-ascent updates of occupancy measures with pessimistic cost bounds under either cumulative or stepwise safety constraints.
\textbf{(ii)} We establish that, over \(T\) learning episodes, \texttt{CSTRL} guarantees zero constraint violation with high probability under either safety formulation for every executed policy.
It returns a safe Markov policy induced by the \emph{primal average} of the occupancy measures of these policies and achieves a suboptimality gap of $\widetilde{\mathcal{O}}(1/\sqrt{T})$ under smooth Tchebycheff scalarization.
Using \(T\) exploration episodes followed by \(K\) planning iterations, \texttt{PF-CSTRL} maintains the same safety guarantee throughout both stages and achieves a suboptimality gap of $\widetilde{\mathcal{O}}(1/\sqrt{T}+1/\sqrt{K})$.
\textbf{(iii)} We extend both algorithms under weaker assumptions about prior knowledge.
When a strictly safe policy is known but its costs are unknown, we estimate its costs conservatively while preserving zero constraint violation.
When both the policy and its costs are unknown, we learn a strictly safe policy and a conservative estimate of its costs in an estimation phase and incur only constant cumulative violation.

\vspace{5pt}
\noindent\textbf{Related Work.}
MORL studies sequential decision-making problems with multiple, potentially conflicting objectives and seeks policies that represent different trade-offs among them~\citep{roijers2013survey,liu2014multiobjective,hayes2022practical}.
A substantial line of work characterizes the Pareto front using scalarization-based formulations, including linear and Tchebycheff scalarizations, and develops methods for learning Pareto optimal policy sets~\citep{roijers2013survey,van2014multi,hayes2022practical,kubo2026single}.
More recent work studies preference-based methods that learn policies for given preferences and preference-free methods that compute policies for any preference specified after exploration~\citep{yang2019generalized,qiu2024traversing,hu2026near}.
However, most existing MORL approaches are developed for unconstrained environments and therefore do not provide explicit safety guarantees~\citep{hayes2022practical,qiu2024traversing}.

Constrained reinforcement learning studies sequential decision-making under explicit safety or resource constraints, most commonly formalized as constrained Markov decision processes~\citep{altman2021constrained}.
Classical approaches enforce expected cumulative constraints over an episode via Lagrangian or primal-dual methods, providing asymptotic or averaged feasibility guarantees but allowing violations during learning~\citep{achiam2017constrained,tessler2018reward,wei2020online}.
To obtain stricter guarantees, recent works adopt conservative or pessimistic designs that achieve zero or bounded cumulative constraint violation, often by tightening constraints or assuming access to a strictly safe policy~\citep{liu2021learning,vaswani2022near,bai2023achieving,ni2026provably}.
Other works incorporate optimism in the face of uncertainty to enable efficient exploration under safety constraints~\citep{qiu2020upper,liu2021learning,ding2023provably,wendl2026safe}, and further investigate stepwise safety~\citep{xiong2023provably,zhao2024absolute}.
However, existing methods are largely limited to single-objective settings and do not address the challenges of multi-objective or preference-free learning.

\section{Preliminaries}
\label{sec:Pre}

We consider a finite-horizon constrained multi-objective Markov decision process.
It is defined by the tuple $(\mathcal{S}, \mathcal{A}, H, P, \boldsymbol{r}, \boldsymbol{c}, \rho)$, where $\mathcal{S}$ and $\mathcal{A}$ are the finite state and action spaces. 
The horizon $H$ is the number of steps in each episode.
The transition kernel is $P:=\{P_h\}_{h=1}^H$ with $P_h:\mathcal{S}\times\mathcal{A}\times\mathcal{S}\mapsto[0,1]$, and $P_h(s'|s,a)$ denotes the probability of transitioning to state $s'\in\mathcal{S}$ from state $s\in\mathcal{S}$ after taking action $a\in\mathcal{A}$ at step $h\in[H]$.
The reward function is vector-valued with $m$ objectives, denoted by $\boldsymbol{r}:=(r_1,\dots,r_m)$, where $r_i:=\{r_{i,h}\}_{h=1}^H$ for every $i\in[m]$.
Each $r_{i,h}:\mathcal{S}\times\mathcal{A}\mapsto[0,1]$ maps a state-action pair to the reward for the $i$-th objective at step $h$.
Moreover, the constraint cost function is vector-valued with $n$ constraints, denoted by $\boldsymbol{c}:=(c_1,\dots,c_n)$.
We further define $c_j:=\{c_{j,h}\}_{h=1}^H$, where each $c_{j,h}:\mathcal{S}\times\mathcal{A}\mapsto[0,1]$ maps a state-action pair to the cost for the $j$-th constraint at step $h$.
The initial state distribution is denoted by $\rho$.
The learner follows a policy $\pi:=\{\pi_h\}_{h=1}^H$, where $\pi_h(a|s)$ represents the probability of taking action $a\in\mathcal A$ in state $s\in\mathcal S$ at step $h\in[H]$.
Let $\Pi$ denote the class of all stochastic Markov policies.
For any $\pi\in\Pi$, we define the value of $\pi$ for the $i$-th objective as
$V_i^r(\pi)
:=
\mathbb{E}[
\sum_{h=1}^H r_{i,h}(s_h,a_h)|
\rho,\pi,P]$ and the value of $\pi$ for the $j$-th constraint as $V_j^c(\pi)
:=
\mathbb{E}[\sum_{h=1}^H c_{j,h}(s_h,a_h)|\rho,\pi,P]$.

In each episode $t\in[T]$, the learner follows the policy $\pi^t$ and generates a trajectory $\{(s_h^t,a_h^t)\}_{h=1}^H$.
At each step $h\in[H]$, the learner observes a stochastic reward $r_{i,h}^t\in[0,1]$ for each objective $i\in[m]$ and a stochastic constraint cost $c_{j,h}^t\in[0,1]$ for each constraint $j\in[n]$.
These observations satisfy $\mathbb{E}[r_{i,h}^t|s_h^t,a_h^t]=r_{i,h}(s_h^t,a_h^t)$ and $\mathbb{E}[c_{j,h}^t|s_h^t,a_h^t] =c_{j,h}(s_h^t,a_h^t)$.
The reward function $\boldsymbol{r}$, the constraint cost function $\boldsymbol{c}$, and the transition kernel $P$ are unknown. 
Under such a scenario, we aim to solve the following constrained optimization problem:
\begin{align*}
\max_{\pi\in\Pi}&~~ V_1^r(\pi),~V_2^r(\pi),\cdots, ~V_m^r(\pi)\\    
\text{s.t.} &~~V_j^c(\pi) \leq \tau_j, \quad \forall j\in[n], 
\end{align*}
where $\tau_j$ is the threshold for constraint $j$, and $\boldsymbol\tau:=(\tau_1,\ldots,\tau_n)$ collects these thresholds.

\vspace{5pt}
\noindent\textbf{Occupancy Measure.} 
In this work, we employ the occupancy measure, i.e., the joint distribution over the state, action, and next state, to formulate the optimization problem. 
Formally, an occupancy measure is defined as follows~\citep{rosenberg2019onlineb}:
\begin{definition}
\label{def:occupancy-measure}
Letting $\Delta_{|\mathcal{S}|\times|\mathcal{A}|\times|\mathcal{S}|}$ denote the probability simplex over $\mathcal{S}\times\mathcal{A}\times\mathcal{S}$, we say that $\theta:=\{\theta_h\}_{h=1}^H$ is an occupancy measure for an episodic MDP with transition kernel $P$ and initial state distribution
$\rho$ if $\theta_h\in
\Delta_{|\mathcal{S}|\times|\mathcal{A}|\times|\mathcal{S}|}$ and 
\[
\left\{
\begin{aligned}
&\sum_{a\in\mathcal{A}}\sum_{s'\in\mathcal{S}}
\theta_1(s,a,s')
=
\rho(s),
&&\forall s\in\mathcal{S},
\\
&\sum_{a\in\mathcal{A}}\sum_{s'\in\mathcal{S}}
\theta_h(s,a,s')
=
\sum_{s'\in\mathcal{S}}\sum_{a\in\mathcal{A}}
\theta_{h-1}(s',a,s),
&&\forall h\in[2,H],\ \forall s\in\mathcal{S},
\\
&\theta_h(s,a,s')
=
P_h(s'|s,a)\sum_{s''\in\mathcal S}\theta_h(s,a,s''),
&&\forall h\in[H],\ \forall (s,a,s')\in\mathcal S \times \mathcal A \times \mathcal S.
\end{aligned}
\right.
\]
\end{definition}
With a slight abuse of notation, we define $\theta_h(s,a):=\sum_{s'\in\mathcal S}\theta_h(s,a,s')$.
The policy induced by $\theta$ is given by
$\pi^\theta_h(a|s)
=
{\theta_h(s,a)}/
{\sum_{a\in\mathcal A}\theta_h(s,a)}.$
For a specific transition kernel \(P\) and policy \(\pi\), we denote the corresponding occupancy measure by \(\theta^{P,\pi}=\{\theta^{P,\pi}_h\}_{h=1}^H\). 
Then, we have $\theta^{P,\pi}_h(s,a,s') := \Pr(s_h=s, a_h=a, s_{h+1}=s'|\rho,\pi,P)$ and $\theta^{P,\pi}_h(s,a) := \Pr(s_h=s, a_h=a|\rho,\pi,P)$. 
Moreover, by the definition of the occupancy measure, we can reconstruct both the transition kernel and the policy as $P_h(s' | s,a) = \theta_h^{P,\pi}(s,a,s') / \theta_h^{P,\pi}(s,a)$ and $\pi_h(a | s) = \theta_h^{P,\pi}(s,a)\big/ \sum_{a\in\mathcal A}\theta_h^{P,\pi}(s,a)$.

For any occupancy measure $\theta$, overloading the notation $V$, the corresponding values for rewards and constraints are written as
\begin{align*}
V_i^r(\theta):=
\langle r_i,\theta\rangle = 
\sum_{h,s,a} r_{i,h}(s,a)\theta_h(s,a),
\quad V_j^c(\theta):=
\langle c_j,\theta\rangle = 
\sum_{h,s,a} c_{j,h}(s,a)\theta_h(s,a).
\end{align*}
Then, we have $V_i^r(\theta^{P,\pi}) = V_i^r(\pi)$ and $V_j^c(\theta^{P,\pi}) = V_j^c(\pi)$.

\vspace{5pt}
\noindent\textbf{Safety Constraints and Safe Policies.}
To formulate the safety constraints, we collect the constraint cost functions defined above into a matrix \(C\in[0,1]^{H|\mathcal S||\mathcal A|\times n}\), whose \(j\)-th column is \((c_{j,h}(s,a))_{h\in[H],\,(s,a)\in\mathcal S\times\mathcal A}\).
We use \(d\in\{1,H\}\) to denote the number of time steps to which each safety constraint applies.
For each constraint \(j\in[n]\), the learner knows the \(d\) time steps to which the constraint applies, and we set \(c_{j,h}\equiv 0\) at all other steps \(h\in[H]\).
When \(d=1\), each constraint applies to a single time step, which gives stepwise safety.
When \(d=H\), each constraint applies to all \(H\) steps, which gives cumulative safety.
Accordingly, we have \(\boldsymbol\tau\in[0,d]^n\).
In both cases, a policy \(\pi\in\Pi\) is \emph{safe} if \(C^\top\theta^{P,\pi}\le\boldsymbol\tau\).
The corresponding safe policy class and occupancy measure set are \(\Pi_{\mathrm{safe}}:=\{\pi\in\Pi:C^\top\theta^{P,\pi}\le\boldsymbol\tau\}\) and \(\Theta_{\mathrm{safe}}:=\{\theta^{P,\pi}:\pi\in\Pi_{\mathrm{safe}}\}\).
We adopt the following standard Slater condition, which requires that the constraints can be satisfied strictly.

\begin{assumption}[Slater condition]
\label{ass:slater}
There exists a strictly safe policy $\pi^\dagger\in\Pi$ such that $C^\top\theta^{P,\pi^\dagger} < \boldsymbol\tau$.
\end{assumption}

Although Assumption~\ref{ass:slater} ensures that such a policy exists, this fact alone does not guarantee safety during learning when the policy is unknown to the learner.

\begin{assumption}
\label{ass:known}
The policy $\pi^\dagger$ in Assumption~\ref{ass:slater} is known to the learner.
Moreover, the learner knows its constraint costs
$\boldsymbol{\gamma}
:=
C^\top\theta^{P,\pi^\dagger},$
which satisfy \(0\le\gamma_j<\tau_j\le d\) for all \(j\in[n]\).
\end{assumption}

Assumption~\ref{ass:known} is standard in the constrained RL literature when zero constraint violation throughout learning is required~\citep{liu2021learning,stradi2025learning}.
In practice, a strictly safe policy is often available~\citep{huang2023safe}.
For example, in autonomous driving, a conservative rule-based controller that prioritizes collision avoidance and enforces strict speed limits provides a natural candidate for $\pi^\dagger$~\citep{likmeta2020combining}.

\vspace{5pt}
\noindent\textbf{Pareto Optimality.} 
In unconstrained MORL, the standard goal is to identify the Pareto optimal policy set, that is, the collection of policies for which no objective can be improved without degrading another~\citep{qiu2024traversing}. 
When explicit safety constraints are imposed, policies that are Pareto optimal in the unconstrained setting may violate them. 
We therefore consider Pareto optimality relative to the safe policy class.

\begin{definition}
[Safe Pareto optimal policy] 
\label{def:pareto}
For any two policies \(\pi \in \Pi_{\mathrm{safe}}\) and \(\pi' \in \Pi_{\mathrm{safe}}\), we say that \(\pi'\) dominates \(\pi\) if and only if $V_i^r(\pi')\geq V_i^r(\pi)$ for all $i \in [m]$ and there exists at least one \(i'\in[m]\) such that $V_{i'}^r(\pi')>V_{i'}^r(\pi).$
A policy \(\pi \in \Pi_{\mathrm{safe}}\) is a safe Pareto optimal policy if and only if no other policy \(\pi' \in \Pi_{\mathrm{safe}}\) dominates \(\pi\).
We denote the set of safe Pareto optimal policies by $\Pi_{\mathrm{safe}}^*$.
\end{definition}

\section{Scalarization of Multiple Objectives}
\label{sec:sca}
To characterize the trade-offs achieved by safe Pareto optimal policies, we seek a scalarization framework that covers the entire safe Pareto front and allows the learner to select a safe Pareto optimal policy according to a given preference.
Since this front may contain infinitely many value vectors, direct enumeration is generally infeasible.
We therefore consider a family of scalar optimization problems parameterized by a preference vector $\boldsymbol{\lambda}\in\Delta_m$.
Here, $\Delta_m$ is the probability simplex in $\mathbb{R}^m$, and $\Delta_m^o:=\{\boldsymbol{\lambda}\in\Delta_m: \lambda_i>0,\ \forall i\in[m]\}$ denotes its relative interior.
In unconstrained MORL, the smooth Tchebycheff scalarization~\citep{lin2024smooth,qiu2024traversing} has been shown to provide both \emph{controllability} with respect to the preference vector $\boldsymbol{\lambda}$ and \emph{full coverage} of the Pareto front.
In our safe MORL setting, however, the policy space is restricted to the safe class \(\Pi_{\mathrm{safe}}\), and we adapt the scalarization to this class as follows.
For notational simplicity, let $V_i(\cdot):=V_i^r(\cdot)$ and $\boldsymbol V(\theta):=(V_1(\theta),\ldots,V_m(\theta))$.

\begin{definition}
[Safe smooth Tchebycheff scalarization]
\label{def:stch}
Let $\boldsymbol{\lambda}:=(\lambda_1,\ldots,\lambda_m)\in\Delta_m$ be a preference vector,
let $\mu>0$ be a smoothing parameter,
and let $\iota\in(0,1]$ be a prespecified regularization parameter.
For each $i\in[m]$, we define $V_i^*:=\max_{\theta\in\Theta_{\mathrm{safe}}}V_i(\theta)$.
The safe smooth Tchebycheff scalarization of a policy $\pi\in\Pi_{\mathrm{safe}}$ is
\[
\stchlm(\pi)
:=
\mu \log\left[
\sum_{i=1}^{m}
\exp\left(
\frac{\lambda_i\bigl(V_i^*+\iota-V_i(\theta^{P,\pi})\bigr)}{\mu}
\right)
\right].
\]
\end{definition}

We first establish that the safe value set is a nonempty convex polytope.

\begin{proposition}
\label{prop:safe_value_polytope}
Under Assumption~\ref{ass:slater}, the safe value set
$\mathcal V_{\mathrm{safe}}
:=
\{
\boldsymbol V(\theta):
\theta\in\Theta_{\mathrm{safe}}\}$
is a nonempty convex polytope.
\end{proposition}

The safe smooth Tchebycheff scalarization retains both controllability and full coverage on the safe policy class.

\begin{proposition}
\label{prop:stch}
Let $\mu>0$. Under Assumption~\ref{ass:slater}, we have
\[\bigcup_{\boldsymbol{\lambda}\in\Delta_m^o}
\argmin_{\pi\in\Pi_{\mathrm{safe}}}
\stchlm(\pi)
=
\Pi_{\mathrm{safe}}^*.\]
\end{proposition}

The proof of Proposition~\ref{prop:stch} uses the fact that $\mathcal V_{\mathrm{safe}}$ is a nonempty convex polytope, as established in Proposition~\ref{prop:safe_value_polytope}.
This convexity relies on $\Pi$ containing all stochastic Markov policies.
The analysis for deterministic policy classes differs substantially, and we leave this extension to future work.
We next present a proposition that further characterizes the safe Pareto optimal policy set.

\begin{proposition}
\label{prop:stch-unique}
Let $\mu>0$ and $\boldsymbol{\lambda}\in\Delta_m^o$.
Under Assumption~\ref{ass:slater}, all minimizers of $\stchlm(\pi)$ over $\Pi_{\mathrm{safe}}$ achieve the same values on all objectives.
\end{proposition}

Together with Proposition~\ref{prop:stch}, this result implies that solving $\min_{\pi\in\Pi_{\mathrm{safe}}} \stchlm(\pi)$ once for each $\boldsymbol{\lambda}\in\Delta_m^o$ suffices to recover every value vector on the safe Pareto front.
For the safe policy class, the corresponding scalarized optimization problem is
\begin{align}
\min_{\pi\in\Pi_{\mathrm{safe}}} \stchlm(\pi).
\label{eq:constrained_stch}
\end{align}
Following~\citet{qiu2024traversing}, the safe smooth Tchebycheff scalarization admits the equivalent reformulation
\begin{align}
\label{eq:eq-stch}
\stchlm(\pi)
&=
\max_{\boldsymbol w\in\Delta_m}
\max_{\{\theta_i\in\Theta_{\mathrm{safe}}\}_{i=1}^m}
\sum_{i=1}^{m}
\left[
w_i\lambda_i
\bigl(V_i(\theta_i)+\iota-V_i(\theta^{P,\pi})\bigr)
-\mu w_i\log w_i
\right].
\end{align}
Consequently, problem~\eqref{eq:constrained_stch} admits an equivalent min-max-max formulation.
This reformulation is well suited to the design of algorithms based on upper confidence bounds.

\vspace{5pt}
\noindent\textbf{Performance Metric.} 
For a given preference vector $\boldsymbol{\lambda}\in\Delta_m$ and smoothing parameter $\mu>0$, we define an optimal safe policy as
$\pi_{\mu,\boldsymbol{\lambda}}^*
\in
\argmin_{\pi\in\Pi_{\mathrm{safe}}}
\stchlm(\pi)$.
The performance of a safe policy $\pi\in\Pi_{\mathrm{safe}}$ is then measured by the suboptimality gap
\[
\mathrm{SubOpt}(\pi)
:=
\stchlm(\pi) - \stchlm(\pi_{\mu,\boldsymbol{\lambda}}^*).
\]
Note that $\mathrm{SubOpt}(\pi)\geq 0$ for all $\pi\in\Pi_{\mathrm{safe}}$, with equality if and only if $\pi$ is optimal under the smooth Tchebycheff scalarization. 

We also assess the safety of any policy $\pi\in\Pi$ by its constraint violation
\[
\mathrm{Violation}(\pi)
:=
\max_{j\in[n]}
\Big[
\bigl(C^\top\theta^{P,\pi}\bigr)_j-\tau_j
\Big]^+,
\]
where $[x]^+:=\max\{0,x\}$.
For stepwise safety, $\bigl(C^\top\theta^{P,\pi}\bigr)_j$ gives the cost of constraint $j$ at the single step to which it applies.
For cumulative safety, it aggregates the costs of constraint $j$ across all $H$ steps.
In both cases, $\pi\in\Pi_{\mathrm{safe}}$ if and only if $\mathrm{Violation}(\pi)=0$.

Our goal is to design a sample-efficient algorithm that outputs a policy $\widehat{\pi}$ minimizing $\mathrm{SubOpt}(\widehat{\pi})$ while guaranteeing $\mathrm{Violation}(\widehat{\pi})=0$ and $\mathrm{Violation}(\pi)=0$ for every policy $\pi$ executed during learning.

\section{Safe Multi-Objective RL}
\label{sec:cstrl}
We introduce \texttt{CSTRL}, our preference-based algorithm to solve~\eqref{eq:constrained_stch} under a fixed preference vector $\boldsymbol{\lambda}\in\Delta_m$.
\texttt{CSTRL} combines optimistic exploration with pessimistic cost bounds to ensure safety with high probability.
The full procedure is summarized in Algorithm~\ref{alg:CSTRL}.
In episode $t$, Algorithm~\ref{alg:CSTRL} collects one trajectory under the main policy $\pi^t$ and one under each auxiliary policy $\nu_i^t$ for $i\in[m]$, yielding $m+1$ trajectories.
The main trajectory is used to update the main empirical estimates, while the trajectory generated by $\nu_i^t$ is used only to update the $i$-th auxiliary empirical estimates.
The counts are
\begin{equation}
\label{eq:rollout-counts}
\begin{aligned}
N_h^t(s,a)
&:=
\sum_{k=1}^{t}
\mathbf{1}_{\{(s_h^k,a_h^k)=(s,a)\}},
\quad
N_h^t(s,a,s')
:=
\sum_{k=1}^{t}
\mathbf{1}_{\{(s_h^k,a_h^k,s_{h+1}^k)=(s,a,s')\}}.
\end{aligned}
\end{equation}
Here, $\mathbf 1_{\{\mathcal E\}}$ is the indicator of event $\mathcal E$.
For each objective $i\in[m]$ and each constraint $j\in[n]$, the reward and cost estimators are
\begin{equation}
\widehat r^{\,t}_{i,h}(s,a)
:=
\frac{
\displaystyle
\sum_{k=1}^{t}
r^k_{i,h}\,
\mathbf{1}_{\{(s_h^k,a_h^k)=(s,a)\}}
}
{\max\{1,N_h^t(s,a)\}},
\quad
\widehat c_{j,h}^{\,t}(s,a)
:=
\frac{
\displaystyle
\sum_{k=1}^{t}
c_{j,h}^k\,
\mathbf{1}_{\{(s_h^k,a_h^k)=(s,a)\}}
}
{\max\{1,N_h^t(s,a)\}}.
\label{eq:reward_cost}
\end{equation}
The empirical transition model is
\begin{equation}
\widehat P_h^t(s'|s,a)
:=
\frac{N_h^t(s,a,s')}{\max\{1,N_h^t(s,a)\}}.
\label{eq:P}
\end{equation}
To quantify estimation uncertainty, we fix $\delta\in(0,1)$ and construct confidence sets for the rewards, costs, and transition kernel.
For each reward function $r_{i,h}$, we define a bonus
\begin{equation}
\sigma_h^t(s,a)
:= \min\!\left\{1,\;
\sqrt{\frac{2\ln(8m|\mathcal{S}||\mathcal{A}|HT/\delta)}
{\max\{1, N_h^{t}(s,a)\}}}
\right\}.
\label{eq:reward_confidence}
\end{equation}
By Lemma~\ref{lem:reward_concentration}, $|\widehat r^{\,t}_{i,h}(s,a) - r_{i,h}(s,a)| \le \sigma_h^t(s,a)$ holds with high probability.
For the cost functions, we define the bonus for every \(j\in[n]\), every time step \(h\) to which constraint \(j\) applies, and every \((s,a)\in\mathcal S\times\mathcal A\) as
\begin{equation}
\xi_{j,h}^t(s,a)
:=
\min\!\left\{1,\;
\sqrt{\frac{\ln(4mndT|\mathcal S||\mathcal A|/\delta)}
{2\max\{1,N_h^t(s,a)\}}}
\right\}.
\label{eq:cost_confidence}
\end{equation}
At all other steps, we set \(\widehat c_{j,h}^t(s,a)=\xi_{j,h}^t(s,a):=0\) for every \((s,a)\in\mathcal S\times\mathcal A\).
For each \(j\in[n]\) and \(h\in[H]\), we collect the cost estimates and bonuses at step \(h\) into the vectors \(\widehat c_{j,h}^t:=(\widehat c_{j,h}^t(s,a))_{(s,a)\in\mathcal S\times\mathcal A}\) and \(\xi_{j,h}^t:=(\xi_{j,h}^t(s,a))_{(s,a)\in\mathcal S\times\mathcal A}\), respectively.
We then define \(\widehat C^t,\Xi^t \in[0,1]^{H|\mathcal S||\mathcal A|\times n}\) as the matrices whose \(j\)-th columns are \((\widehat c_{j,h}^t)_{h\in[H]}\) and \((\xi_{j,h}^t)_{h\in[H]}\), respectively.
Lemma~\ref{lem:gen_confidence} ensures that $|\widehat C^t-C|\le\Xi^t$ holds with high probability.
For the transition kernel, we define a confidence set $\mathcal P^t$ by imposing concentration bounds $\bigl|P_h(s'|s,a)-\widehat P^{t}_h(s'|s,a)\bigr| \le \epsilon^t_h(s,a)$,
where
\begin{equation}
\epsilon_h^t(s,a)
:=
\sqrt{
\frac{2\ln\!\bigl(2mTH|\mathcal{S}|^2|\mathcal{A}|/\delta\bigr)}
{\max\{1,N_h^{t}(s,a)\}}
}.
\label{eq:transition_confidence}
\end{equation}
By Lemma~\ref{lem:confidenceset}, with probability at least $1-\delta$, the true kernel lies in $\mathcal P^t$ for all $t\in[T]$.
For the occupancy measure updates, we define the optimistic reward estimates by
\begin{equation}
\bar r_{i,h}^t(s,a)
:=
\min\!\left\{1,\widehat r_{i,h}^t(s,a)+\sigma_h^t(s,a)\right\},
\quad \forall i\in[m].
\label{eq:optimistic_reward}
\end{equation}
Given the preference vector $\boldsymbol\lambda$ and the smooth Tchebycheff dual variable $\boldsymbol w^t$, we define the scalarized optimistic reward
\begin{equation}
\bar R_h^t(s,a) := \sum_{i=1}^m w_i^{t}\lambda_i \bar r_{i,h}^t(s,a).
\label{eq:scalarized_reward}
\end{equation}

\begin{table}[!t]
\begin{minipage}{\linewidth}
\begin{algorithm}[H]
\caption{\texttt{CSTRL}: Constrained Smooth Tchebycheff RL}
\label{alg:CSTRL}
\small
\begin{algorithmic}[1]
\Require{$\{\alpha_t\}_{t=1}^{T-1},\{\beta_t\}_{t=1}^{T-1},\boldsymbol\lambda,\mu$}
\State Initialize
$\widetilde\theta_h^1(s,a,s')
=\frac{1}{|\mathcal S|^2|\mathcal A|}$;
set
$\widetilde\pi^1=\mathrm{Unif}(\mathcal A)$, $\psi^0:=\max_{j\in[n]}\{\frac{d-\tau_j}{d-\gamma_j}\}$,
and
$w_i^1=1/m,\  \widetilde\theta_i^1=\widetilde\theta^1$ for all $i\in[m]$; select $\pi^1$ as $\pi^\dagger$ w.p. $\psi^0$ and as $\widetilde\pi^1$ w.p. $1-\psi^0$; for each $i\in[m]$, select $\nu_i^1$ as $\pi^\dagger$ w.p. $\psi^0$ and as $\widetilde\pi^1$ w.p. $1-\psi^0$; roll out $\pi^1$ and $\{\nu_i^1\}_{i=1}^m$ 
\For{$t=1,2,\ldots,T-1$}
\For{$i=1,\ldots,m$}
\State Construct $(\bar r_i^t,\widehat C^t,\mathcal P^t,\Xi^t)$ from the trajectories generated by $\{\nu_i^k\}_{k=1}^t$ using~\eqref{eq:rollout-counts}--\eqref{eq:optimistic_reward}
\State Set
$\widehat\theta_{i,h}^{t+1}(s,a,s')
=\widetilde\theta_{i,h}^t(s,a,s')
\exp\{\eta\,\bar r_{i,h}^t(s,a)\}$
for all $h\in[H]$
\State
$(\widetilde\theta_i^{t+1},
\nu_i^{t+1},
\psi_i^t)
\gets
\texttt{SafeProj}\!(
\widehat\theta_i^{t+1},
\mathcal P^t,
\widehat C^t,
\Xi^t)$
\State Compute
$\widetilde V_i^{t+1}
=\langle\bar r_i^t,\widetilde\theta_i^{t+1}\rangle$
\EndFor
\State Construct $(\{\bar r_i^t\}_{i=1}^m,\widehat C^t,\mathcal P^t,\Xi^t)$ from the trajectories generated by $\{\pi^k\}_{k=1}^t$ using~\eqref{eq:rollout-counts}--\eqref{eq:optimistic_reward}
\State Compute $\bar R^t$ by~\eqref{eq:scalarized_reward} and set
$\widehat\theta_h^{t+1}(s,a,s')
=\widetilde\theta_h^t(s,a,s')
\exp\{\eta\,\bar R_h^t(s,a)\}$
for all $h\in[H]$
\State
$(\widetilde\theta^{t+1},
\pi^{t+1},
\psi^t)
\gets
\texttt{SafeProj}\!(
\widehat\theta^{t+1},
\mathcal P^t,
\widehat C^t,
\Xi^t)$
\State Compute
$V_i^{t+1}
=\langle\bar r_i^t,\widetilde\theta^{t+1}\rangle$
for all $i\in[m]$
\State\label{line:stch-w}
$\widetilde w_i^t=(1-\alpha_t)w_i^t+\alpha_t/m$,
$w_i^{t+1}\propto(\widetilde w_i^t)^{1-\mu\beta_t}
\exp\{\beta_t\lambda_i(\widetilde V_i^{t+1}+\iota-V_i^{t+1})\}$
for all $i\in[m]$
\State Roll out $\pi^{t+1}$ and $\{\nu_i^{t+1}\}_{i\in[m]}$
\EndFor
\State {\bfseries Return:}
$\bar{\theta}
=
\frac{1}{T}\sum_{t=1}^{T}\theta^{P,\pi^{t}}$
\end{algorithmic}
\end{algorithm}
\vspace{-18pt}
\begin{subroutine}[H]
\caption{\texttt{SafeProj}: Safe Projection}
\label{alg:safe-proj}
\small
\begin{algorithmic}[1]
\Require{
$\widehat\theta$,
$\mathcal P$,
$\widehat C$,
$\Xi$}
\If{$\texttt{Proj}(
\widehat\theta,
\mathcal P,
\widehat C,
\Xi)$ is feasible}
\State Set
$\widetilde\theta
=
\texttt{Proj}(
\widehat\theta,
\mathcal P,
\widehat C,
\Xi)$
and $\widetilde\pi=\pi^{\widetilde\theta}$
\State Compute $\psi$ by~\eqref{eq:eta_t}
\Else
\State Set
$\widetilde\theta
=
\texttt{Proj}(
\widehat\theta,
\mathcal P)$,
$\widetilde\pi=\pi^{\widetilde\theta}$,
and $\psi=1$
\EndIf
\State
Select $\pi$ as $\pi^\dagger$ w.p. $\psi$
and as $\widetilde\pi$ w.p. $1-\psi$
\State {\bfseries Return:}
$(\widetilde\theta,\pi,\psi)$
\end{algorithmic}
\end{subroutine}
\end{minipage}
\end{table}
\texttt{CSTRL} learns directly on occupancy measures and maintains a dual variable \(\boldsymbol w^t\in\Delta_m\) for the smooth Tchebycheff scalarization.
At episode \(t\in[T-1]\), it applies exponential updates with step size \(\eta>0\) to the main and auxiliary occupancy measures using \(\bar R^t\) and \(\bar r_i^t\), respectively, yielding \(\widehat\theta^{\,t+1}\) and \(\widehat\theta_i^{\,t+1}\) for each \(i\in[m]\).
These updated occupancy measures may violate the safety constraints.
We therefore project them onto an optimistic feasible occupancy measure set.
We define the occupancy measure set associated with $\mathcal P$ by $\Theta(\mathcal P):=\{\theta^{P',\pi}:P'\in\mathcal P,\ \pi\in\Pi\}.$
For a single transition kernel \(P\), we define \(\Theta(P):=\Theta(\{P\})\).
The projection minimizes the KL divergence to the exponentially updated iterate over $\Theta(\mathcal P^t)$, subject to an optimistic version of the safety constraints:
\begin{equation}
\texttt{Proj}\big(
\widehat{\theta}^{\,t+1},
\mathcal{P}^t,
\widehat{C}^{\,t},
\Xi^t
\big)
:=
\left\{
\begin{aligned}
&\arg\min_{\theta \in \Theta(\mathcal P^t)}
D_{\mathrm{KL}}\!(\theta \Vert \widehat{\theta}^{\,t+1})
\\
&\text{s.t.}\quad
\bigl(\widehat C^{\,t}-\Xi^t\bigr)^\top\theta
\leq\boldsymbol\tau,
\end{aligned}
\right.
\label{eq:proj-def}
\end{equation}
where $D_{\mathrm{KL}}(\cdot\Vert\cdot)$ is the KL divergence. 
By Lemma~\ref{lem:occ-polytope}, $\Theta(\mathcal P^t)$ is a convex polytope.
Since the safety constraints in~\eqref{eq:proj-def} are linear in \(\theta\) and $D_{\mathrm{KL}}(\theta\Vert\widehat\theta^{t+1})$ is convex in $\theta$, problem~\eqref{eq:proj-def} is a convex program.
We denote the projected occupancy measures for the main and auxiliary updates in episode $t+1$ by $\widetilde\theta^{t+1}$ and $\{\widetilde\theta_i^{t+1}\}_{i=1}^m$, respectively.
We define their induced policies as
$\widetilde\pi^{t+1}
:=
\pi^{\widetilde\theta^{t+1}}$
and
$\widetilde\nu_i^{t+1}
:=
\pi^{\widetilde\theta_i^{t+1}}$
for each $i\in[m]$.
The projection and the choice of the executed policy are implemented by Subroutine~\ref{alg:safe-proj} (\texttt{SafeProj}).
When the projection in~\eqref{eq:proj-def} is feasible, we use pessimistic cost bounds for $\widetilde\pi^{t+1}$ to define the probability of selecting $\pi^\dagger$ as
\begin{equation}
\begin{aligned}
\psi^t
:=
\max_{j\in[n]}
\frac{
\left[
\min\left\{
\max_{P'\in\mathcal P^t}
\sum_{h=1}^H
\bigl(\widehat c_{j,h}^t+\xi_{j,h}^t\bigr)^\top
\theta_h^{P',\widetilde\pi^{t+1}},
d
\right\}
-\tau_j
\right]^+
}{
\max\left\{
\min\left\{
\max_{P'\in\mathcal P^t}
\sum_{h=1}^H
\bigl(\widehat c_{j,h}^t+\xi_{j,h}^t\bigr)^\top
\theta_h^{P',\widetilde\pi^{t+1}},
d
\right\}
-\gamma_j,
\,
\tau_j-\gamma_j
\right\}
}.
\end{aligned}
\label{eq:eta_t}
\end{equation}
When \(d=1\), each constraint applies to a single step, while \(\widehat c_{j,h}^t=\xi_{j,h}^t=0\) at all other steps.
The sum in~\eqref{eq:eta_t} therefore reduces to a single term, and the truncated quantity is bounded by \(1\).
When \(d=H\), each constraint applies to every step.
In this case, the sum aggregates the contributions from all \(H\) steps, and the truncated quantity is bounded by \(H\).
If the problem~\eqref{eq:proj-def} is infeasible, it reduces to
$\texttt{Proj}(\widehat{\theta}^{\,t+1},\mathcal P^t)
:=
\arg\min_{\theta\in\Theta(\mathcal P^t)}
D_{\mathrm{KL}}\!(\theta\Vert\widehat{\theta}^{\,t+1})$.
In this case, we set $\psi^t=1$ and execute $\pi^\dagger$.
For each $i\in[m]$, the probability $\psi_i^t$ in the auxiliary update is defined analogously, with $\widetilde\pi^{t+1}$ replaced by $\widetilde\nu_i^{t+1}$.
To state \texttt{SafeProj} in a generic form, we suppress iteration and objective indices.
Given the inputs $\widehat\theta$, $\mathcal P$, $\widehat C$, and $\Xi$, \texttt{SafeProj} constructs the projected occupancy measure $\widetilde\theta$ and its induced policy $\widetilde\pi=\pi^{\widetilde\theta}$.
It then selects $\pi^\dagger$ with probability $\psi$ and $\widetilde\pi$ with probability $1-\psi$ at the beginning of the episode, and follows the selected policy throughout the episode.
Throughout the algorithms, ``w.p.'' stands for ``with probability''.
The subroutine returns $(\widetilde\theta,\pi,\psi)$, where $\pi$ is the policy executed throughout the episode.
After the occupancy measure updates, \texttt{CSTRL} updates the smooth Tchebycheff dual variable.
It first mixes $\boldsymbol w^t$ with the uniform vector on $\Delta_m$,
which keeps the logarithmic terms bounded as in~\citet{wei2020online} and~\mbox{\citet{qiu2023gradient}}.
It then applies an entropic mirror-ascent update to obtain $\boldsymbol w^{t+1}$.
This update uses the optimistic single-objective value
$\widetilde V_i^{t+1}
=
\langle\bar r_i^t,\widetilde\theta_i^{t+1}\rangle$
from the $i$-th auxiliary update and the optimistic value
$V_i^{t+1}
=
\langle\bar r_i^t,\widetilde\theta^{t+1}\rangle$
from the main update, where the two $\bar r_i^t$ come from the auxiliary and main trajectories, respectively.
The returned main policy $\pi^{t+1}$ and auxiliary policies $\{\nu_i^{t+1}\}_{i=1}^m$ are then rolled out in episode $t+1$.
Algorithm~\ref{alg:CSTRL} returns $\bar{\theta}:=\frac{1}{T}\sum_{t=1}^{T}\theta^{P,\pi^{t}}$, the primal average of the occupancy measures of the executed policies.

\vspace{5pt}
\noindent\textbf{Theoretical Results.} 
The following theorem quantifies the suboptimality of \texttt{CSTRL}.

\begin{theorem}[Suboptimality gap of \texttt{CSTRL}]
\label{thm:subopt}
Suppose Assumptions~\ref{ass:slater} and~\ref{ass:known} hold.
For any preference vector \(\boldsymbol{\lambda}\in\Delta_m\) and smoothing parameter \(\mu>0\), set \(\eta=T^{-1/2}\), \(\alpha_t=1/t^2\), and \(\beta_t=1/(\mu t)\) for \(t\in[T-1]\).
Then, with probability at least \(1-\delta\), the suboptimality gap of \texttt{CSTRL} satisfies
\[\mathrm{SubOpt}(\widehat \pi)
\le
\widetilde{\mathcal O}\Biggl(
H^2|\mathcal S|^{3/2}\sqrt{\frac{|\mathcal A|}{T}}
+
\phi_{\mathrm{safe}}H^3|\mathcal S|^{3/2}\sqrt{\frac{|\mathcal A|}{T}} 
+
\frac{H^2}{\mu T}
+
\frac{\mu}{T}
\Biggr),\]
where $\widehat{\pi}=\pi^{\bar\theta}$ is the Markov policy induced by the returned occupancy measure $\bar\theta$,
$\phi_{\mathrm{safe}}
:=
\max_{j\in[n]}
\{\frac{d^2}{(\tau_j-\gamma_j)^2}\}$, and $\widetilde{\mathcal{O}}(\cdot)$ hides logarithmic factors in $m$, $n$, $|\mathcal{S}|$, $|\mathcal{A}|$, $H$, $T$, and $1/\delta$.
\end{theorem}

Theorem~\ref{thm:subopt} guarantees that \texttt{CSTRL} achieves a suboptimality gap of order $\widetilde{\mathcal{O}}(T^{-1/2})$ with respect to the smooth Tchebycheff scalarization objective.
Compared with existing work, our result offers several improvements.
First, \citet{qiu2024traversing} study unconstrained MORL and output a policy mixture that randomly selects from $\{\pi^1,\ldots,\pi^T\}$, which does not yield a single Markov policy. 
In contrast, \texttt{CSTRL} constructs a single Markov policy $\widehat{\pi}$ from the primal average $\bar{\theta}$ of the occupancy measures, and the suboptimality guarantee applies to this policy.
Second, \citet{xiong2023provably} address stepwise safety constraints but are restricted to the single-objective setting.
Our framework extends this line of work to the multi-objective setting, covers both cumulative and stepwise safety, and provably converges to the corresponding safe Pareto optimal policy set.
Third, \citet{gu2025safe} consider constrained MORL but only guarantee an $\mathcal{O}(1/\sqrt{T})$ expected constraint violation bound. 
In contrast, Theorem~\ref{thm:safety} shows that \texttt{CSTRL} achieves zero constraint violation with high probability throughout learning, providing a strictly stronger safety guarantee.
\begin{theorem}[Safety of \texttt{CSTRL}]
\label{thm:safety}
Suppose Assumptions~\ref{ass:slater} and~\ref{ass:known} hold.
Let $\{\pi^t\}_{t=1}^T$ and $\{\nu_i^t\}_{t\in[T],\,i\in[m]}$ be the policies executed by Algorithm~\ref{alg:CSTRL}, and let $\widehat{\pi}$ be the Markov policy induced by the returned occupancy measure $\bar\theta$.
Then, with probability at least $1-\delta$,
every executed policy is safe, i.e.,
$\mathrm{Violation}(\pi^t)=0$ and
$\mathrm{Violation}(\nu_i^t)=0$
for all $t\in[T]$ and $i\in[m]$,
and the output policy satisfies
$\mathrm{Violation}(\widehat{\pi})=0$.
\end{theorem}
Theorem~\ref{thm:safety} establishes stepwise safety when \(d=1\) and cumulative safety when \(d=H\).

\section{Preference-Free Safe Multi-Objective RL}
\label{sec:pfcstrl}
We note that \texttt{CSTRL} is preference-based, since its preference vector \(\boldsymbol\lambda\) is fixed before exploration and a new exploration process is required whenever \(\boldsymbol\lambda\) changes.
To avoid repeated environment interaction across preferences while maintaining safety under either cumulative or stepwise constraints, we propose \texttt{PF-CSTRL}.
The algorithm consists of a safe exploration stage followed by a preference-dependent planning stage.
Its exploration stage collects a dataset using safe policies without a prespecified preference.
The planning stage then computes a safe policy for any \(\boldsymbol\lambda\in\Delta_m\) from this dataset without further environment interaction.
The exploration stage collects one trajectory per episode, and the counts and estimators in~\eqref{eq:rollout-counts}--\eqref{eq:P} are computed from these trajectories.
Compared with \texttt{CSTRL}, the main change is in Line~\ref{line:ex_r} of Algorithm~\ref{alg:PFCSTRL}, where the preference-weighted reward is replaced by the exploration reward \(\overline r_h^t(s,a)\).
This reward combines the reward bonus \(\sigma_h^t(s,a)\) and the transition bonus \(\zeta_h^t(s,a)\) to encourage broad state-action coverage for subsequent planning.
We define the transition bonus
\begin{equation}
\zeta_h^t(s,a)
:=
\min\!\biggl\{
H,
H\sqrt{
\frac{
8|\mathcal S|
\ln\!\bigl(2mTH|\mathcal S||\mathcal A|/\delta\bigr)
}{
\max\{1,N_h^t(s,a)\}
}
}
\biggr\}.
\label{eq:transition-width-bonus}
\end{equation}

After the exploration stage, we omit the episode superscript and use $(\widehat r,\widehat C,\mathcal P,\Xi,\sigma,\zeta)$ to denote the corresponding quantities evaluated at the final counts $N^T$.
The fixed optimistic rewards \(\{\bar r_i\}_{i=1}^m\) are then given by~\eqref{eq:optimistic_reward}.
At planning iteration \(k\in[K]\), the scalarized optimistic reward \(\bar R^k\) is given by~\eqref{eq:scalarized_reward} with the dual variable \(\boldsymbol w^k\).
Both stages use \texttt{SafeProj} and exponential updates with \(\eta>0\).
We denote the policies induced by the projected occupancy measures in exploration episode \(t+1\) and planning iteration \(k\) by
$\widetilde\pi^{t+1}
:=
\pi^{\widetilde\theta^{t+1}}$
and
$\widetilde\pi^{k+1}
:=
\pi^{\widetilde\theta^{k+1}}$,
respectively.

{\algrenewcommand\algorithmicindent{0.8em}
\begin{table}[!t]
\begin{minipage}[b]{0.505\linewidth}
\begin{algorithm}[H]
\caption{Exploration Stage for \texttt{PF-CSTRL}}
\label{alg:PFCSTRL}
{\fontsize{8pt}{17.83pt}\selectfont
\begin{algorithmic}[1]
\State Initialize
$\widetilde\theta_h^1(s,a,s')
=
\frac{1}{|\mathcal S|^2|\mathcal A|}$;
set
$\widetilde\pi^1=\mathrm{Unif}(\mathcal A)$
and
$\psi^0
:=
\max_{j\in[n]} \{\frac{d-\tau_j}{d-\gamma_j}\}$;
select $\pi^1$ as $\pi^\dagger$ w.p. $\psi^0$
and as $\widetilde\pi^1$ w.p. $1-\psi^0$;
roll out $\pi^1$, initialize $\mathcal D$,
and construct $N^1$ by~\eqref{eq:rollout-counts}
\For{$t=1,2,\ldots,T-1$}
\State Construct \((\widehat C^t,\mathcal P^t,\Xi^t)\) by~\eqref{eq:reward_cost}, \eqref{eq:P}, \eqref{eq:cost_confidence}, and~\eqref{eq:transition_confidence}
\State Compute $\sigma_h^t$ and $\zeta_h^t$ by~\eqref{eq:reward_confidence} and~\eqref{eq:transition-width-bonus}
\State \label{line:ex_r} Uncertainty-guided exploration reward: 
\Statex \quad $\overline r_h^t(\cdot,\cdot)
= \max\{\zeta_h^t(\cdot,\cdot)/H,
\sigma_h^t(\cdot,\cdot)\}$
\State \label{alg3:theta}
\(\widehat\theta_h^{t+1}(s,a,s')
=
\widetilde\theta_h^{t}(s,a,s')\exp\{\eta\overline r_h^t(s,a)\}\)
\State
$(\widetilde\theta^{t+1},
\pi^{t+1},
\psi^t)
\gets
\texttt{SafeProj}(
\widehat\theta^{t+1},
\mathcal P^t,
\widehat C^t,
\Xi^t)$
\State Roll out
$\pi^{t+1}$, update \(\mathcal D\), and compute \(N^{t+1}\) by~\eqref{eq:rollout-counts}
\EndFor
\State {\bfseries Return:} \(\mathcal D\)
\end{algorithmic}}
\end{algorithm}
\end{minipage}
\hfill
\begin{minipage}[b]{0.485\linewidth}
\begin{algorithm}[H]
\caption{Planning Stage for \texttt{PF-CSTRL}}
\label{alg:pfcstrl-plan-rf}
{\fontsize{8pt}{12pt}\selectfont
\begin{algorithmic}[1]
\Require{\(\mathcal D\) from Algorithm~\ref{alg:PFCSTRL}, 
\(\boldsymbol\lambda\), 
\(\{\alpha_k\}_{k=1}^K\), 
\(\{\beta_k\}_{k=1}^K\),
\(\mu\)}
\State Initialize \(w_i^1=1/m\) and
$\widetilde\theta_h^1(s,a,s')
=
\frac{1}{|\mathcal S|^2|\mathcal A|}$
\State Construct the fixed
\(\widehat C\), \(\mathcal P\), \(\Xi\), and
\(\bar r_i\) from \(\mathcal D\) by
\eqref{eq:reward_cost}--\eqref{eq:optimistic_reward}
\For{$i=1,\ldots,m$} \label{line:aux-proj-start}
\State $
\widehat\theta_{i,h}(s,a,s')
=
\widetilde\theta_h^1(s,a,s')
\exp\{\eta^{-1}\bar r_{i,h}(s,a)\}$
\If{\(\texttt{Proj}(\widehat\theta_{i},\mathcal P,\widehat C,\Xi)\) is feasible}
\State Set
$\widetilde\theta_i
=
\texttt{Proj}(\widehat\theta_{i},\mathcal P,\widehat C,\Xi)$
\Else
\State Set
$\widetilde\theta_i
=
\texttt{Proj}(\widehat\theta_{i},\mathcal P)$
\EndIf
\State Set 
\(\widetilde{V}_i
=
\langle\bar r_i,\widetilde\theta_i\rangle\)
\EndFor \label{line:aux-proj-end}
\For{$k=1,2,\ldots,K$}
\State 
\(\widehat\theta_h^{k+1}(s,a,s')
=
\widetilde\theta_h^k(s,a,s')\exp\{\eta\bar R_h^k(s,a)\}\)
\State
$(\widetilde\theta^{k+1},
\pi^{k+1},
\psi^k)
\gets
\texttt{SafeProj}(
\widehat\theta^{k+1},
\mathcal P,
\widehat C,
\Xi)$
\State Set
$V_i^{k+1}
=
\langle\bar r_i,\widetilde\theta^{k+1}\rangle,
\, \forall i\in[m]$
\State 
\(\widetilde w_i^k=(1-\alpha_k)w_i^k+\alpha_k/m, \forall i\in[m]\)
\State 
$w_i^{k+1}
\! \propto \!
(\widetilde w_i^k)^{1-\mu\beta_k}
\!\exp\{\beta_k\lambda_i(\widetilde{V}_i+\!\iota\!-\!V_i^{k+1})\},
\!\forall i\!\in\![m]$
\EndFor
\State {\bfseries Return:} $\bar{\theta}=\frac{1}{K}\sum_{k=1}^{K}\theta^{P,\pi^{k+1}}$
\end{algorithmic}}
\end{algorithm}
\end{minipage}
\end{table}
}

\vspace{5pt}
\noindent\textbf{Theoretical Results.} 
The following theorem bounds the suboptimality of \texttt{PF-CSTRL}.

\begin{theorem}[Suboptimality gap of \texttt{PF-CSTRL}]
\label{thm:pfcstrl_subopt}
Suppose Assumptions~\ref{ass:slater} and~\ref{ass:known} hold and \texttt{PF-CSTRL} is run for \(T\) exploration episodes and \(K\) planning iterations.
Let $\widehat{\pi}$ denote the policy induced by the averaged occupancy measure output at the end of the planning stage.
For any preference vector \(\boldsymbol\lambda\in\Delta_m\) and smoothing parameter \(\mu>0\), set \(\alpha_k=1/k^2\), \(\beta_k=1/(\mu k)\), \(\eta=T^{-1/2}\) for exploration, and \(\eta=K^{-1/2}\) for planning.
Then, with probability at least \(1-\delta\), we have
\[
\mathrm{SubOpt}(\widehat\pi)
\le
\widetilde{\mathcal O}\!\left(
\phi_{\mathrm{safe}}H^4|\mathcal S|^{2}
\sqrt{\frac{|\mathcal A|}{T}}
+
\frac{H}{\sqrt K}
+
\frac{H^2}{\mu K}
+
\frac{\mu}{K}
\right),
\]
where
$\phi_{\mathrm{safe}}
=
\max_{j\in[n]}
\left\{
\frac{d^2}{(\tau_j-\gamma_j)^2}
\right\}$,
and \(\widetilde{\mathcal O}(\cdot)\) hides logarithmic factors in $m$, $n$, $|\mathcal S|$, $|\mathcal A|$, $H$, $T$, $K$, and \(1/\delta\).
\end{theorem}

The bound in Theorem~\ref{thm:pfcstrl_subopt} separates the statistical error of preference-free exploration from the optimization error of planning on the fixed dataset.
The \(T\)-dependent term arises from the fixed reward and transition bonuses and the probabilities of selecting $\pi^\dagger$ during planning.
The \(K\)-dependent terms arise from the main occupancy measure updates, the optimistic single-objective values \(\widetilde V_i\), and the smooth Tchebycheff dual updates.
\texttt{PF-CSTRL} guarantees safety under either formulation throughout both stages.

\begin{theorem}
[Safety of \texttt{PF-CSTRL}]
\label{thm:safety-pfcstrl}
Suppose Assumptions~\ref{ass:slater} and~\ref{ass:known} hold.
Let $\{\pi^t\}_{t=1}^{T}$ and $\{\pi^{k+1}\}_{k=1}^{K}$ denote the policies generated in the exploration and planning stages, respectively, by Algorithms~\ref{alg:PFCSTRL} and~\ref{alg:pfcstrl-plan-rf}.
Then, with probability at least \(1-\delta\), we have $\mathrm{Violation}(\pi^t)=0, \forall t\in[T],$ and $\mathrm{Violation}(\pi^{k+1})=0, \forall k\in[K].$
Consequently, the returned policy \(\widehat\pi\) induced by \(\bar\theta\) satisfies \(\mathrm{Violation}(\widehat\pi)=0\).
\end{theorem}

\section{Extensions}
\label{sec:extensions}
Both \texttt{CSTRL} and \texttt{PF-CSTRL} assume that the learner knows a strictly safe policy $\pi^\dagger$ and its costs $\boldsymbol{\gamma}$.
We relax these requirements in this section.

\subsection{Known \texorpdfstring{$\pi^\dagger$}{pi-dagger} and Unknown \texorpdfstring{$\boldsymbol{\gamma}$}{gamma}}
\label{sec:unknown-gamma}
We first remove the requirement that the learner knows $\boldsymbol\gamma$.
\begin{assumption}
\label{ass:unknown-gamma}
The policy $\pi^\dagger$ in Assumption~\ref{ass:slater} is known to the learner, but its costs $\boldsymbol{\gamma}$ are unknown.
\end{assumption}
Since $\pi^\dagger$ is known, the learner can estimate $\boldsymbol\gamma$ by repeatedly executing $\pi^\dagger$.
We define the Slater margin of $\pi^\dagger$ as \(\varrho^\dagger:=\min_{j\in[n]}(\tau_j-\gamma_j)\).
The empirical costs of $\pi^\dagger$ after \(t\) episodes are $\bar\gamma_j^t:=\frac{1}{t}\sum_{k=1}^t\sum_{h=1}^H c_{j,h}^k$.
The stopping time \(t_s\) is the first episode \(t\in[T]\) with $\tau_j-\bar\gamma_j^t\ge3d\sqrt{\ln(4nT/\delta)/(2t)}$ for all $j\in[n]$, or \(T\) if no such episode exists.
The learner executes $\pi^\dagger$ for $t_s$ episodes, which we call the estimation phase, and then constructs a conservative upper confidence bound $\widehat{\boldsymbol\gamma}$ on $\boldsymbol\gamma$ with $\widehat\gamma_j:=\bar\gamma_j^{t_s}+d\sqrt{\ln(4nT/\delta)/(2t_s)}$.
Our analysis shows that, with high probability,
$t_s
\le
1+{8d^2\ln(4nT/\delta)}/{(\varrho^\dagger)^2}
=
\widetilde{\mathcal O}(d^2/(\varrho^\dagger)^2).$
The resulting preference-based extension is \texttt{CSTRL-UC} (\texttt{CSTRL} with Unknown Constraint Costs).
It runs \texttt{CSTRL} for the remaining \(T-t_s\) episodes with $\widehat{\boldsymbol\gamma}$ in place of $\boldsymbol\gamma$.
To compute a safe policy for any subsequently specified preference, we also combine this estimation phase with \texttt{PF-CSTRL}.
The resulting preference-free extension is \texttt{PF-CSTRL-UC} (\texttt{PF-CSTRL} with Unknown Constraint Costs).
It runs \texttt{PF-CSTRL} with $\widehat{\boldsymbol\gamma}$ in place of $\boldsymbol\gamma$ for the remaining \(T-t_s\) exploration episodes, followed by \(K\) planning iterations.

\vspace{5pt}
\noindent\textbf{Theoretical Results.} 
We now present the guarantees for \texttt{CSTRL-UC} and \texttt{PF-CSTRL-UC}.

\begin{theorem}
[Suboptimality gap of \texttt{CSTRL-UC}]
\label{thm:unknown-gamma-subopt}
Suppose Assumptions~\ref{ass:slater} and~\ref{ass:unknown-gamma} hold.
Let $\widehat{\pi}$ be the policy returned by \texttt{CSTRL-UC}.
Assume that $T>1+8d^2\ln(4nT/\delta)/(\varrho^\dagger)^2$.
For any preference vector $\boldsymbol{\lambda}\in\Delta_m$ and smoothing parameter $\mu>0$, set \(\eta=(T-t_s)^{-1/2}\), \(\alpha_t=1/t^2\), and \(\beta_t=1/(\mu t)\).
Then, with probability at least $1-\delta$, we have $t_s<T$ and
\[
\mathrm{SubOpt}(\widehat{\pi})
\!\leq\!
\widetilde{\mathcal{O}}\!\left(
H^2|\mathcal S|^{3/2}\sqrt{\frac{|\mathcal A|}{T-t_s}}
+
\widehat{\phi}_{\mathrm{safe}}\,H^3|\mathcal{S}|^{3/2}\sqrt{\frac{|\mathcal A|}{T-t_s}}
+
\frac{H^2}{\mu (T-t_s)}
+
\frac{\mu}{T-t_s}\right),
\]
where
$\widehat{\phi}_{\mathrm{safe}}
:=
\max_{j\in[n]}
\frac{d^2}{(\tau_j-\widehat{\gamma}_j)^2}$
satisfies $\widehat{\phi}_{\mathrm{safe}}\le 4\phi_{\mathrm{safe}}$, and $\widetilde{\mathcal{O}}(\cdot)$ hides logarithmic factors in $m$, $n$, $|\mathcal{S}|$, $|\mathcal{A}|$, $H$, $T$, and $1/\delta$.
\end{theorem}

\begin{theorem}
[Safety of \texttt{CSTRL-UC}]
\label{thm:unknown-gamma-safety}
Suppose Assumptions~\ref{ass:slater} and~\ref{ass:unknown-gamma} hold.
Assume that $T>1+8d^2\ln(4nT/\delta)/(\varrho^\dagger)^2$.
Let $\widehat{\pi}$ be the policy returned by \texttt{CSTRL-UC}. 
With probability at least $1-\delta$, the stopping time satisfies $t_s<T$, and every policy executed is safe.
The returned policy satisfies $\mathrm{Violation}(\widehat{\pi}) = 0$.
\end{theorem}

\begin{theorem}[Suboptimality gap of \texttt{PF-CSTRL-UC}]
\label{thm:pfcstrl-uc-subopt}
Under Assumptions~\ref{ass:slater} and~\ref{ass:unknown-gamma}, let $\widehat{\pi}$ be the policy returned by \texttt{PF-CSTRL-UC}.
Assume that $T>1+8d^2\ln(4nT/\delta)/(\varrho^\dagger)^2$.
For any preference vector \(\boldsymbol\lambda\in\Delta_m\) and smoothing parameter \(\mu>0\), set \(\alpha_k=1/k^2\), \(\beta_k=1/(\mu k)\), \(\eta=(T-t_s)^{-1/2}\) for exploration, and \(\eta=K^{-1/2}\) for planning.
Then, with probability at least \(1-\delta\), we have \(t_s<T\) and
\[
\mathrm{SubOpt}(\widehat{\pi})
\le
\widetilde{\mathcal O}\!\left(
\widehat{\phi}_{\mathrm{safe}}H^4|\mathcal S|^{2}
\sqrt{\frac{|\mathcal A|}{T-t_s}}
+
\frac{H}{\sqrt K}
+
\frac{H^2}{\mu K}
+
\frac{\mu}{K}
\right),
\]
where \(\widehat{\phi}_{\mathrm{safe}}\) is defined in Theorem~\ref{thm:unknown-gamma-subopt} and satisfies \(\widehat{\phi}_{\mathrm{safe}} \le 4\phi_{\mathrm{safe}}\).
The notation \(\widetilde{\mathcal O}(\cdot)\) hides logarithmic factors in \(m\), \(n\), $|\mathcal S|$, $|\mathcal A|$, \(H\), \(T\), \(K\), and \(1/\delta\).
\end{theorem}

\begin{theorem}[Safety of \texttt{PF-CSTRL-UC}]
\label{thm:pfcstrl-uc-safety}
Suppose Assumptions~\ref{ass:slater} and~\ref{ass:unknown-gamma} hold.
Assume that $T>1+8d^2\ln(4nT/\delta)/(\varrho^\dagger)^2$.
With probability at least \(1-\delta\), the stopping time satisfies \(t_s<T\).
On the same event, every policy executed during the estimation phase or the exploration stage is safe, as is every main policy generated during the planning stage.
The returned policy \(\widehat{\pi}\) satisfies \(\mathrm{Violation}(\widehat{\pi})=0\).
\end{theorem}

\subsection{Unknown \texorpdfstring{$\pi^\dagger$}{pi-dagger} and Unknown \texorpdfstring{$\boldsymbol{\gamma}$}{gamma}}
\label{sec:cstrl-us}
Both \texttt{CSTRL-UC} and \texttt{PF-CSTRL-UC} still require the learner to know \(\pi^\dagger\).
We now remove this requirement and assume only Assumption~\ref{ass:slater}, under which a strictly safe policy exists but the learner knows neither this policy nor its costs.
\texttt{CSTRL-US} (\texttt{CSTRL} with an Unknown Safe Policy) first estimates a strictly safe policy and its costs and then runs \texttt{CSTRL}.
Algorithm~\ref{alg:cstrl_us} in Appendix~\ref{app:us} gives the full procedure.
We define the Slater parameter as
$\varrho
:=
\max_{\pi\in\Pi}
\min_{j\in[n]}
\{\tau_j-(C^\top\theta^{P,\pi})_j\}.$
For the estimation phase, we define $\mathtt{Alg}^P$ and $\mathtt{Alg}^D$ as the primal and dual regret minimizers, respectively.
\begin{itemize}
\item \textbf{Primal algorithm $\mathtt{Alg}^P$.}
For each episode $t$ in the estimation phase, let $\phi^t\in\Delta_n$ denote the probability distribution over the $n$ constraints, where $\Delta_n$ is the probability simplex in $\mathbb R^n$.
For each $h\in[H]$, we define the observed loss as $\ell_h^t:=\sum_{j=1}^n\phi_j^t c_{j,h}^t-\frac{1}{H}\sum_{j=1}^n\phi_j^t\tau_j$.
For any policy $\pi\in\Pi$, we define its expected cumulative loss as $\bar\ell^t(\pi):=\sum_{j=1}^n\phi_j^t[(C^\top\theta^{P,\pi})_j-\tau_j]$.
At the beginning of episode $t$, $\mathtt{Alg}^P$ takes $\{(s_h^k,a_h^k,\ell_h^k)\}_{k<t,\,h\in[H]}$ as input and outputs a policy $\pi^t\in\Pi$.
After executing $\pi^t$, the learner executes $\mathtt{Alg}^P$ using $\{(s_h^t,a_h^t,\ell_h^t)\}_{h=1}^H$.
$\mathtt{Alg}^P$ can be implemented using the algorithm of~\citet{jin2020learning} together with a standard doubling trick~\citep{lattimore2020bandit}.
We assume that, with probability at least $1-\delta$, its regret satisfies
$\sum_{k=1}^t[\bar\ell^k(\pi^k)-\bar\ell^k(\pi)]
\le
\mathcal O\bigl(H^2|\mathcal S|\sqrt{|\mathcal A|t\ln(2H|\mathcal S||\mathcal A|t/\delta)}\bigr)$
simultaneously for every $t\ge1$ and every fixed $\pi\in\Pi$.

\item \textbf{Dual algorithm $\mathtt{Alg}^D$.}
For each episode $t$ in the estimation phase, we define $\varrho_j^t:=\tau_j-\sum_{h=1}^H c_{j,h}^t$ for every $j\in[n]$ and set $\varrho^t:=(\varrho_1^t,\ldots,\varrho_n^t)\in\mathbb R^n$.
At the beginning of episode $t$, $\mathtt{Alg}^D$ takes $\{\varrho^k\}_{k=1}^{t-1}$ as input and outputs the distribution $\phi^t\in\Delta_n$.
After observing the constraint costs in episode $t$, the learner executes $\mathtt{Alg}^D$ using $\varrho^t$.
$\mathtt{Alg}^D$ can be implemented using online gradient descent~\citep{orabona2019modern}, whose regret satisfies $\sum_{k=1}^t\langle\phi^k-\phi,\varrho^k\rangle\le\mathcal O(d\sqrt{nt})$ for every $t\ge1$ and every fixed $\phi\in\Delta_n$.
\end{itemize}
We define the stopping time of the estimation phase as
$t_s
\!:=\!
\min\{
t\in[T]\!:\!
\min_{j\in[n]}
\sum_{k=1}^t \varrho_j^k
\ge
\mathcal O((H^2|\mathcal S|\sqrt{|\mathcal A|}+d\sqrt n)\sqrt{t\ln(2nH|\mathcal S||\mathcal A|t/\delta)})\},$
where the right-hand side of the inequality is the stopping threshold.
At the end of episode \(t_s\), the algorithm constructs the policy $\widehat\pi^\dagger$, which selects $\pi^t$ with probability $1/t_s$ for each $t\in[t_s]$.
It then computes an estimate $\widehat\varrho$ of $\varrho$ and sets $\widehat{\boldsymbol\gamma} := (\widehat\gamma_1,\ldots,\widehat\gamma_n)$, where $\widehat\gamma_j:=\tau_j-\widehat\varrho$ for every $j\in[n]$.
The algorithm runs \texttt{CSTRL} for the remaining $T-t_s$ episodes with $\widehat\pi^\dagger$ and $\widehat{\boldsymbol\gamma}$ in place of $\pi^\dagger$ and $\boldsymbol\gamma$.
To avoid repeated environment interaction across preferences, we also extend \texttt{PF-CSTRL} to this setting using the same estimation phase.
The resulting extension is \texttt{PF-CSTRL-US} (\texttt{PF-CSTRL} with an Unknown Safe Policy).
It runs \texttt{PF-CSTRL} with $\widehat\pi^\dagger$ and $\widehat{\boldsymbol\gamma}$ in place of $\pi^\dagger$ and $\boldsymbol\gamma$ for the remaining \(T-t_s\) exploration episodes, followed by \(K\) planning iterations.
For both \texttt{CSTRL-US} and \texttt{PF-CSTRL-US}, we assume that $T$ is sufficiently large, and it suffices to take $T=\widetilde{\mathcal O}((H^4|\mathcal S|^2|\mathcal A|+nd^2)/\varrho^2)$, where $\widetilde{\mathcal O}(\cdot)$ hides logarithmic factors in $n$, $H$, $|\mathcal S|$, $|\mathcal A|$, $1/\delta$, and $1/\varrho$.
Under this condition, Lemma~\ref{lem:us_stop_time} gives $t_s<T$ with high probability.

\vspace{5pt}
\noindent\textbf{Theoretical Results.} 
We now present the guarantees for \texttt{CSTRL-US} and \texttt{PF-CSTRL-US}.
\begin{theorem}[Suboptimality gap of \texttt{CSTRL-US}]
\label{thm:us_subopt}
Suppose Assumption~\ref{ass:slater} holds.
Let $\widehat{\pi}$ be the policy returned by Algorithm~\ref{alg:cstrl_us}.
For any preference vector $\boldsymbol{\lambda}\in\Delta_m$ and smoothing parameter $\mu>0$, set \(\eta=(T-t_s)^{-1/2}\), \(\alpha_t=1/t^2\), and \(\beta_t=1/(\mu t)\).
Then, with probability at least $1-\delta$, we have $T>t_s$ and
\[
\mathrm{SubOpt}(\widehat{\pi})
\!\leq \!
\widetilde{\mathcal{O}}\!\left(
H^2|\mathcal{S}|^{3/2}\sqrt{\frac{|\mathcal A|}{T-t_s}}
\!+\!
\widehat{\phi}_{\mathrm{safe}}\,H^3|\mathcal{S}|^{3/2}\sqrt{\frac{|\mathcal A|}{T-t_s}}
\!+\!
\frac{H^2}{\mu (T-t_s)}
\!+\!
\frac{\mu}{T-t_s}
\right),
\]
where $\widehat{\phi}_{\mathrm{safe}}
:=
\max_{j\in[n]}
\frac{d^2}{(\tau_j-\widehat{\gamma}_j)^2}$ satisfies $\widehat{\phi}_{\mathrm{safe}} \le 4d^2/\varrho^2$, and $\widetilde{\mathcal{O}}(\cdot)$ hides logarithmic factors in $m$, $n$, $|\mathcal{S}|$, $|\mathcal{A}|$, $H$, $T$, and $1/\delta$.
\end{theorem}

\begin{theorem}[Constant violation of \texttt{CSTRL-US}]
\label{thm:us_safety}
Under Assumption~\ref{ass:slater}, let $\{\pi^t\}_{t=1}^{T}$ be the main policies executed by Algorithm~\ref{alg:cstrl_us}.
Then, with probability at least $1-\delta$, we have $T>t_s$ and $\sum_{t=1}^{T}\mathrm{Violation}(\pi^t)\le d t_s$.
On the same event, every policy executed after the estimation phase is safe, and the output policy $\widehat{\pi}$ satisfies $\mathrm{Violation}(\widehat{\pi})=0$.
\end{theorem}

\begin{theorem}[Suboptimality gap of \texttt{PF-CSTRL-US}]
\label{thm:pfcstrl-us-subopt}
Suppose Assumption~\ref{ass:slater} holds.
Let $\widehat{\pi}$ be the policy returned by \texttt{PF-CSTRL-US}.
For any preference vector \(\boldsymbol\lambda\in\Delta_m\) and smoothing parameter \(\mu>0\), set \(\alpha_k=1/k^2\), \(\beta_k=1/(\mu k)\), \(\eta=(T-t_s)^{-1/2}\) for exploration, and \(\eta=K^{-1/2}\) for planning.
Then, with probability at least $1-\delta$, we have $T>t_s$ and
\[
\mathrm{SubOpt}(\widehat{\pi})
\le
\widetilde{\mathcal O}\!\left(
\widehat{\phi}_{\mathrm{safe}}H^4|\mathcal S|^{2}
\sqrt{\frac{|\mathcal A|}{T-t_s}}
+
\frac{H}{\sqrt K}
+
\frac{H^2}{\mu K}
+
\frac{\mu}{K}
\right),
\]
where \(\widehat{\phi}_{\mathrm{safe}}\) is defined in Theorem~\ref{thm:us_subopt} and satisfies \(\widehat{\phi}_{\mathrm{safe}}\le4d^2/\varrho^2\).
The notation \(\widetilde{\mathcal O}(\cdot)\) hides logarithmic factors in \(m\), \(n\), $|\mathcal S|$, $|\mathcal A|$, \(H\), \(T\), \(K\), and \(1/\delta\).
\end{theorem}

\begin{theorem}[Constant violation of \texttt{PF-CSTRL-US}]
\label{thm:pfcstrl-us-safety}
Suppose Assumption~\ref{ass:slater} holds.
Let $\{\pi^t\}_{t=1}^{T}$ be the policies executed during the estimation phase and the subsequent exploration stage of \texttt{PF-CSTRL-US}.
Then, with probability at least $1-\delta$, we have $T>t_s$ and $\sum_{t=1}^{T}\mathrm{Violation}(\pi^t)\le d t_s$.
On the same event, every policy executed during preference-free exploration and every main policy generated during planning is safe, and the output policy $\widehat{\pi}$ satisfies $\mathrm{Violation}(\widehat{\pi})=0$.
\end{theorem}

\section{Theoretical Analysis} 
\label{sec:theory}
This section outlines the main arguments used to establish our theoretical results.

\subsection{Proof Sketch of Theorem~\ref{thm:subopt}} 
\label{sec:sketch-subopt}
The full proof of Theorem~\ref{thm:subopt} is provided in Appendix~\ref{app:cstrl-subopt}. 
For each objective \(i\in[m]\), let
\(\theta_i^*
\in
\arg\max_{\theta\in\Theta_{\mathrm{safe}}}V_i(\theta)\) 
be a safe single-objective optimal occupancy measure.
Let \(\theta_{\boldsymbol{\lambda}}^*\in\Theta_{\mathrm{safe}}\) be an optimal safe occupancy measure for the smooth Tchebycheff objective under the preference \(\boldsymbol\lambda\).
Since $\stchlm(\pi^\theta)$ is convex in $\theta$ and changes by at most $H$ between any two occupancy measures, the contribution of $\pi^1$ to $\mathrm{SubOpt}(\widehat\pi)$ is at most $H/T$.
Combining this separation with the min-max reformulation of the smooth Tchebycheff scalarization yields
\begin{align*}
&\mathrm{SubOpt}(\widehat\pi)
\le
\frac{H}{T}
+
\underbrace{
\max_{\boldsymbol w\in\Delta_m}
\sum_{i=1}^m
w_i\lambda_i
\frac1T\sum_{t=1}^{T-1}
\Bigl(
V_i^*
-
V_i(\theta^{P,\widetilde\nu_i^{t+1}})
\Bigr)}_{\mathrm{(I)}}\displaybreak[1] \\
&\quad +
\underbrace{
\max_{\boldsymbol w\in\Delta_m}
\frac1T\sum_{t=1}^{T-1}
\sum_{i=1}^m
\left[
(w_i-w_i^t)\lambda_i
\left(
V_i(\theta^{P,\widetilde\nu_i^{t+1}})
+\iota
-
V_i(\theta^{P,\pi^{t+1}})
\right)
-\mu w_i\log w_i
+\mu w_i^t\log w_i^t
\right]}_{\mathrm{(II)}}\\
&\quad+
\underbrace{
\frac1T\sum_{t=1}^{T-1}
\sum_{i=1}^m
w_i^t\lambda_i
\left(
V_i(\theta_{\boldsymbol{\lambda}}^*)-V_i(\theta^{P,\pi^{t+1}})
\right)
}_{\mathrm{(III)}} 
+
\underbrace{
\frac{H}{T}\sum_{t=1}^{T-1}
\max_{\boldsymbol w\in\Delta_m}
\sum_{i=1}^m
w_i\lambda_i\psi_i^t
}_{\mathrm{(IV)}}.
\end{align*}
Terms \(\mathrm{(I)}\) and \(\mathrm{(III)}\) are the auxiliary and main occupancy measure learning errors.
Term \(\mathrm{(II)}\) contains the dual and estimation errors, while Term \(\mathrm{(IV)}\) arises from the probabilities of selecting $\pi^\dagger$ in the auxiliary updates.

\vspace{5pt}
\noindent\textbf{Bounding Term \(\mathrm{(I)}\).}
On the high-probability event, \(\theta_i^*\) is feasible for the auxiliary projection.
Lemma~\ref{lem:aux-occ-regret} therefore gives
\(\textstyle
\tfrac1T \! \sum_{t=1}^{T-1}
(V_i^*
\!-\!
V_i(\theta^{P,\widetilde\nu_i^{t+1}}))
\allowbreak
\!\le\!
\widetilde{\mathcal O}\!(
H^2|\mathcal S|^{3/2}\sqrt{|\mathcal A|/T})
\allowbreak
\!+\!
\tfrac{4H}{T}\sum_{t=1}^{T-1}\psi_i^t\).
The difference between the projected occupancy measure \(\widetilde\theta_i^{t+1}\in\Theta(\mathcal P^t)\) and the occupancy measure of \(\widetilde\nu_i^{t+1}\) under \(P\) is controlled by Lemma~\ref{lem:aux-weighted-transition-bridge}. 
The remaining error is controlled by the probabilities $\psi_i^t$.
For every \(i\in[m]\), we have
$\sum_{t=1}^{T-1}\psi_i^t
\le
\widetilde{\mathcal O}\!(
\phi_{\mathrm{safe}}H^2|\mathcal S|^{3/2}
\sqrt{|\mathcal A|T}).$
Since \(\boldsymbol w,\boldsymbol\lambda\in\Delta_m\), we obtain
$\mathrm{(I)}
\le
\widetilde{\mathcal O}\!(
(1+\phi_{\mathrm{safe}}H)
H^2|\mathcal S|^{3/2}\sqrt{|\mathcal A|/T}).$

\vspace{5pt}
\noindent\textbf{Bounding Term \(\mathrm{(II)}\).}
We control Term \(\mathrm{(II)}\) by comparing the true values in the smooth Tchebycheff objective with the optimistic values used in the dual update.
For each \(t\in[T-1]\), we define
\(\widetilde{\boldsymbol V}^{t+1}
:=
(\widetilde V_1^{t+1},\ldots,\widetilde V_m^{t+1})\)
and
\(\boldsymbol V^{t+1}
:=
(V_1^{t+1},\ldots,V_m^{t+1})\).
We also define the constant vector
\(\boldsymbol{\iota}
:=
(\iota,\ldots,\iota)\in\mathbb R^m\).
Then we obtain
\begin{align*}
\mathrm{(II)}
&\le
\underbrace{\max_{\boldsymbol w\in\Delta_m}
\frac1T\sum_{t=1}^{T-1}
\left[
\left\langle
\boldsymbol w-\boldsymbol w^t,\,
\boldsymbol\lambda\odot
\left(
\widetilde{\boldsymbol V}^{t+1}
+\boldsymbol \iota
-{\boldsymbol V}^{t+1}
\right)
\right\rangle
-\mu\langle\boldsymbol w,\log\boldsymbol w\rangle
+\mu\langle\boldsymbol w^t,\log\boldsymbol w^t\rangle
\right]}_{\mathrm{(II.1)}}\\
&\quad+
\underbrace{\frac{2}{T}\sum_{t=1}^{T-1}\sum_{i=1}^m
\lambda_i
\left|
V_i(\theta^{P,\widetilde\nu_i^{t+1}})
-
\widetilde V_i^{t+1}
\right|}_{\mathrm{(II.2)}}
+
\underbrace{\frac{2}{T}\sum_{t=1}^{T-1}\sum_{i=1}^m
\lambda_i
\left|
V_i^{t+1}
-
V_i(\theta^{P,\pi^{t+1}})\right|}_{\mathrm{(II.3)}}.
\end{align*} 
Lemma~\ref{lem:oco-stch} controls the dual mirror-ascent term from the entropic update of \(\{\boldsymbol w^t\}_{t=1}^{T-1}\).
Therefore, we have
$\mathrm{(II.1)}
\le
\widetilde{\mathcal O}\!(
\frac{H^2}{\mu T}
+
\frac{\mu}{T}).$
For every \(i\in[m]\) and \(t\in[T-1]\), 
the reward confidence event gives
\(0\le \bar r_i^t-r_i\le 2\sigma^t\). 
Thus, we obtain
$|
V_i(\theta^{P,\widetilde\nu_i^{t+1}})
-
\widetilde V_i^{t+1}|
\le
2\langle \sigma^t,\widetilde\theta_i^{t+1}\rangle
+
\|
\widetilde\theta_i^{t+1}
-
\theta^{P,\widetilde\nu_i^{t+1}}
\|_1 .$
The reward bonus term is controlled by summing the bonuses as in the proof of Lemma~\ref{lem:aux-occ-regret}.
The transition term is controlled by Lemma~\ref{lem:aux-weighted-transition-bridge}, together with the bound on $\sum_{t=1}^{T-1}\psi_i^t$ above.
Consequently, we have
$\mathrm{(II.2)}
\le
\widetilde{\mathcal O}\!(
H^2|\mathcal S|^{3/2}\sqrt{|\mathcal A|/T}
+
\phi_{\mathrm{safe}}H^3|\mathcal S|^{3/2}
\sqrt{|\mathcal A|/T}).$
Using \(0\le \bar r_i^t-r_i\le 2\sigma^t\) and
$\theta^{P,\pi^{t+1}}
=
(1-\psi^t)\theta^{P,\widetilde\pi^{t+1}}
+
\psi^t\theta^{P,\pi^\dagger},$
we have
$|
V_i^{t+1}
-
V_i(\theta^{P,\pi^{t+1}})|
\le
2\langle \sigma^t,\widetilde\theta^{t+1}\rangle
+
\|
\widetilde\theta^{t+1}
-
\theta^{P,\widetilde\pi^{t+1}}
\|_1
+
H\psi^t.$
The reward bonus term is bounded by the argument in the proof of Lemma~\ref{lem:main-occ-regret}, and the transition term is controlled by Lemma~\ref{lem:weighted-transition-bridge}.
The policy decomposition used in the proof of Theorem~\ref{thm:safety}, together with Lemma~\ref{lem:mixing-bound}, gives
$\sum_{t=1}^{T-1}\psi^t
\le
\widetilde{\mathcal O}\!(
\phi_{\mathrm{safe}}H^2|\mathcal S|^{3/2}
\sqrt{|\mathcal A|T})$.
Together, these bounds give
$\mathrm{(II.3)}
\le
\widetilde{\mathcal O}\!(
H^2|\mathcal S|^{3/2}\sqrt{|\mathcal A|/T}
+
\phi_{\mathrm{safe}}H^3|\mathcal S|^{3/2}
\sqrt{|\mathcal A|/T}).$
Combining the bounds for $\mathrm{(II.1)}$, $\mathrm{(II.2)}$, and $\mathrm{(II.3)}$, 
we obtain
$\mathrm{(II)}
\le
\widetilde{\mathcal O}\!(
H^2|\mathcal S|^{3/2}\sqrt{|\mathcal A|/T}
+
\phi_{\mathrm{safe}}H^3|\mathcal S|^{3/2}
\sqrt{|\mathcal A|/T}
+
\frac{H^2}{\mu T}
+
\frac{\mu}{T}).$

\vspace{5pt}
\noindent\textbf{Bounding Term \(\mathrm{(III)}\).}
Since \(V_i(\theta)=\langle r_i,\theta\rangle\), 
Term \(\mathrm{(III)}\) can be written as
\begin{align*}
\mathrm{(III)}
=
\underbrace{\frac1T\sum_{t=1}^{T-1}
\left\langle
\sum_{i=1}^m w_i^t\lambda_i r_i,\,
\theta_{\boldsymbol{\lambda}}^*-\widetilde\theta^{t+1}
\right\rangle}_{\mathrm{(III.1)}}
+
\underbrace{\frac1T\sum_{t=1}^{T-1}
\left\langle
\sum_{i=1}^m w_i^t\lambda_i r_i,\,
\widetilde\theta^{t+1}-\theta^{P,\pi^{t+1}}
\right\rangle}_{\mathrm{(III.2)}} .
\end{align*}
Since \(\theta_{\boldsymbol{\lambda}}^*\in\Theta_{\mathrm{safe}}\),
Lemma~\ref{lem:main-occ-regret} applies on the high-probability event and gives
$\mathrm{(III.1)}
\le
\widetilde{\mathcal O}\!(
H^2|\mathcal S|^{3/2}\sqrt{|\mathcal A|/T})
+
\frac{2H}{T}\sum_{t=1}^{T-1}\psi^t.$
For the second term, the scalarized reward \(\sum_{i=1}^m w_i^t\lambda_i r_i\) is bounded in \([0,1]\) since \(\boldsymbol w^t,\boldsymbol\lambda\in\Delta_m\) and \(r_i\in[0,1]\). 
Using the decomposition of \(\theta^{P,\pi^{t+1}}\) above and Lemma~\ref{lem:weighted-transition-bridge}, we obtain
${\mathrm{(III.2)}}
\le
\widetilde{\mathcal O}\!(
H^2|\mathcal S|^{3/2}\sqrt{|\mathcal A|/T})
+
\frac{3H}{T}\sum_{t=1}^{T-1}\psi^t .$
Thus, we obtain
$\mathrm{(III)}
\le
\widetilde{\mathcal O}\!(
H^2|\mathcal S|^{3/2}\sqrt{|\mathcal A|/T}
+
\phi_{\mathrm{safe}}H^3|\mathcal S|^{3/2}
\sqrt{|\mathcal A|/T}).$

\vspace{5pt}
\noindent\textbf{Bounding Term \(\mathrm{(IV)}\).}
Since \(\boldsymbol w,\boldsymbol\lambda\in\Delta_m\), we have
$\mathrm{(IV)}
\le
\frac{H}{T}\sum_{i=1}^m\lambda_i\sum_{t=1}^{T-1}\psi_i^t .$
Using the bound on $\sum_{t=1}^{T-1}\psi_i^t$ above,
we have
$\mathrm{(IV)}
\le
\widetilde{\mathcal O}\!(
\phi_{\mathrm{safe}}H^3|\mathcal S|^{3/2}
\sqrt{|\mathcal A|/T}).$

\vspace{5pt}
\noindent\textbf{Final Bound.}
Combining the bounds for Terms \(\mathrm{(I)}\), \(\mathrm{(II)}\), \(\mathrm{(III)}\), and \(\mathrm{(IV)}\)
with the $H/T$ contribution of the initial policy, 
we obtain, with probability at least $1-\delta$,
\[
\mathrm{SubOpt}(\widehat\pi)
\le
\widetilde{\mathcal O}\!\left(
H^2|\mathcal S|^{3/2}\sqrt{|\mathcal A|/T}
+
\phi_{\mathrm{safe}}H^3|\mathcal S|^{3/2}
\sqrt{|\mathcal A|/T}
+
\frac{H^2}{\mu T}
+
\frac{\mu}{T}
\right).
\]

\begin{remark}
In the unconstrained MORL analysis of~\citet{qiu2024traversing}, greedy updates and optimism bound the auxiliary and main learning errors using bonus sums.
This argument does not directly apply under safety constraints, since a policy chosen greedily with respect to optimistic Q-functions may violate the safety constraints.
We therefore apply mirror ascent followed by \texttt{SafeProj} to every auxiliary and main occupancy measure update.
Accordingly, we bound the auxiliary and main learning errors in Terms~$\mathrm{(I)}$ and~$\mathrm{(III)}$ directly using the occupancy measure updates rather than the optimistic Bellman recursions used by~\citet{qiu2024traversing}.
The probabilities of selecting $\pi^\dagger$ do not appear in their unconstrained analysis.
The probabilities $\psi_i^t$ in the auxiliary updates enter Terms~$\mathrm{(I)}$, $\mathrm{(II)}$, and~$\mathrm{(IV)}$, while the probabilities $\psi^t$ in the main updates enter Terms~$\mathrm{(II)}$ and~$\mathrm{(III)}$.
Bounding their sums relies on Lemma~\ref{lem:mixing-bound} and yields the additional term involving $\phi_{\mathrm{safe}}$ in the suboptimality bound.
\end{remark}

\subsection{Proof Sketch of Theorem~\ref{thm:safety}}
\label{sec:sketch-safety}

The full proof is given in Appendix~\ref{app:safety}.
We work on the high-probability event on which \(|\widehat C^t-C|\le\Xi^t\) and \(P\in\mathcal P^t\) for all \(t\in[T-1]\).
In the first episode, the constraint costs of $\widetilde\pi^1$ are each at most $d$.
The choice of $\psi^0$ therefore gives $(1-\psi^0)d+\psi^0\gamma_j\le\tau_j$ for every $j\in[n]$.
The same argument applies to the initial auxiliary policies.
For each later episode, \texttt{SafeProj} uses $\pi^\dagger$ if the projection in~\eqref{eq:proj-def} is infeasible.
Otherwise, the cost confidence bound and $P\in\mathcal P^t$ provide the cost bound for $\widetilde\pi^{t+1}$.
The choice of $\psi^t$ in~\eqref{eq:eta_t} then ensures $C^\top\theta^{P,\pi^{t+1}}\le\boldsymbol\tau$.
The same argument applies to each auxiliary policy $\nu_i^{t+1}$.
Since $\theta^{P,\pi^t}\in\Theta_{\mathrm{safe}}$ for every $t\in[T]$, convexity gives $\bar\theta\in\Theta_{\mathrm{safe}}$.
Thus, the induced Markov policy $\widehat\pi=\pi^{\bar\theta}$ is safe.

\subsection{Proof Sketch of Theorem~\ref{thm:pfcstrl_subopt}}
\label{sec:sketch-pfcstrl}

The complete proof of Theorem~\ref{thm:pfcstrl_subopt} is provided in Appendix~\ref{app:pfcstrl-subopt}. 
The analysis is carried out on the joint event defined in Appendix~\ref{app:pfcstrl-subopt}, which holds with probability at least \(1-\delta\).
Let $\bar\theta$ be the averaged occupancy measure returned by Algorithm~\ref{alg:pfcstrl-plan-rf}.
We use the optimal safe occupancy measure $\theta_{\boldsymbol\lambda}^*$ introduced in Section~\ref{sec:sketch-subopt}.
The first step is to replace the true single-objective optimal values \(\{V_i^*\}_{i=1}^m\) by the fixed optimistic single-objective values \(\{\widetilde{V}_i\}_{i=1}^m\) computed in the planning stage. 
By Lemma~\ref{lem:pf-fixed-aux-excess}, on the joint event, for all \(i\in[m]\), we have
$|
\widetilde{V}_i
-
V_i^*|
\le
\widetilde{\mathcal O}\!(
\phi_{\mathrm{safe}}H^4|\mathcal S|^{2}
\sqrt{|\mathcal A|/T}
+
H/\sqrt{K}).$
Since
$V_i(\bar\theta)
=
\frac1K\sum_{k=1}^K
V_i(\theta^{P,\pi^{k+1}})$
and
$V_i^{k+1}
=
\langle\bar r_i,\widetilde\theta^{k+1}\rangle$,
we obtain
\begin{align}
\mathrm{SubOpt}(\widehat\pi)
&\le
\underbrace{\frac1K\sum_{k=1}^K\sum_{i=1}^m
\lambda_i
\left|
V_i^{k+1}
-
V_i(\theta^{P,\pi^{k+1}})
\right|}_{\mathrm{(I)}}
\nonumber\\
&\quad +
\underbrace{
\max_{\boldsymbol w\in\Delta_m}
\frac1K\sum_{k=1}^K
\sum_{i=1}^m
\bigg[
(w_i-w_i^k)\lambda_i
\left(
\widetilde{V}_i+\iota
-V_i^{k+1}
\right) 
-\mu w_i\log w_i
+\mu w_i^k\log w_i^k
\bigg]
}_{\mathrm{(II)}}\nonumber\\
&\quad 
+\underbrace{\frac1K\sum_{k=1}^K
\sum_{i=1}^m
w_i^k\lambda_i
\left(
V_i(\theta_{\boldsymbol{\lambda}}^*)
-
V_i^{k+1}
\right)}_{\mathrm{(III)}}
+
\widetilde{\mathcal O}\!\left(
\phi_{\mathrm{safe}}H^4|\mathcal S|^{2}
\sqrt{|\mathcal A|/T}
+
H/\sqrt{K}
\right).
\label{eq:sketch-pf-main-decomp}
\end{align}
Term \(\mathrm{(I)}\) compares the optimistic main value \(V_i^{k+1}\) with the true value of the policy \(\pi^{k+1}\).
Terms \(\mathrm{(II)}\) and \(\mathrm{(III)}\) are the dual and primal planning errors, respectively.

\vspace{5pt}
\noindent\textbf{Bounding Term \(\mathrm{(I)}\).}
The reward confidence event gives \(0\le \bar r_i-r_i\le 2\sigma\). 
Since \(\widetilde\theta^{k+1}\in\Theta(\mathcal P)\), inequality~\eqref{eq:pf-transition-transfer} bounds the transition error.
By construction of \(\pi^{k+1}\), we have
$\theta^{P,\pi^{k+1}}
=
(1-\psi^k)\theta^{P,\widetilde\pi^{k+1}}
+
\psi^k\theta^{P,\pi^\dagger}.$
Combining this decomposition with the reward and transition bounds, we obtain
$|
V_i^{k+1}
-
V_i(\theta^{P,\pi^{k+1}})|
\le
2\langle \sigma,\widetilde\theta^{k+1}\rangle
+
2\sqrt{|\mathcal S|}\langle \zeta,\theta^{P,\widetilde\pi^{k+1}}\rangle
+
H\psi^k .$
After averaging over \(k\), Lemma~\ref{lem:pf-exploration-coverage} bounds the bonus terms, and Lemma~\ref{lem:pf-planning-mixing} bounds $\frac1K\sum_{k=1}^K\psi^k$.
Therefore, we have
$\mathrm{(I)}
\le
\widetilde{\mathcal O}\!(
\phi_{\mathrm{safe}}H^4|\mathcal S|^{2}
\sqrt{|\mathcal A|/T}).$

\vspace{5pt}
\noindent\textbf{Bounding Term \(\mathrm{(II)}\).}
For the planning stage, we define \(\widetilde{\boldsymbol V} := (\widetilde V_1,\ldots,\widetilde V_m)\).
For each \(k\in[K]\), we define
\(\boldsymbol V^{k+1}
:=
(V_1^{k+1},\ldots,V_m^{k+1})\).
Applying Lemma~\ref{lem:oco-stch} to the \(K\) dual updates with \(\boldsymbol\lambda\odot (\widetilde{\boldsymbol V}+\boldsymbol\iota-\boldsymbol V^{k+1})\) gives
$\mathrm{(II)}
\le
\widetilde{\mathcal O}\!(
\frac{H^2}{\mu K}
+
\frac{\mu}{K}).$

\vspace{5pt}
\noindent\textbf{Bounding Term \(\mathrm{(III)}\).}
Since \(\bar r_i\ge r_i\) and \(\widetilde\theta^{k+1}\ge 0\), 
we have
$V_i(\theta_{\boldsymbol{\lambda}}^*)-V_i^{k+1}
=
\langle r_i,\theta_{\boldsymbol{\lambda}}^*\rangle
-
\langle \bar r_i,\widetilde\theta^{k+1}\rangle
\le
\langle r_i,\theta_{\boldsymbol{\lambda}}^*-\widetilde\theta^{k+1}\rangle .$
Therefore, we obtain
$\mathrm{(III)}
\le
\frac1K\sum_{k=1}^K
\langle
\sum_{i=1}^m w_i^k\lambda_i r_i,\,
\theta_{\boldsymbol{\lambda}}^*-\widetilde\theta^{k+1}
\rangle .$
Since  \(\theta_{\boldsymbol{\lambda}}^*\in\Theta_{\mathrm{safe}}\), 
by Lemma~\ref{lem:pf-planning-occupancy} with \(\theta=\theta_{\boldsymbol{\lambda}}^*\),
we have
$\mathrm{(III)}
\le \allowbreak
\widetilde{\mathcal O}\!(
\frac{H}{\sqrt K}
+ \allowbreak
\phi_{\mathrm{safe}}\,H^3|\mathcal S|^{2}\sqrt{\frac{|\mathcal A|}{T}}).$

\vspace{5pt}
\noindent\textbf{Final Bound.}
Combining \eqref{eq:sketch-pf-main-decomp} and the bounds for Terms \(\mathrm{(I)}\), \(\mathrm{(II)}\), and \(\mathrm{(III)}\), 
we obtain
\[
\mathrm{SubOpt}(\widehat\pi)
\le
\widetilde{\mathcal O}\!\left(
\phi_{\mathrm{safe}}H^4|\mathcal S|^{2}
\sqrt{|\mathcal A|/T}
+
{H}/{\sqrt K}
+
\frac{H^2}{\mu K}
+
\frac{\mu}{K}
\right).
\]

\subsection{Proof Sketch of Theorem~\ref{thm:safety-pfcstrl}}
\label{sec:sketch-safety-pfcstrl}

The full proof is given in Appendix~\ref{app:safety-pfcstrl}.
We work on a joint high-probability event that combines the cost and transition confidence bounds over the \(T\) exploration episodes with the confidence bounds computed from the final exploration dataset.
\texttt{PF-CSTRL} constructs its initial policy as in \texttt{CSTRL} and uses the same \texttt{SafeProj} subroutine.
The proof of Theorem~\ref{thm:safety} therefore gives $C^\top\theta^{P,\pi^t}\le\boldsymbol\tau$ for every $t\in[T]$.
Since $(\widehat C,\Xi,\mathcal P)$ remain fixed during planning, the bounds $|\widehat C-C|\le\Xi$ and $P\in\mathcal P$ continue to hold.
The same \texttt{SafeProj} argument gives $C^\top\theta^{P,\pi^{k+1}}\le\boldsymbol\tau$ for every $k\in[K]$.
Since $\theta^{P,\pi^{k+1}}\in\Theta_{\mathrm{safe}}$ for every $k\in[K]$, convexity gives $\bar\theta\in\Theta_{\mathrm{safe}}$.
Thus, the induced Markov policy $\widehat\pi=\pi^{\bar\theta}$ is safe.

\subsection{Proof Sketch for the Unknown Constraint Cost Extensions}
The complete proofs are provided in Appendix~\ref{app:unknown-gamma}.
By Hoeffding's inequality and the stopping rule, we have \(t_s=\widetilde{\mathcal O}(d^2/(\varrho^\dagger)^2)\) and \(\gamma_j\le\widehat\gamma_j\le(\gamma_j+\tau_j)/2\) for every \(j\in[n]\) with high probability.
The assumed lower bound on $T$ then gives $t_s<T$.
The bound on $\widehat{\boldsymbol\gamma}$ also gives \(\widehat\phi_{\mathrm{safe}}\le4\phi_{\mathrm{safe}}\).
Every policy executed during the estimation phase is \(\pi^\dagger\) and is therefore safe.
Using \(\widehat{\boldsymbol\gamma}\) as a conservative upper bound, the analyses underlying Theorems~\ref{thm:subopt} and~\ref{thm:safety} apply to the remaining \(T-t_s\) episodes, yielding Theorem~\ref{thm:unknown-gamma-subopt} and Theorem~\ref{thm:unknown-gamma-safety}.
With the same conservative bound \(\widehat{\boldsymbol\gamma}\), the \texttt{PF-CSTRL} analyses apply to the remaining \(T-t_s\) exploration episodes and \(K\) planning iterations, giving Theorems~\ref{thm:pfcstrl-uc-subopt} and~\ref{thm:pfcstrl-uc-safety}.

\subsection{Proof Sketch for the Unknown Safe Policy Extensions}
The complete proofs are provided in Appendix~\ref{app:us}.
Lemma~\ref{lem:us_stop_time}, together with the condition on \(T\) stated above, gives \(T>t_s\) with high probability.
On the same event, Lemma~\ref{lem:us_delta_hat} gives \(\widehat\varrho\ge\varrho/2\), \(C^\top\theta^{P,\widehat\pi^\dagger} \le\widehat{\boldsymbol\gamma}<\boldsymbol\tau\), and \(\widehat\phi_{\mathrm{safe}}\le4d^2/\varrho^2\).
The estimated \(\widehat\pi^\dagger\) and \(\widehat{\boldsymbol\gamma}\) are then fixed.
The proof of Theorem~\ref{thm:subopt} applies to the remaining \(T-t_s\) episodes, giving Theorem~\ref{thm:us_subopt}.
Similarly, the proof of Theorem~\ref{thm:pfcstrl_subopt} applies to the remaining \(T-t_s\) exploration episodes and \(K\) planning iterations, giving Theorem~\ref{thm:pfcstrl-us-subopt}.
During the estimation phase, each executed policy has violation at most \(d\).
With the estimated \(\widehat\pi^\dagger\) and \(\widehat{\boldsymbol\gamma}\), the proofs of Theorems~\ref{thm:safety} and~\ref{thm:safety-pfcstrl} establish safety of all policies executed after estimation, the main policies generated during planning, and the returned policies.
Consequently, both extensions satisfy
$\sum_{t=1}^{T}\mathrm{Violation}(\pi^t)
\le
d t_s,$
which proves Theorems~\ref{thm:us_safety} and~\ref{thm:pfcstrl-us-safety}.

\section{Conclusion}
We develop a unified framework for safe multi-objective reinforcement learning that covers both cumulative and stepwise safety constraints.
Pessimistic cost bounds ensure safety throughout learning, while averaging the resulting occupancy measures produces a single safe Markov policy.
Separating exploration from planning also allows the learner to compute a policy for any given preference from the collected data, without further interaction with the environment.
In both preference-based and preference-free settings, zero constraint violation can be guaranteed when a strictly safe policy is known but its costs are unknown.
When the strictly safe policy itself is also unknown, the cumulative constraint violation remains bounded by a constant.

\newpage

\appendix

\startcontents[appendices]
\printcontents[appendices]{}{1}{
  \section*{Appendix Contents}
  \setcounter{tocdepth}{2}
}

\newpage

\section{Proof of Proposition~\ref{prop:safe_value_polytope}}
\label{app:safe_value_polytope}

\begin{proof}
We work with the state-action marginals $\theta_h(s,a)$ of an occupancy measure.
Since the transition kernel $P$ is fixed, these marginals determine the corresponding occupancy measure through $\theta_h(s,a,s')=\theta_h(s,a)P_h(s'|s,a)$.
Let $\theta=(\theta_h(s,a))_{h\in[H],\,s\in\mathcal S,\,a\in\mathcal A}$ be a vector in $\mathbb R^{H|\mathcal S||\mathcal A|}$, and consider the following constraints:
\begin{align}
&\theta_h(s,a)\ge 0,
&&
\forall h\in[H],\ s\in\mathcal S,\ a\in\mathcal A,
\label{eq:value_polytope_nonneg}
\\
&\sum_{a\in\mathcal A}\theta_1(s,a)
=
\rho(s),
&&
\forall s\in\mathcal S,
\label{eq:value_polytope_init}
\\
&\sum_{s\in\mathcal S}\sum_{a\in\mathcal A}
\theta_h(s,a)P_h(s'|s,a)
=
\sum_{a'\in\mathcal A}\theta_{h+1}(s',a'),
&&
\forall h\in[H-1],\ s'\in\mathcal S.
\label{eq:value_polytope_flow}
\end{align}
We show below that $\theta\in\Theta(P)$ if and only if its marginals satisfy \eqref{eq:value_polytope_nonneg}--\eqref{eq:value_polytope_flow}.

We next establish the correspondence between stochastic Markov policies and the elements of $\Theta(P)$.
For any policy $\pi\in\Pi$, its state-action occupancy measure $\theta^{P,\pi}$ satisfies \eqref{eq:value_polytope_nonneg} and \eqref{eq:value_polytope_init}.
Moreover, since
$\sum_{s\in\mathcal S}\sum_{a\in\mathcal A}
\theta_h^{P,\pi}(s,a)P_h(s'| s,a)
=
\Pr(s_{h+1}=s'|\rho,\pi,P)
=
\sum_{a'\in\mathcal A}
\theta_{h+1}^{P,\pi}(s',a'),$
\eqref{eq:value_polytope_flow} also holds.
Since $\Theta(P)=\{\theta^{P,\pi}:\pi\in\Pi\}$, every $\theta\in\Theta(P)$ satisfies \eqref{eq:value_polytope_nonneg}--\eqref{eq:value_polytope_flow}.

Conversely, take any $\theta$ satisfying \eqref{eq:value_polytope_nonneg}--\eqref{eq:value_polytope_flow}.
We define its state marginal as $\theta_h(s):=\sum_{a\in\mathcal A}\theta_h(s,a)$.
We construct a Markov policy $\pi^\theta=\{\pi_h^\theta\}_{h=1}^H$ as follows.
For each $h\in[H]$ and $s\in\mathcal S$, if $\theta_h(s)>0$, we set $\pi_h^\theta(a|s):=\theta_h(s,a)/\theta_h(s)$ for every $a\in\mathcal A$.
If $\theta_h(s)=0$, we set $\pi_h^\theta(\cdot|s)$ to the uniform distribution over $\mathcal A$.
We prove by induction that $\theta_h^{P,\pi^\theta}(s,a)=\theta_h(s,a)$ for all $h\in[H]$, $s\in\mathcal S$, and $a\in\mathcal A$.
For $h=1$, \eqref{eq:value_polytope_init} gives $\theta_1(s)=\rho(s)$.
If $\theta_1(s)>0$, then
$\theta_1^{P,\pi^\theta}(s,a)
=\rho(s)\pi_1^\theta(a| s)
=\theta_1(s,a)$.
If $\theta_1(s)=0$, nonnegativity implies $\theta_1(s,a)=0$ for every $a\in\mathcal A$.
Since $\rho(s)=0$, we also have $\theta_1^{P,\pi^\theta}(s,a)=0$.
Thus, the induction claim holds for $h=1$.

Suppose that it holds at some $h\in[H-1]$.
For every $s'\in\mathcal S$,
we have
\begin{align*}
\Pr(s_{h+1}=s'|\rho,\pi^\theta,P)
=
\sum_{s\in\mathcal S}\sum_{a\in\mathcal A}
\theta_h(s,a)P_h(s'| s,a)
=
\sum_{a'\in\mathcal A}\theta_{h+1}(s',a')
=
\theta_{h+1}(s'),
\end{align*}
where the second equality follows from~\eqref{eq:value_polytope_flow}.
If $\theta_{h+1}(s')>0$, then
$\theta_{h+1}^{P,\pi^\theta}(s',a')
=\Pr(s_{h+1}=s'|\rho,\pi^\theta,P)
\pi_{h+1}^\theta(a'| s')
=\theta_{h+1}(s')
\theta_{h+1}(s',a')/\theta_{h+1}(s')
=\theta_{h+1}(s',a')$.
If $\theta_{h+1}(s')=0$, both
$\theta_{h+1}^{P,\pi^\theta}(s',a')$
and $\theta_{h+1}(s',a')$ are zero.
The induction is therefore complete.
Since \(\theta_h(s,a,s')=\theta_h(s,a)P_h(s'|s,a)\), we obtain \(\theta^{P,\pi^\theta}=\theta\).
Consequently, every $\theta\in\Theta(P)$ is induced by a stochastic Markov policy, and $\Theta(P)=\{\theta^{P,\pi}:\pi\in\Pi\}$.
Since Definition~\ref{def:occupancy-measure} requires $\theta_h\in\Delta_{|\mathcal S|\times|\mathcal A|\times|\mathcal S|}$, the marginals satisfy $\sum_{s,a}\theta_h(s,a)=1$ for every $h\in[H]$.
Thus, $\Theta(P)$ is defined by finitely many linear equalities and inequalities. 
Since it is bounded, $\Theta(P)$ is a convex polytope.

We next impose the safety constraints.
The safe occupancy measure set admits the following linear characterization:
\begin{equation}
\Theta_{\mathrm{safe}}
=
\left\{
\theta\in\Theta(P):
\sum_{h=1}^H
\sum_{s\in\mathcal S}
\sum_{a\in\mathcal A}
\theta_h(s,a)c_{j,h}(s,a)
\le\tau_j,\ 
\forall j\in[n]
\right\}.
\label{eq:safe_occupancy_polytope}
\end{equation}
The occupancy measure of any $\pi\in\Pi_{\mathrm{safe}}$ belongs to $\Theta(P)$ and satisfies the inequalities in \eqref{eq:safe_occupancy_polytope}.
Conversely, if $\theta$ belongs to the set on the right-hand side of~\eqref{eq:safe_occupancy_polytope}, the policy $\pi^\theta$ constructed above satisfies $\theta^{P,\pi^\theta}=\theta$.
It therefore satisfies $C^\top\theta^{P,\pi^\theta}\le\boldsymbol\tau$ and belongs to $\Pi_{\mathrm{safe}}$.
This proves \eqref{eq:safe_occupancy_polytope}.
Since \eqref{eq:safe_occupancy_polytope} adds finitely many linear inequalities to the constraints defining $\Theta(P)$, $\Theta_{\mathrm{safe}}$ is a convex polytope.
Assumption~\ref{ass:slater} guarantees that it is nonempty.
Moreover, the correspondence between policies and occupancy measures gives
$\mathcal V_{\mathrm{safe}}
=\{\boldsymbol V(\theta^{P,\pi}):\pi\in\Pi_{\mathrm{safe}}\}
=\{\boldsymbol V(\theta):\theta\in\Theta_{\mathrm{safe}}\}$.

Since $\boldsymbol V(\cdot)$ is linear,
$\mathcal V_{\mathrm{safe}}$ is the linear image of
$\Theta_{\mathrm{safe}}$.
The linear image of a polytope is a polytope, since a polytope is the convex hull of finitely many points and linear mappings preserve convex combinations.
It is nonempty since $\Theta_{\mathrm{safe}}$ is nonempty.
Therefore, $\mathcal V_{\mathrm{safe}}$ is a nonempty convex polytope.
This completes the proof.
\end{proof}

\section{Proof of Proposition~\ref{prop:stch}}
\label{app:stch}

\begin{proof}
Under Assumption~\ref{ass:slater},
Proposition~\ref{prop:safe_value_polytope} ensures that
$\mathcal V_{\mathrm{safe}}$ is a nonempty convex polytope.
Fix any $\mu>0$.

For any $\boldsymbol{\lambda}\in\Delta_m^o$, consider a policy
$\pi^*
\in
\argmin_{\pi\in\Pi_{\mathrm{safe}}}
\stchlm(\pi).$
Suppose for contradiction that $\pi^*\notin\Pi_{\mathrm{safe}}^*$.
Definition~\ref{def:pareto} then implies that there exists a policy $\pi'\in\Pi_{\mathrm{safe}}$ satisfying $V_i(\theta^{P,\pi'}) \ge V_i(\theta^{P,\pi^*})$ for every $i\in[m]$, with strict inequality for at least one $i_0\in[m]$.
Since $\boldsymbol{\lambda}\in\Delta_m^o$, we have $\lambda_i>0$ for every $i\in[m]$.
Consequently, each exponential term in $\stchlm(\pi')$ is no larger than its counterpart under $\pi^*$, and the term corresponding to $i_0$ is strictly smaller.
We therefore obtain $\stchlm(\pi')<\stchlm(\pi^*)$.
This contradicts the optimality of $\pi^*$.
Thus, every minimizer is safe Pareto optimal.

Conversely, consider any $\pi^0\in\Pi_{\mathrm{safe}}^*$ and let $\boldsymbol v^0:=\boldsymbol V(\theta^{P,\pi^0})$.
Then $\boldsymbol v^0$ is a Pareto optimal point of $\mathcal V_{\mathrm{safe}}$.
Since $\mathcal V_{\mathrm{safe}}$ is a convex polytope, a standard polyhedral supporting-hyperplane argument gives a vector $\boldsymbol a=(a_1,\ldots,a_m)\in\mathbb R_{++}^m$ and a constant $\kappa\in\mathbb R$ satisfying $\boldsymbol a^\top\boldsymbol v^0=\kappa$ and $\boldsymbol a^\top\boldsymbol v\le\kappa$ for every $\boldsymbol v\in\mathcal V_{\mathrm{safe}}.$
We regard $\stchlm$ as a function on $\mathbb R^m$ by replacing $V_i(\theta^{P,\pi})$ in Definition~\ref{def:stch} with $v_i$ for any $\boldsymbol v=(v_1,\ldots,v_m)\in\mathbb R^m$.
For every $i\in[m]$, the definition of $V_{i}^*$ gives $V_{i}^*+\iota-v_i^0\ge\iota>0$.

We next construct a preference vector for which $\boldsymbol v^0$ minimizes the scalarization over the supporting hyperplane.
For each $i\in[m]$, the function $x\mapsto x\exp(x(V_{i}^*+\iota-v_i^0)/\mu)$ is continuous and strictly increasing on $[0,\infty)$.
It equals zero at $x=0$ and tends to infinity as $x\to\infty$.
Therefore, for every $w>0$, we define $\lambda_i(w)>0$ as the unique solution to
$\lambda_i(w)
\exp(
\frac{\lambda_i(w)
(V_{i}^*+\iota-v_i^0)}
{\mu})=\mu w a_i.$
Since each of these functions is continuous and strictly increasing, its inverse is also continuous and strictly increasing.
Moreover, we have $\lambda_i(w)\to0$ as $w\to0^+$ and $\lambda_i(w)\to\infty$ as $w\to\infty$.
Consequently, the mapping $w\mapsto\sum_{i=1}^m\lambda_i(w)$ is continuous and strictly increasing from zero to infinity.
The intermediate value theorem gives a unique $w^*>0$ satisfying $\sum_{i=1}^m\lambda_i(w^*)=1$.
We define $\boldsymbol{\lambda}^0 :=(\lambda_1(w^*),\ldots,\lambda_m(w^*))$.
Then $\boldsymbol{\lambda}^0\in\Delta_m^o$, and we have
\begin{equation}
\lambda_i^0
\exp\left(
\frac{\lambda_i^0
(V_{i}^*+\iota-v_i^0)}
{\mu}
\right)
=
\mu w^*a_i,
\quad
\forall i\in[m],
\label{eq:stch_foc}
\end{equation}
where $\lambda_i^0:=\lambda_i(w^*)$.

Since the logarithm is strictly increasing, minimizing $\mathrm{STCH}^{\mu}_{\boldsymbol{\lambda}^0}$ is equivalent to minimizing its exponential sum.
This exponential sum is strictly convex in the value vector.
Equation~\eqref{eq:stch_foc} is precisely the first-order optimality condition at $\boldsymbol v^0$ for minimizing this sum subject to $\boldsymbol a^\top\boldsymbol v=\kappa$.
Since the exponential sum is strictly convex and $\boldsymbol a^\top\boldsymbol v=\kappa$ is a linear equality constraint, this first-order condition is sufficient for global optimality.
Therefore, $\boldsymbol v^0$ minimizes
$\mathrm{STCH}^{\mu}_{\boldsymbol{\lambda}^0}$
over the supporting hyperplane.

It remains to extend this conclusion to $\mathcal V_{\mathrm{safe}}$.
Consider any $\boldsymbol v\in\mathcal V_{\mathrm{safe}}$.
The supporting-hyperplane inequality gives \(\frac{\kappa-\boldsymbol a^\top\boldsymbol v} {\sum_{i=1}^m a_i}\ge0\).
Let \(\boldsymbol 1\in\mathbb R^m\) denote the all-ones vector.
We further define
\(\boldsymbol u
:=
\boldsymbol v+
\frac{\kappa-\boldsymbol a^\top\boldsymbol v}
{\sum_{i=1}^m a_i}\,\boldsymbol 1\).
Then $\boldsymbol a^\top\boldsymbol u=\kappa$ and $\boldsymbol u\ge\boldsymbol v$ componentwise.
Since every component of $\boldsymbol{\lambda}^0$ is positive, the scalarization is strictly decreasing in each coordinate.
We therefore have
$\mathrm{STCH}^{\mu}_{\boldsymbol{\lambda}^0}
(\boldsymbol v^0)
\le
\mathrm{STCH}^{\mu}_{\boldsymbol{\lambda}^0}
(\boldsymbol u)
\le
\mathrm{STCH}^{\mu}_{\boldsymbol{\lambda}^0}
(\boldsymbol v).$
Thus, $\boldsymbol v^0$ is a minimizer over $\mathcal V_{\mathrm{safe}}$.
Equivalently, we obtain
$\pi^0
\in
\argmin_{\pi\in\Pi_{\mathrm{safe}}}
\mathrm{STCH}^{\mu}_{\boldsymbol{\lambda}^0}(\pi).$
This completes the proof.
\end{proof}

\section{Proof of Proposition~\ref{prop:stch-unique}}
\label{app:stch_unique}

\begin{proof}
Under Assumption~\ref{ass:slater},
Proposition~\ref{prop:safe_value_polytope} ensures that
$\mathcal V_{\mathrm{safe}}$ is a nonempty convex polytope.
Since $\mathcal V_{\mathrm{safe}}$ is compact and the scalarization is continuous in the value vector, its minimum over $\mathcal V_{\mathrm{safe}}$ is attained.
Consider any two policies
$\pi^1,\pi^2
\in
\argmin_{\pi\in\Pi_{\mathrm{safe}}}
\stchlm(\pi).$
For each $k\in\{1,2\}$, we define
$\boldsymbol v^k:=\boldsymbol V(\theta^{P,\pi^k})$ and
$\boldsymbol z^k\in\mathbb R^m$ by
$z_i^k
:=
\lambda_i
(V_{i}^*+\iota-v_i^k)/\mu$
for every $i\in[m]$.
Since $\mathcal V_{\mathrm{safe}}$ is convex, we have
$\bar{\boldsymbol v}
:=
(\boldsymbol v^1+\boldsymbol v^2)/2
\in\mathcal V_{\mathrm{safe}}$.
Therefore, there exists a policy
$\bar\pi\in\Pi_{\mathrm{safe}}$ satisfying
$\boldsymbol V(\theta^{P,\bar\pi})=\bar{\boldsymbol v}$.
The argument of the log-sum-exp function at \(\bar\pi\) is $(\boldsymbol z^1+\boldsymbol z^2)/2$.
The mapping $\boldsymbol z\mapsto \mu\log\bigl(\sum_{i=1}^m\exp(z_i)\bigr)$ is convex.
Its Hessian has null space spanned by $\boldsymbol 1$. 
Therefore, the mapping is strictly convex along every segment whose direction is not a multiple of $\boldsymbol 1$.
If $\boldsymbol z^1-\boldsymbol z^2$ is not a multiple of $\boldsymbol 1$, strict convexity gives
\[
\stchlm(\bar\pi)
<
\frac{
\stchlm(\pi^1)
+
\stchlm(\pi^2)
}{2}
=
\stchlm(\pi^1),
\] which contradicts the optimality of $\pi^1$ and $\pi^2$.

It remains to consider the case $\boldsymbol z^1-\boldsymbol z^2=\alpha\boldsymbol 1$ for some $\alpha\in\mathbb R$.
The translation identity of the log-sum-exp function gives
$\stchlm(\pi^1)
=
\stchlm(\pi^2)
+\mu\alpha$.
Since $\pi^1$ and $\pi^2$ are both minimizers, their scalarization values are equal.
Therefore, we obtain $\alpha=0$ and $\boldsymbol z^1=\boldsymbol z^2$.
Since $\boldsymbol{\lambda}\in\Delta_m^o$, we have $\lambda_i>0$ for every $i\in[m]$.
It follows that $v_i^1=v_i^2$ for every $i\in[m]$.
Thus, all minimizers over $\Pi_{\mathrm{safe}}$ achieve the same values on all objectives.
This completes the proof.
\end{proof}

\section{Proof of Theorem~\ref{thm:subopt}}
\label{app:cstrl-subopt}
We define the transition event
$\mathcal E_T(\delta)
\!:=\!
\{ P \!\in\! \mathcal{P}^t \text{ for all } t\!\in\![T-1]\},$
which holds with probability at least \(1-\delta\) by Lemma~\ref{lem:confidenceset}.
We also define the cost event
$\mathcal E_C(\delta)
:=
\{|\widehat C^t-C|\le\Xi^t
\text{ for all }t\in[T-1]\},$
which holds with probability at least $1-\delta$ by Lemma~\ref{lem:gen_confidence}.
We further define the reward event
$\mathcal E_R(\delta)
:=
\{
|\widehat r_{i,h}^t(s,a)-r_{i,h}(s,a)|
\le
\sigma_h^t(s,a),
\text{ for all } i\in[m], h\in[H], t\in[T-1], (s,a)\in\mathcal S\times\mathcal A\}.$
By Lemma~\ref{lem:reward_concentration}, \(\mathcal E_R(\delta)\) holds with probability at least \(1-\delta\).
Moreover, we define the joint confidence event
$\mathcal E_{R,C,T}(\delta)
:=
\mathcal E_R(\delta/3)
\cap
\mathcal E_C(\delta/3)
\cap
\mathcal E_T(\delta/3).$
By a union bound, this event holds with probability at least \(1-\delta\).
The analysis also uses the martingale events introduced in the proofs of Lemmas~\ref{lem:weighted-transition-bridge} and~\ref{lem:main-occ-regret}.
Throughout the proofs of Theorems~\ref{thm:subopt} and~\ref{thm:safety}, we work on the intersection of the events \(\mathcal E_{R,C,T}\) for the main trajectory sequence and each auxiliary trajectory sequence, together with the corresponding martingale events.
There are \(\mathcal O(m)\) such events, and we divide the failure probability \(\delta\) equally among them.
By a union bound, this intersection holds with probability at least \(1-\delta\), and this allocation affects only \(\ln(m/\delta)\) factors, which are already absorbed into the \(\widetilde{\mathcal O}(\cdot)\) notation.

\subsection{Lemmas}

\begin{lemma}
\label{lem:confidenceset}
We define the transition confidence set at episode $t$ as
\[
\mathcal{P}^t
\!:=\!
\left\{
P'\!:\!
\bigl|P_h'(s'| s,a) - \widehat{P}_h^t(s'| s,a)\bigr|
\!\leq\!
\sqrt{\frac{2\ln\!\bigl(2mTH|\mathcal{S}|^2|\mathcal{A}|/\delta\bigr)}
{\max\{1,N_h^t(s,a)\}}},
\forall h\in[H],s,s'\in\mathcal{S},a\in\mathcal{A}
\right\}.
\]
With probability at least $1-\delta$, the true transition kernel $P$ belongs to $\mathcal{P}^t$ for all $t\in[T]$.
\end{lemma}

\begin{proof}
Fix any $(h,s,a,s')\in[H]\times\mathcal S\times\mathcal A\times\mathcal S$.
Conditional on the history before each visit to $(s,a)$ at step $h$, the next state is $s'$ with probability $P_h(s'|s,a)$.
At a fixed step $h$, the pair $(s,a)$ can be visited at most once in each episode.
Therefore, we have $N_h^t(s,a)\le T$.
Applying the Azuma-Hoeffding inequality for each possible positive number of visits and taking a union bound over these possibilities and all $(h,s,a,s')$ yields, with probability at least $1-\delta$,
\[
\bigl|
P_h(s'|s,a)-\widehat P_h^{\,t}(s'|s,a)
\bigr|
\le
\sqrt{
\frac{
2\ln\!\bigl(2mTH|\mathcal S|^2|\mathcal A|/\delta\bigr)
}{
N_h^t(s,a)
}
}
\]
simultaneously for every $t\in[T]$ such that $N_h^t(s,a)\ge1$.
When $N_h^t(s,a)=0$, the claim follows since both $P_h(s'|s,a)$ and $\widehat P_h^{\,t}(s'|s,a)$ belong to $[0,1]$ and $\epsilon_h^t(s,a)>1$.
Therefore, we have $P\in\mathcal P^t$ for all $t\in[T]$.
This completes the proof.
\end{proof}

\begin{lemma}
\label{lem:occ-polytope}
For every $t\in[T-1]$, $\Theta(\mathcal P^t)$ is exactly the set of
$\theta=\{\theta_h\}_{h=1}^H$ such that
$\theta_h\in\Delta_{|\mathcal S|\times|\mathcal A|\times|\mathcal S|}$ for every
$h\in[H]$ and
\begin{align*}
\sum_{a\in\mathcal A}\sum_{s'\in\mathcal S}
\theta_1(s,a,s')
&=
\rho(s),
&&\forall s\in\mathcal S,
\\
\sum_{a\in\mathcal A}\sum_{s'\in\mathcal S}
\theta_h(s,a,s')
&=
\sum_{s'\in\mathcal S}\sum_{a\in\mathcal A}
\theta_{h-1}(s',a,s),
&&\forall h\in[2,H],\ \forall s\in\mathcal S,
\end{align*}
and
\[
\bigl|
\theta_h(s,a,s')
-
\widehat P_h^{\,t}(s'|s,a)\theta_h(s,a)
\bigr|
\le
\epsilon_h^t(s,a)\theta_h(s,a),
\quad
\forall h\in[H],\ 
(s,a,s')\in\mathcal S\times\mathcal A\times\mathcal S.
\]
In particular, $\Theta(\mathcal P^t)$ is a convex polytope.
\end{lemma}

\begin{proof}
Fix any $t\in[T-1]$.
We first establish the stated characterization of $\Theta(\mathcal P^t)$.
Consider any $\theta\in\Theta(\mathcal P^t)$.
By definition, there exist $P'\in\mathcal P^t$ and $\pi\in\Pi$ such that $\theta=\theta^{P',\pi}$.
Definition~\ref{def:occupancy-measure} gives nonnegative entries for each $\theta_h$ whose sum is one, together with the two equalities displayed in the lemma.
It also gives
$\theta_h(s,a,s')
=
P_h'(s'|s,a)\theta_h(s,a).$
Since $P'\in\mathcal P^t$, we have
$\bigl|
P_h'(s'|s,a)-\widehat P_h^{\,t}(s'|s,a)
\bigr|
\le
\epsilon_h^t(s,a).$
Multiplying this inequality by $\theta_h(s,a)\ge0$ yields
$|\theta_h(s,a,s')
-
\widehat P_h^{\,t}(s'|s,a)\theta_h(s,a)|\le
\epsilon_h^t(s,a)\theta_h(s,a).$
Thus, every element of $\Theta(\mathcal P^t)$ satisfies the constraints in the lemma statement.

Conversely, consider any $\theta$ satisfying these constraints.
First, $\mathcal P^t$ is nonempty.
When $N_h^t(s,a)\!\ge\!1$, the empirical estimate $\widehat P_h^{\,t}(\cdot|s,a)$ is a probability distribution.
When $N_h^t(s,a)=0$, we have
$\widehat P_h^{\,t}(s'|s,a)=0$ for every $s'\in\mathcal S$ and
$\epsilon_h^t(s,a)>1$, 
and therefore any probability distribution satisfies the corresponding inequalities.
We may therefore fix some $\bar P^t\in\mathcal P^t$.
We define a transition kernel $P'$ by setting $P_h'(s'|s,a):=\theta_h(s,a,s')/\theta_h(s,a)$ whenever $\theta_h(s,a)>0$.
If $\theta_h(s,a)=0$, we set $P_h'(s'|s,a):=\bar P_h^{\,t}(s'|s,a)$.
When $\theta_h(s,a)>0$, the nonnegativity of $\theta_h(s,a,s')$ and $\sum_{s'\in\mathcal S}\theta_h(s,a,s')=\theta_h(s,a)$ imply that $P_h'(\cdot|s,a)$ is a probability distribution.
Dividing the inequality in the lemma statement by $\theta_h(s,a)$ gives $|P_h'(s'|s,a)-\widehat P_h^{\,t}(s'|s,a)| \le \epsilon_h^t(s,a).$
When $\theta_h(s,a)=0$, the same inequality follows from $\bar P^t\in\mathcal P^t$.
Since the inequalities defining $\mathcal P^t$ are imposed separately for every $(h,s,a)$, we conclude that $P'\in\mathcal P^t$.

By construction, we have
$\theta_h(s,a,s')
=
P_h'(s'|s,a)\theta_h(s,a)$
whenever $\theta_h(s,a)>0$.
If $\theta_h(s,a)=0$, the nonnegativity of $\theta_h$ and the identity
$\theta_h(s,a)=\sum_{s'\in\mathcal S}\theta_h(s,a,s')$ imply that
$\theta_h(s,a,s')=0$ for every $s'\in\mathcal S$.
The same equality therefore holds for all $(h,s,a,s')$.
Let $\pi^\theta$ be the policy induced by $\theta$, 
with an arbitrary action distribution at any $(h,s)$ satisfying
$\sum_{a\in\mathcal A}\theta_h(s,a)=0$.
For every $h\in[H-1]$ and $s'\in\mathcal S$, we have
$\sum_{s\in\mathcal S}\sum_{a\in\mathcal A}
\theta_h(s,a)P_h'(s'|s,a)
=
\sum_{s\in\mathcal S}\sum_{a\in\mathcal A}
\theta_h(s,a,s')
=
\sum_{a'\in\mathcal A}\theta_{h+1}(s',a'),$
where the last equality is the second equality displayed in the lemma.
Together with the first equality displayed in the lemma, these are exactly the equations used in the occupancy measure construction in Appendix~\ref{app:safe_value_polytope}, with $P'$ in place of $P$.
The same induction therefore gives
$\theta_h^{P',\pi^\theta}(s,a)
=
\theta_h(s,a)$
for every $h\in[H]$ and $(s,a)\in\mathcal S\times\mathcal A$.
Consequently, we have
$\theta_h^{P',\pi^\theta}(s,a,s')
=
P_h'(s'|s,a)\theta_h^{P',\pi^\theta}(s,a)
=
\theta_h(s,a,s'),$
which proves that $\theta=\theta^{P',\pi^\theta}\in\Theta(\mathcal P^t)$.
The stated set identity follows.

It remains to show that $\Theta(\mathcal P^t)$ is a convex polytope.
Since $\theta_h(s,a)=\sum_{s'\in\mathcal S}\theta_h(s,a,s')$, the inequality in the lemma statement is equivalent to two linear inequalities in $\theta$.
The simplex constraints and the two equalities in the lemma statement are linear.
Since each $\theta_h$ belongs to a probability simplex, all components are bounded.
Therefore, $\Theta(\mathcal P^t)$ is a convex polytope.
This completes the proof.
\end{proof}

\begin{lemma}
\label{lem:gen_confidence}
With probability at least $1-\delta$, we have
$|\widehat C^t-C|
\le
\Xi^t,\,
\forall t\in[T].$
\end{lemma}

\begin{proof}
Fix any $j\in[n]$, any $h\in[H]$ to which constraint $j$ applies, and any $(s,a)\in\mathcal S\times\mathcal A$.
Conditional on the history before each visit to $(s,a)$ at step $h$, the observed cost has mean $c_{j,h}(s,a)$.
At a fixed step $h$, the pair $(s,a)$ can be visited at most once in each episode.
Therefore, we have $N_h^t(s,a)\le T$.
Applying the Azuma-Hoeffding inequality for each possible positive number of visits and taking a union bound over these possibilities, all $j\in[n]$, the $d$ applicable time steps, and all $(s,a)\in\mathcal S\times\mathcal A$ yields, with probability at least $1-\delta$,
$\bigl|
\widehat c_{j,h}^{\,t}(s,a)-c_{j,h}(s,a)
\bigr|
\le
\sqrt{
\frac{
\ln\!\bigl(4mndT|\mathcal S||\mathcal A|/\delta\bigr)}{2N_h^t(s,a)}}$
simultaneously for every $t\in[T]$ such that $N_h^t(s,a)\ge1$.
Since both terms on the left belong to $[0,1]$, clipping the right-hand side at one preserves the inequality.
When $N_h^t(s,a)=0$, the inequality follows from $\xi_{j,h}^t(s,a)=1$.
Therefore, we have $|\widehat C^t-C|\le\Xi^t$ for all $t\in[T]$.
This completes the proof.
\end{proof}

\begin{lemma}
\label{lem:reward_concentration}
With probability at least \(1-\delta\), for every \(i\in[m]\), \(h\in[H]\), \(t\in[T]\), and \((s,a)\in\mathcal S\times\mathcal A\), it holds that
\[|\widehat r_{i,h}^t(s,a)-r_{i,h}(s,a)|
\le
\sigma_h^t(s,a),\]
where
$\sigma_h^t(s,a)
=
\min\{
1,
\sqrt{
\frac{2\ln\!(8m|\mathcal S||\mathcal A|HT/\delta)}
{\max\{1,N_h^t(s,a)\}}}\}.$
\end{lemma}

\begin{proof}
Fix any $i\in[m]$, $h\in[H]$, and $(s,a)\in\mathcal S\times\mathcal A$.
Conditional on the history before each visit to $(s,a)$ at step $h$, the observed reward has mean $r_{i,h}(s,a)$.
At a fixed step $h$, the pair $(s,a)$ can be visited at most once in each episode.
Therefore, we have $N_h^t(s,a)\le T$.
Applying the Azuma-Hoeffding inequality for each possible positive number of visits and taking a union bound over these possibilities, all $i\in[m]$, all $h\in[H]$, and all $(s,a)\in\mathcal S\times\mathcal A$ yields, with probability at least $1-\delta$,
\[
\bigl|
\widehat r_{i,h}^{\,t}(s,a)-r_{i,h}(s,a)
\bigr|
\le
\sqrt{
\frac{
2\ln\!\bigl(8m|\mathcal S||\mathcal A|HT/\delta\bigr)
}{
N_h^t(s,a)
}
}
\]
simultaneously for every $t\in[T]$ such that $N_h^t(s,a)\ge1$.
Since both terms on the left belong to $[0,1]$, clipping the right-hand side at one preserves the inequality.
When $N_h^t(s,a)=0$, the inequality follows from $\sigma_h^t(s,a)=1$.
Therefore, the claimed bound holds for every $i\in[m]$, $h\in[H]$, $t\in[T]$, and $(s,a)\in\mathcal S\times\mathcal A$.
This completes the proof.
\end{proof}

\begin{lemma}
\label{lem:simulation}
For any transition kernel \(P'\) and any policy \(\pi\in\Pi\), we have
\[
\bigl\|
\theta^{P',\pi}
-
\theta^{P,\pi}
\bigr\|_1
\le
H\sum_{h=1}^H\sum_{s,a}
\theta_h^{P,\pi}(s,a)
\bigl\|P'_h(\cdot|s,a)-P_h(\cdot|s,a)\bigr\|_1.
\]
\end{lemma}

\begin{proof}
Fix \(h\in[H]\).
By Definition~\ref{def:occupancy-measure}, we have \(\theta_h^{P',\pi}(s,a,s')=\theta_h^{P',\pi}(s,a)P'_h(s'|s,a)\) and \(\theta_h^{P,\pi}(s,a,s')=\theta_h^{P,\pi}(s,a)P_h(s'|s,a)\).
The triangle inequality therefore gives
\begin{align}
\bigl\|\theta_h^{P',\pi}-\theta_h^{P,\pi}\bigr\|_1
&\le
\sum_{s,a,s'}
\bigl|\theta_h^{P',\pi}(s,a)-\theta_h^{P,\pi}(s,a)\bigr|
P'_h(s'|s,a)
\nonumber\\
&\quad+
\sum_{s,a,s'}
\theta_h^{P,\pi}(s,a)
\bigl|P'_h(s'|s,a)-P_h(s'|s,a)\bigr|
\nonumber\\
&=
\sum_{s,a}
\bigl|\theta_h^{P',\pi}(s,a)-\theta_h^{P,\pi}(s,a)\bigr|
\nonumber\\
&\quad+
\sum_{s,a}
\theta_h^{P,\pi}(s,a)
\bigl\|P'_h(\cdot|s,a)-P_h(\cdot|s,a)\bigr\|_1,
\label{eq:simulation-step}
\end{align}
where the left-hand side is the \(L_1\) distance between the joint distributions of \((s_h,a_h,s_{h+1})\).
For \(h\le H-1\), summing over the state and action at step \(h\) gives
\[\theta_{h+1}^{P',\pi}(s',a')
=
\pi_{h+1}(a'|s')
\sum_{s,a}
\theta_h^{P',\pi}(s,a,s')\]
for every \((s',a')\in\mathcal S\times\mathcal A\), and the same identity holds for \(P\).
Since \(\sum_{a'\in\mathcal A}\pi_{h+1}(a'|s')=1\), we obtain
$\sum_{s',a'}
|\theta_{h+1}^{P',\pi}(s',a')-\theta_{h+1}^{P,\pi}(s',a')|
\le
\|\theta_h^{P',\pi}-\theta_h^{P,\pi}\|_1 .$
Moreover, we have \(\theta_1^{P',\pi}(s,a)=\theta_1^{P,\pi}(s,a)=\rho(s)\pi_1(a|s)\).
Combining these facts with~\eqref{eq:simulation-step}, induction on \(h\) gives
$\|\theta_h^{P',\pi}-\theta_h^{P,\pi}\|_1
\le
\sum_{k=1}^{h}\sum_{s,a}
\theta_k^{P,\pi}(s,a)
\|P'_k(\cdot|s,a)-P_k(\cdot|s,a)\|_1$
for every \(h\in[H]\).
Bounding the right-hand side by the corresponding sum over all \(k\in[H]\) and summing over \(h\in[H]\) proves the lemma.
\end{proof}

\begin{lemma}
\label{lem:weighted-transition-bridge}
With probability at least $1-\delta$, for any sequence $P^{t+1}\in\mathcal P^t$ and the policies generated by Algorithm~\ref{alg:CSTRL}, we have
\begin{align*}
\frac1T\sum_{t=1}^{T-1}(1-\psi^t)
\bigl\|
\theta^{P^{t+1},\widetilde\pi^{t+1}}
-\theta^{P,\widetilde\pi^{t+1}}
\bigr\|_1
\le
\widetilde{\mathcal O}\!\left(
H^2|\mathcal S|^{3/2}\sqrt{|\mathcal A|/T}
\right).
\end{align*}
Consequently, we have
\begin{align*}
\frac1T\sum_{t=1}^{T-1}
\bigl\|
\theta^{P^{t+1},\widetilde\pi^{t+1}}
-\theta^{P,\widetilde\pi^{t+1}}
\bigr\|_1
\le
\widetilde{\mathcal O}\!\left(
H^2|\mathcal S|^{3/2}\sqrt{|\mathcal A|/T}
\right)
+\frac{2H}{T}\sum_{t=1}^{T-1}\psi^t .
\end{align*}
\end{lemma}

\begin{proof}
We work on the transition confidence event on which $P\in\mathcal P^t$ for all $t\in[T-1]$. 
Since both $P$ and $P^{t+1}$ belong to $\mathcal P^t$, by the definition of $\mathcal P^t$ in~\eqref{eq:transition_confidence}, for every $h\in[H]$ and $(s,a)\in\mathcal S\times\mathcal A$, we have
\begin{align}
\bigl\|P_h^{t+1}(\cdot| s,a)-P_h(\cdot| s,a)\bigr\|_1
&\le
\sum_{s'\in\mathcal S}
\bigl|P_h^{t+1}(s'| s,a)-\widehat P_h^t(s'| s,a)\bigr|
+
\sum_{s'\in\mathcal S}
\bigl|\widehat P_h^t(s'| s,a)-P_h(s'| s,a)\bigr|
\nonumber\\
&\le
2|\mathcal S|\epsilon_h^t(s,a)
\le
2|\mathcal S|
\sqrt{
\frac{2\ln(2mTH|\mathcal S|^2|\mathcal A|/\delta)}
{\max\{1,N_h^t(s,a)\}}
}.
\label{eq:bridge-transition-bound}
\end{align}
Lemma~\ref{lem:simulation} with $P'=P^{t+1}$ and $\pi=\widetilde\pi^{t+1}$ gives
\begin{align}
\bigl\|
\theta^{P^{t+1},\widetilde\pi^{t+1}}
-
\theta^{P,\widetilde\pi^{t+1}}
\bigr\|_1
&\le
H\sum_{h=1}^H\sum_{s,a}
\theta_h^{P,\widetilde\pi^{t+1}}(s,a)
\bigl\|P_h^{t+1}(\cdot| s,a)-P_h(\cdot| s,a)\bigr\|_1 .
\label{eq:bridge-simulation}
\end{align}
Applying~\eqref{eq:bridge-simulation} with weight $(1-\psi^t)$ and using
$\theta^{P,\pi^{t+1}}
=
(1-\psi^t)\theta^{P,\widetilde\pi^{t+1}}
+
\psi^t\theta^{P,\pi^\dagger},$
we have
$(1-\psi^t)\theta_h^{P,\widetilde\pi^{t+1}}(s,a)
\le
\theta_h^{P,\pi^{t+1}}(s,a).$
Therefore, by~\eqref{eq:bridge-transition-bound} and~\eqref{eq:bridge-simulation}, we obtain
\begin{align}
&\sum_{t=1}^{T-1}
(1-\psi^t)
\bigl\|
\theta^{P^{t+1},\widetilde\pi^{t+1}}
-
\theta^{P,\widetilde\pi^{t+1}}
\bigr\|_1 \nonumber\\
& \quad \le
2H|\mathcal S|
\sqrt{2\ln(2mTH|\mathcal S|^2|\mathcal A|/\delta)}
\sum_{t=1}^{T-1}\sum_{h=1}^H\sum_{s,a}
\frac{\theta_h^{P,\pi^{t+1}}(s,a)}
{\sqrt{\max\{1,N_h^t(s,a)\}}}.
\label{eq:bridge-before-visitation}
\end{align}

Let $\mathcal F_t$ denote the history generated by all main and auxiliary trajectories up to the end of episode $t$.
Since $\pi^{t+1}$ and $N_h^t$ are $\mathcal F_t$-measurable, the main trajectory satisfies
$\mathbb E\![
\mathbf 1\{(s_h^{t+1},a_h^{t+1})=(s,a)\}
\,|\,\mathcal F_t]
=
\theta_h^{P,\pi^{t+1}}(s,a).$
Therefore,
$Y^t
:=
\sum_{h=1}^H\sum_{s,a}
\frac{\theta_h^{P,\pi^{t+1}}(s,a)}
{\sqrt{\max\{1,N_h^t(s,a)\}}}
-
\sum_{h=1}^H
\frac{1}
{\sqrt{\max\{1,N_h^t(s_h^{t+1},a_h^{t+1})\}}}$
is $\mathcal F_{t+1}$-measurable and satisfies $\mathbb E[Y^t|\mathcal F_t]=0$.
Moreover, we have $|Y^t|\le H$.
By the Azuma-Hoeffding inequality, with probability at least $1-\delta$, we have
\begin{align}
\sum_{t=1}^{T-1}\sum_{h=1}^H\sum_{s,a}
\frac{\theta_h^{P,\pi^{t+1}}(s,a)}
{\sqrt{\max\{1,N_h^t(s,a)\}}}
&\le
\sum_{t=1}^{T-1}\sum_{h=1}^H
\frac{1}
{\sqrt{\max\{1,N_h^t(s_h^{t+1},a_h^{t+1})\}}}
+
H\sqrt{2T\ln(1/\delta)}.
\label{eq:bridge-hoeffding}
\end{align}

Fix $h\in[H]$ and $(s,a)\in\mathcal S\times\mathcal A$.
Equation~\eqref{eq:rollout-counts} gives
$N_h^t(s,a)
=
\sum_{k=1}^{t}
\mathbf 1\{(s_h^k,a_h^k)=(s,a)\}.$
Therefore, we have
\[
\sum_{t=1}^{T-1}
\frac{
\mathbf 1\{(s_h^{t+1},a_h^{t+1})=(s,a)\}
}
{\sqrt{\max\{1,N_h^t(s,a)\}}}
\le
2\sqrt{N_h^T(s,a)}.
\]
Summing over $(s,a)$ and applying the Cauchy-Schwarz inequality gives
\begin{align}
\sum_{s,a}\sum_{t=1}^{T-1}
\frac{
\mathbf 1\{(s_h^{t+1},a_h^{t+1})=(s,a)\}
}
{\sqrt{\max\{1,N_h^t(s,a)\}}}
&\le
2\sqrt{
|\mathcal S||\mathcal A|
\sum_{s,a}N_h^T(s,a)
}
=
2\sqrt{|\mathcal S||\mathcal A|T}.
\label{eq:bridge-realized-count}
\end{align}
Applying~\eqref{eq:bridge-realized-count} to the main trajectories in~\eqref{eq:bridge-hoeffding} and summing over $h\in[H]$, we obtain
\begin{align}
\sum_{t=1}^{T-1}\sum_{h=1}^H\sum_{s,a}
\frac{\theta_h^{P,\pi^{t+1}}(s,a)}
{\sqrt{\max\{1,N_h^t(s,a)\}}}
&\le
\mathcal O\!\left(
H\sqrt{|\mathcal S||\mathcal A|T}
+
H\sqrt{T\ln(1/\delta)}
\right).
\label{eq:bridge-visitation-final}
\end{align}
Using~\eqref{eq:bridge-visitation-final} in~\eqref{eq:bridge-before-visitation}, we obtain
$\frac1T\sum_{t=1}^{T-1}(1-\psi^t)
\|
\theta^{P^{t+1},\widetilde\pi^{t+1}}
-\theta^{P,\widetilde\pi^{t+1}}
\|_1
\le
\widetilde{\mathcal O}\!(
H^2|\mathcal S|^{3/2}\sqrt{|\mathcal A|/T}).$

For the second claim, since both occupancy measures have total mass \(H\), we have 
$\|
\theta^{P^{t+1},\widetilde\pi^{t+1}}
-
\theta^{P,\widetilde\pi^{t+1}}\|_1\le 2H$.
Therefore, we have
\begin{align*}
\frac1T\sum_{t=1}^{T-1}
\bigl\|
\theta^{P^{t+1},\widetilde\pi^{t+1}}
-
\theta^{P,\widetilde\pi^{t+1}}
\bigr\|_1
&\le
\frac1T\sum_{t=1}^{T-1} (1-\psi^t)\bigl\|
\theta^{P^{t+1},\widetilde\pi^{t+1}}
-
\theta^{P,\widetilde\pi^{t+1}}
\bigr\|_1
+
\frac{2H}{T}\sum_{t=1}^{T-1}\psi^t \\
&\le
\widetilde{\mathcal O}\!\left(
H^2|\mathcal S|^{3/2}\sqrt{|\mathcal A|/T}
\right)
+
\frac{2H}{T}\sum_{t=1}^{T-1}\psi^t .
\end{align*}
This completes the proof.
\end{proof}
\begin{lemma}
\label{lem:main-occ-regret}
Let \(\theta\in\Theta_{\mathrm{safe}}\) be any safe occupancy measure, and let \(\widetilde\theta^{t+1}\) be the occupancy measure generated by \texttt{SafeProj} in Algorithm~\ref{alg:CSTRL}.
On the event \(\mathcal E_{R,C,T}(\delta)\), with \(\eta=T^{-1/2}\), we have
\begin{align*}
\frac1T\sum_{t=1}^{T-1}
\left\langle
\sum_{i=1}^m w_i^t\lambda_i r_i,\,
\theta-\widetilde\theta^{t+1}
\right\rangle
\le
\widetilde{\mathcal O}\!\left(
H^2|\mathcal S|^{3/2}\sqrt{|\mathcal A|/T}\right)
+ \frac{2H}{T} \sum_{t=1}^{T-1} \psi^t.
\end{align*}
\end{lemma}

\begin{proof}
On the event $\mathcal E_{R,C,T}(\delta)$, we have
$P\in\mathcal P^t$ and
$\widehat C^t-\Xi^t\le C$.
Since $\theta\in\Theta_{\mathrm{safe}}$, it follows that
$\theta\in\Theta(\mathcal P^t)$ and
$(\widehat C^t-\Xi^t)^\top\theta
\le
C^\top\theta
\le
\boldsymbol\tau.$
Thus, $\theta$ is feasible for the projection at every episode $t$.
By Lemma~\ref{lem:occ-polytope}, the feasible occupancy measure set is convex.
Therefore, the exponential update
$\widehat\theta_h^{t+1}(s,a,s')
=
\widetilde\theta_h^t(s,a,s')\exp\{\eta\bar R_h^t(s,a)\}$
followed by the KL projection gives, for every feasible $\theta$,
\begin{align}
\langle \bar R^t,\theta-\widetilde\theta^{t+1}\rangle
&\le
\frac{
D_{\mathrm{KL}}(\theta\Vert\widetilde\theta^t)
-
D_{\mathrm{KL}}(\theta\Vert\widetilde\theta^{t+1})
}{\eta}.
\label{eq:main-omd-one-step}
\end{align}
Summing~\eqref{eq:main-omd-one-step} over \(t=1,\ldots,T-1\) cancels the intermediate KL terms and gives
\begin{align}
\frac1T\sum_{t=1}^{T-1}
\langle \bar R^t,\theta-\widetilde\theta^{t+1}\rangle
&\le
\frac{D_{\mathrm{KL}}(\theta\Vert\widetilde\theta^1)}{\eta T}.
\label{eq:main-optimistic-regret}
\end{align}
Since the initial occupancy measure is set to $\widetilde\theta_h^1(s,a,s')=\frac{1}{|\mathcal S|^2|\mathcal A|},$ and every feasible occupancy measure satisfies $\sum_{h=1}^H\sum_{s,a,s'}\theta_h(s,a,s')=H,$  we have $D_{\mathrm{KL}}(\theta\Vert\widetilde\theta^1) \le H\log(|\mathcal S|^2|\mathcal A|).$

We next express the preceding inequality using the true scalarized reward $R_h^t(s,a):=\sum_{i=1}^m w_i^t\lambda_i r_{i,h}(s,a)$.
By Lemma~\ref{lem:reward_concentration} and~\eqref{eq:optimistic_reward}, on $\mathcal E_{R,C,T}(\delta)$, for all $i,h,s,a,t$, we have
$0\le
\bar r_{i,h}^t(s,a)-r_{i,h}(s,a)
\le 2\sigma_h^t(s,a)$.
Consequently, for all $h,s,a,t$, we have
$0\le
\bar R_h^t(s,a)-R_h^t(s,a)
\le
2\sum_{i=1}^m w_i^t\lambda_i\sigma_h^t(s,a)$.
Therefore, for every $t$, we have
\begin{align}
\langle R^t,\theta-\widetilde\theta^{t+1}\rangle
&=
\langle \bar R^t,\theta-\widetilde\theta^{t+1}\rangle
-
\langle \bar R^t-R^t,\theta-\widetilde\theta^{t+1}\rangle
\nonumber\\
&\le
\langle \bar R^t,\theta-\widetilde\theta^{t+1}\rangle
+
2\left\langle
\sum_{i=1}^m w_i^t\lambda_i\sigma^t,\,
\widetilde\theta^{t+1}
\right\rangle .
\label{eq:true-to-opt-main}
\end{align}
The last inequality uses $\bar R^t-R^t\ge 0$ and $\theta\ge 0$.

We now bound the second term in~\eqref{eq:true-to-opt-main}.
By construction, we have $\widetilde\pi^{t+1}=\pi^{\widetilde\theta^{t+1}}$ and $\widetilde\theta^{t+1}=\theta^{P^{t+1},\widetilde\pi^{t+1}}$ for some $P^{t+1}\in\mathcal P^t$.
Moreover, at the beginning of episode $t+1$,
\texttt{SafeProj} selects $\pi^\dagger$ with probability $\psi^t$ and $\widetilde\pi^{t+1}$ with probability $1-\psi^t$.
Therefore, we have
$\theta^{P,\pi^{t+1}}
=
(1-\psi^t)\theta^{P,\widetilde\pi^{t+1}}
+
\psi^t\theta^{P,\pi^\dagger}$.
Since
$0\le \sum_{i=1}^m w_i^t\lambda_i\sigma_h^t(s,a)\le 1$
and $\sum_{h=1}^H\sum_{s,a,s'}\widetilde\theta_h^{t+1}(s,a,s')=H$, 
we have
\begin{align}
&\left\langle
\sum_{i=1}^m w_i^t\lambda_i\sigma^t,\,
\widetilde\theta^{t+1}
\right\rangle\nonumber\\
&\quad=
(1-\psi^t)
\left\langle
\sum_{i=1}^m w_i^t\lambda_i\sigma^t,\,
\widetilde\theta^{t+1}
\right\rangle
+
\psi^t
\left\langle
\sum_{i=1}^m w_i^t\lambda_i\sigma^t,\,
\widetilde\theta^{t+1}
\right\rangle
\nonumber\\
&\quad\le
(1-\psi^t)
\left\langle
\sum_{i=1}^m w_i^t\lambda_i\sigma^t,\,
\theta^{P,\widetilde\pi^{t+1}}
\right\rangle
+
(1-\psi^t)
\left\|
\theta^{P^{t+1},\widetilde\pi^{t+1}}
-
\theta^{P,\widetilde\pi^{t+1}}
\right\|_1
+
\psi^t H
\nonumber\\
&\quad\le
\left\langle
\sum_{i=1}^m w_i^t\lambda_i\sigma^t,\,
\theta^{P,\pi^{t+1}}
\right\rangle
+
(1-\psi^t)
\left\|
\theta^{P^{t+1},\widetilde\pi^{t+1}}
-
\theta^{P,\widetilde\pi^{t+1}}
\right\|_1
+
\psi^t H .
\label{eq:main-bonus-weighted-bridge}
\end{align}
The last inequality follows from
$(1-\psi^t)\theta_h^{P,\widetilde\pi^{t+1}}(s,a)
\le
\theta_h^{P,\pi^{t+1}}(s,a),
\, \forall h,s,a.$
The first term in~\eqref{eq:main-bonus-weighted-bridge} is controlled by the visitation bound for the executed main trajectories.
Since $\pi^{t+1}$ and $\{\sigma_h^t\}_{h}$ are $\mathcal F_t$-measurable, we have
$\mathbb E\![
\mathbf 1\{(s_h^{t+1},a_h^{t+1})=(s,a)\}
\,|\,\mathcal F_t]
=
\theta_h^{P,\pi^{t+1}}(s,a).$
We define
$Y^t
:=
\sum_{h=1}^H\sum_{s,a}
(\sum_{i=1}^m w_i^t\lambda_i\sigma_h^t(s,a))
\theta_h^{P,\pi^{t+1}}(s,a)  -
\sum_{h=1}^H
\sum_{i=1}^m w_i^t\lambda_i
\sigma_h^t(s_h^{t+1},a_h^{t+1}).$
Then we have $\mathbb E[Y^t|\mathcal F_t]=0$. 
Moreover, since $\boldsymbol w^t,\boldsymbol\lambda\in\Delta_m$ and $0\le\sigma_h^t(s,a)\le1$,
we obtain
$0\le
\sum_{i=1}^m w_i^t\lambda_i\sigma_h^t(s,a)
\le
\sum_{i=1}^m w_i^t\lambda_i
\le 1.$
Thus, we have $|Y^t|\le H$. 
By Hoeffding's inequality for bounded martingale differences, with probability at least $1-\delta$,
we obtain
\begin{align}
\sum_{t=1}^{T-1}
\left\langle
\sum_{i=1}^m w_i^t\lambda_i\sigma^t,\,
\theta^{P,\pi^{t+1}}
\right\rangle
&\le
\sum_{t=1}^{T-1}\sum_{h=1}^H
\sum_{i=1}^m w_i^t\lambda_i
\sigma_h^t(s_h^{t+1},a_h^{t+1})
+
H\sqrt{2T\ln(1/\delta)}.
\label{eq:main-bonus-hoeffding-step}
\end{align}
It remains to bound the realized sum of reward bonuses.
By the definition of
$\sigma_h^t$ in~\eqref{eq:reward_confidence},
we have
\begin{align}
&\sum_{t=1}^{T-1}\sum_{h=1}^H
\sum_{i=1}^m w_i^t\lambda_i
\sigma_h^t(s_h^{t+1},a_h^{t+1})
\le
\sqrt{2\ln(8m|\mathcal S||\mathcal A|HT/\delta)}
\sum_{t=1}^{T-1}\sum_{h=1}^H
\frac{1}
{\sqrt{\max\{1,N_h^t(s_h^{t+1},a_h^{t+1})\}}}.
\label{eq:main-realized-bonus-before-count}
\end{align}
Summing~\eqref{eq:bridge-realized-count} over $h\in[H]$ for the main trajectories and using~\eqref{eq:main-realized-bonus-before-count},
we obtain
\begin{align}
&\sum_{t=1}^{T-1}\sum_{h=1}^H
\sum_{i=1}^m w_i^t\lambda_i
\sigma_h^t(s_h^{t+1},a_h^{t+1})
\le
2H\sqrt{2|\mathcal S||\mathcal A|T\ln(8m|\mathcal S||\mathcal A|HT/\delta)}.
\label{eq:main-realized-bonus-final}
\end{align}
Using~\eqref{eq:main-bonus-hoeffding-step} together with~\eqref{eq:main-realized-bonus-final},
we obtain
$\frac1T\sum_{t=1}^{T-1}
\langle
\sum_{i=1}^m w_i^t\lambda_i\sigma^t,\,
\theta^{P,\pi^{t+1}}
\rangle
\le
\widetilde{\mathcal O}\!(
H\sqrt{\frac{|\mathcal S||\mathcal A|}{T}}
+
H\sqrt{\frac{1}{T}}
).$

By Lemma~\ref{lem:weighted-transition-bridge}, the second term is controlled as
$\frac1T\sum_{t=1}^{T-1}
(1-\psi^t)
\|
\theta^{P^{t+1},\widetilde\pi^{t+1}}
- 
\theta^{P,\widetilde\pi^{t+1}}
\|_1
\le \allowbreak
\widetilde{\mathcal O}\!(
H^2|\mathcal S|^{3/2}\sqrt{|\mathcal A|/T}).$
Therefore, we obtain
\[
\frac1T\sum_{t=1}^{T-1}
\left\langle
\sum_{i=1}^m w_i^t\lambda_i\sigma^t,\,
\widetilde\theta^{t+1}
\right\rangle
\le
\widetilde{\mathcal O}\!\left(
H^2|\mathcal S|^{3/2}\sqrt{|\mathcal A|/T}
\right)
+
\frac{H}{T}\sum_{t=1}^{T-1}\psi^t .
\]

Combining this bound with~\eqref{eq:true-to-opt-main} and~\eqref{eq:main-optimistic-regret}, and taking $\eta=T^{-1/2}$, gives
\[
\frac1T\sum_{t=1}^{T-1}
\left\langle
\sum_{i=1}^m w_i^t\lambda_i r_i,\,
\theta-\widetilde\theta^{t+1}
\right\rangle
\le
\widetilde{\mathcal O}\!\left(
H^2|\mathcal S|^{3/2}\sqrt{|\mathcal A|/T}
\right)
+
\frac{2H}{T}\sum_{t=1}^{T-1}\psi^t.\]
This completes the proof.
\end{proof}

\begin{lemma}
\label{lem:aux-occ-regret}
Let $\widetilde\theta_i^{t+1}$ be the projected occupancy measure generated by Algorithm~\ref{alg:CSTRL}, and let $\theta_i^*\in\Theta_{\mathrm{safe}}$ be a single-objective optimal occupancy measure for objective $i$. 
On the event $\mathcal E_{R,C,T}(\delta)$, with $\eta=T^{-1/2}$, for every $i\in[m]$, we have
\begin{align*}
\frac1T\sum_{t=1}^{T-1}
\left\langle r_i,\theta_i^*-\widetilde\theta_i^{t+1}\right\rangle
\le
\widetilde{\mathcal O}\!\left(
H^2|\mathcal S|^{3/2}\sqrt{|\mathcal A|/T}
\right)
+\frac{2H}{T}\sum_{t=1}^{T-1}\psi_{i}^t.
\end{align*}
Moreover, we have
\begin{align*}
\frac1T\sum_{t=1}^{T-1}
\left\langle r_i,\theta_i^*-\theta^{P,\widetilde\nu_i^{t+1}}\right\rangle
\le
\widetilde{\mathcal O}\!\left(
H^2|\mathcal S|^{3/2}\sqrt{|\mathcal A|/T}
\right)
+\frac{4H}{T}\sum_{t=1}^{T-1}\psi_{i}^t.
\end{align*}
\end{lemma}

\begin{proof}
Fix $i\in[m]$.
For every $t\in[T-1]$, $\theta_i^*$ is feasible for the auxiliary projection on $\mathcal E_{R,C,T}(\delta)$.
Since $\theta_i^*\in\Theta_{\mathrm{safe}}$,
$P\in\mathcal P^t$, and
$\widehat C^t-\Xi^t\le C$, we have
$(\widehat C^t-\Xi^t)^\top\theta_i^*
\le
C^\top\theta_i^*
\le
\boldsymbol\tau.$
Applying the KL projection argument in the proof of
Lemma~\ref{lem:main-occ-regret} with $\bar R^t$ replaced by
$\bar r_i^t$, and using the uniform initialization, gives
$\frac1T\sum_{t=1}^{T-1}
\langle\bar r_i^t,\theta_i^*-\widetilde\theta_i^{t+1}\rangle
\le
\frac{D_{\mathrm{KL}}(\theta_i^*\Vert\widetilde\theta_i^1)}
{\eta T}
\le
\frac{H\log(|\mathcal S|^2|\mathcal A|)}{\eta T}$.
With $\eta=T^{-1/2}$, the right-hand side is $\widetilde{\mathcal O}(H/\sqrt T)$.

We next express the preceding inequality using the true reward \(r_i\).
By Lemma~\ref{lem:reward_concentration} and~\eqref{eq:optimistic_reward}, on $\mathcal E_{R,C,T}(\delta)$, we have
$0
\le
\bar r_{i,h}^t(s,a)-r_{i,h}(s,a)
\le
2\sigma_h^t(s,a),
\, \forall h\in[H],\,s\in\mathcal S,\,a\in\mathcal A,\,t\in[T-1].$
Therefore, we have
\begin{align}
\left\langle
r_i,\theta_i^*-\widetilde\theta_i^{t+1}
\right\rangle
&\le
\left\langle
\bar r_i^t,\theta_i^*-\widetilde\theta_i^{t+1}
\right\rangle
+
2\left\langle
\sigma^t,\widetilde\theta_i^{t+1}
\right\rangle .
\label{eq:aux-true-to-opt}
\end{align}
It remains to control the reward bonus term. 
By construction, we have $\widetilde\nu_i^{t+1}=\pi^{\widetilde\theta_i^{t+1}}$ and $\widetilde\theta_i^{t+1}=\theta^{P_i^{t+1},\widetilde\nu_i^{t+1}}$ for some $P_i^{t+1}\in\mathcal P^t$. 
Moreover, in the auxiliary update, \texttt{SafeProj} selects $\pi^\dagger$ with probability $\psi_i^t$ and $\widetilde\nu_i^{t+1}$ with probability $1-\psi_i^t$.
Therefore, we have
$\theta^{P,\nu_i^{t+1}}
=
(1-\psi_i^t)\theta^{P,\widetilde\nu_i^{t+1}}
+
\psi_i^t\theta^{P,\pi^\dagger}$.
Since 
$0\le \sigma_h^t(s,a)\le 1$ and
$\sum_{h=1}^H\sum_{s,a,s'}\widetilde\theta_{i,h}^{t+1}(s,a,s')=H$, the same decomposition as in~\eqref{eq:main-bonus-weighted-bridge} gives
\begin{align}
\left\langle
\sigma^t,\widetilde\theta_i^{t+1}
\right\rangle
&\le
\left\langle
\sigma^t,\theta^{P,\nu_i^{t+1}}
\right\rangle
+
(1-\psi_{i}^t)
\left\|
\theta^{P_i^{t+1},\widetilde\nu_i^{t+1}}
-
\theta^{P,\widetilde\nu_i^{t+1}}
\right\|_1
+
H\psi_{i}^t.
\label{eq:aux-bonus-weighted-bridge}
\end{align}
The first term in~\eqref{eq:aux-bonus-weighted-bridge} is bounded by the same visitation argument as in Lemma~\ref{lem:main-occ-regret}, applied to the trajectories generated by $\{\nu_i^t\}_{t=1}^T$.
Thus, we have
$\frac1T\sum_{t=1}^{T-1}
\langle
\sigma^t,\theta^{P,\nu_i^{t+1}}
\rangle
\le
\widetilde{\mathcal O}\!(
H\sqrt{\frac{|\mathcal S||\mathcal A|}{T}}
+
H\sqrt{\frac{1}{T}}).$
The second term in~\eqref{eq:aux-bonus-weighted-bridge} is bounded by applying the argument of Lemma~\ref{lem:weighted-transition-bridge} to the same auxiliary trajectories.
Specifically, we replace $(\widetilde\pi^{t+1},\pi^{t+1},\psi^t)$ by $(\widetilde\nu_i^{t+1},\nu_i^{t+1},\psi_{i}^t)$ and obtain
$\frac1T\sum_{t=1}^{T-1}
(1-\psi_{i}^t)
\|
\theta^{P_i^{t+1},\widetilde\nu_i^{t+1}}
-
\theta^{P,\widetilde\nu_i^{t+1}}\|_1
\le
\widetilde{\mathcal O}\!(
H^2|\mathcal S|^{3/2}\sqrt{|\mathcal A|/T}).$
Combining these bounds, 
we have
$\frac1T\sum_{t=1}^{T-1}
\langle
\sigma^t,\widetilde\theta_i^{t+1}
\rangle
\le
\widetilde{\mathcal O}\!(
H^2|\mathcal S|^{3/2}\sqrt{|\mathcal A|/T})
+
\frac{H}{T}\sum_{t=1}^{T-1}\psi_{i}^t.$
Taking the average of~\eqref{eq:aux-true-to-opt} over \(t=1,\ldots,T-1\) and using the preceding two bounds, we obtain
\[
\frac1T\sum_{t=1}^{T-1}
\left\langle
r_i,\theta_i^*-\widetilde\theta_i^{t+1}
\right\rangle
\le
\widetilde{\mathcal O}\!\left(
H^2|\mathcal S|^{3/2}\sqrt{|\mathcal A|/T}
\right)
+\frac{2H}{T}\sum_{t=1}^{T-1}\psi_{i}^t.
\]
 
By construction, we have $\widetilde\theta_i^{t+1} = \theta^{P_i^{t+1},\widetilde\nu_i^{t+1}}$ for some $P_i^{t+1}\in\mathcal P^t$.
Therefore, since $0\le r_{i,h}(s,a)\le 1$, we have
$\langle
r_i,\widetilde\theta_i^{t+1}-\theta^{P,\widetilde\nu_i^{t+1}}
\rangle
\le
\|
\theta^{P_i^{t+1},\widetilde\nu_i^{t+1}}
-
\theta^{P,\widetilde\nu_i^{t+1}}
\|_1.$
Applying the unweighted form of Lemma~\ref{lem:weighted-transition-bridge} to the auxiliary policy sequence yields
$\frac1T\sum_{t=1}^{T-1}
\|
\theta^{P_i^{t+1},\widetilde\nu_i^{t+1}}
-
\theta^{P,\widetilde\nu_i^{t+1}}
\|_1
\le
\widetilde{\mathcal O}\!(
H^2|\mathcal S|^{3/2}\sqrt{|\mathcal A|/T}
)
+
\frac{2H}{T}\sum_{t=1}^{T-1}\psi_{i}^t.$
Therefore, we obtain
\begin{align*}
\frac1T\sum_{t=1}^{T-1}
\left\langle r_i,\theta_i^*-\theta^{P,\widetilde\nu_i^{t+1}}\right\rangle
&\le
\widetilde{\mathcal O}\!\left(
H^2|\mathcal S|^{3/2}\sqrt{|\mathcal A|/T}
\right)
+
\frac{4H}{T}\sum_{t=1}^{T-1}\psi_{i}^t.
\end{align*}
This completes the proof.
\end{proof}

\begin{lemma}
\label{lem:aux-weighted-transition-bridge}
With probability at least $1-\delta$, for every objective $i\in[m]$ and any $P_i^{t+1}\in\mathcal P^t$ associated with the auxiliary policies generated by Algorithm~\ref{alg:CSTRL}, we have
\begin{align*}
\frac1T\sum_{t=1}^{T-1}(1-\psi_i^t)
\bigl\|
\theta^{P_i^{t+1},\widetilde\nu_i^{t+1}}
-
\theta^{P,\widetilde\nu_i^{t+1}}
\bigr\|_1
\le
\widetilde{\mathcal O}\!\left(
H^2|\mathcal S|^{3/2}\sqrt{|\mathcal A|/T}
\right).
\end{align*}
Consequently, we have
\begin{align*}
\frac1T\sum_{t=1}^{T-1}
\bigl\|
\theta^{P_i^{t+1},\widetilde\nu_i^{t+1}}
-
\theta^{P,\widetilde\nu_i^{t+1}}
\bigr\|_1
\le
\widetilde{\mathcal O}\!\left(
H^2|\mathcal S|^{3/2}\sqrt{|\mathcal A|/T}
\right)
+
\frac{2H}{T}\sum_{t=1}^{T-1}\psi_i^t .
\end{align*}
\end{lemma}
\begin{proof}
The proof repeats the argument of Lemma~\ref{lem:weighted-transition-bridge} using the auxiliary trajectory sequence for objective $i$.
We replace $(\widetilde\pi^{t+1},\pi^{t+1},\psi^t)$ by $(\widetilde\nu_i^{t+1},\nu_i^{t+1},\psi_i^t)$.
\end{proof}

\begin{lemma}
\label{lem:mixing-bound}
Let \((\mathcal P,\widehat C,\Xi)\) be the confidence set, cost estimates, and cost bonuses used in a call of \texttt{SafeProj}.
Let \(\widetilde\theta\in\Theta(\mathcal P)\) satisfy \((\widehat C-\Xi)^\top\widetilde\theta\le\boldsymbol\tau\), and let \(\widetilde\pi:=\pi^{\widetilde\theta}\).
Let \(\psi\) be given by~\eqref{eq:eta_t} with \((\mathcal P^t,\widehat C^t,\Xi^t,\widetilde\pi^{t+1})\) replaced by \((\mathcal P,\widehat C,\Xi,\widetilde\pi)\).
Then we have
\[
\psi
\le
(1-\psi)
\max_{j\in[n]}
\frac{1}{\tau_j-\gamma_j}
\biggl[
3\max_{P'\in\mathcal P}
\bigl\|
\theta^{P',\widetilde\pi}
-
\theta^{P,\widetilde\pi}
\bigr\|_1
+
2\max_{j'\in[n]}
\sum_{h=1}^H
\xi_{j',h}^\top
\theta_h^{P,\widetilde\pi}
\biggr].
\]
\end{lemma}

\begin{proof}
If \(\psi=0\), the claim holds trivially.
Otherwise, let \(j\in[n]\) attain the maximum in~\eqref{eq:eta_t}.
Since \(\psi>0\), the minimum in the numerator of the corresponding ratio exceeds \(\tau_j\), and~\eqref{eq:eta_t} gives
\[
\psi
=
\frac{
\min\left\{
\max_{P'\in\mathcal P}
\sum_{h=1}^H
(\widehat c_{j,h}+\xi_{j,h})^\top
\theta_h^{P',\widetilde\pi},
d
\right\}
-\tau_j
}{
\min\left\{
\max_{P'\in\mathcal P}
\sum_{h=1}^H
(\widehat c_{j,h}+\xi_{j,h})^\top
\theta_h^{P',\widetilde\pi},
d
\right\}
-\gamma_j
}.
\]
Rearranging this identity and bounding the minimum by its first argument, we obtain
\begin{equation}
\psi(\tau_j-\gamma_j)
\le
(1-\psi)
\biggl(
\max_{P'\in\mathcal P}
\sum_{h=1}^H
(\widehat c_{j,h}+\xi_{j,h})^\top
\theta_h^{P',\widetilde\pi}
-\tau_j
\biggr).
\label{eq:mixing-rearranged}
\end{equation}
We next bound the right-hand side of~\eqref{eq:mixing-rearranged}.
Since \(\widehat c_{j,h}+\xi_{j,h}\in[0,2]^{|\mathcal S||\mathcal A|}\) for every \(h\in[H]\), we have
\[
\max_{P'\in\mathcal P}
\sum_{h=1}^H
(\widehat c_{j,h}+\xi_{j,h})^\top
\theta_h^{P',\widetilde\pi}
\le
\sum_{h=1}^H
(\widehat c_{j,h}+\xi_{j,h})^\top
\theta_h^{P,\widetilde\pi}
+
2\max_{P'\in\mathcal P}
\bigl\|
\theta^{P',\widetilde\pi}
-
\theta^{P,\widetilde\pi}
\bigr\|_1 .
\]
As in the proof of Lemma~\ref{lem:occ-polytope}, \(\widetilde\theta\) is the occupancy measure of \(\widetilde\pi\) under some transition kernel in \(\mathcal P\).
Since \(\widehat c_{j,h}-\xi_{j,h}\in[-1,1]^{|\mathcal S||\mathcal A|}\) for every \(h\in[H]\), the constraint on \(\widetilde\theta\) gives
\[
\tau_j
\ge
\sum_{h=1}^H
(\widehat c_{j,h}-\xi_{j,h})^\top
\widetilde\theta_h
\ge
\sum_{h=1}^H
(\widehat c_{j,h}-\xi_{j,h})^\top
\theta_h^{P,\widetilde\pi}
-
\max_{P'\in\mathcal P}
\bigl\|
\theta^{P',\widetilde\pi}
-
\theta^{P,\widetilde\pi}
\bigr\|_1 .
\]
Subtracting the last inequality from the preceding one, we obtain
\[
\max_{P'\in\mathcal P}
\sum_{h=1}^H
(\widehat c_{j,h}+\xi_{j,h})^\top
\theta_h^{P',\widetilde\pi}
-\tau_j
\le
3\max_{P'\in\mathcal P}
\bigl\|
\theta^{P',\widetilde\pi}
-
\theta^{P,\widetilde\pi}
\bigr\|_1
+
2\sum_{h=1}^H
\xi_{j,h}^\top
\theta_h^{P,\widetilde\pi}.
\]
Substituting this bound into~\eqref{eq:mixing-rearranged} and dividing both sides by \(\tau_j-\gamma_j>0\) proves the lemma.
\end{proof}

\begin{lemma}
\label{lem:oco-stch} 
Set \(\alpha_t=1/t^2\) and \(\beta_t=1/(\mu t)\) for \(t\in[T-1]\).
Then the updating rule of \(\boldsymbol w\) in Algorithm~\ref{alg:CSTRL} guarantees that
\begin{align*}
&\max_{\boldsymbol{w} \in\Delta_m} \frac{1}{T}\sum_{t=1}^{T-1}  \left[\langle \boldsymbol{w}-\boldsymbol{w}^t, \boldsymbol{\lambda}\odot (\widetilde{\boldsymbol V}^{t+1}
+\boldsymbol{\iota}
-{\boldsymbol V}^{t+1})\rangle -\mu  \langle \boldsymbol{w},\log \boldsymbol{w}\rangle + \mu  \langle \boldsymbol{w}^t,\log \boldsymbol{w}^t\rangle  \right] \\
&\quad \le
\widetilde{\mathcal O}\!\left(
\frac{H^2}{\mu T}
+
\frac{\mu}{T}
\right).
\end{align*}
\end{lemma}

\begin{proof} 
The mirror-ascent update that produces $\boldsymbol w^{t+1}$ in Line~\ref{line:stch-w} of Algorithm~\ref{alg:CSTRL} is equivalent to solving the following minimization problem
\begin{align*}
\minimize_{\boldsymbol{w}\in\Delta_m} \quad
&-\beta_t \Big\langle \boldsymbol{w},
\boldsymbol{\lambda} \odot
\bigl(\widetilde{\boldsymbol V}^{t+1}
+\boldsymbol{\iota}
-{\boldsymbol V}^{t+1}\bigr)
-\mu \log \widetilde{\boldsymbol{w}}^t
\Big\rangle
+ D_{\mathrm{KL}}\big( \boldsymbol{w}\Vert \widetilde{\boldsymbol{w}}^t \big).
\end{align*}
Let $\bw^{t+1}$ be the solution to the above optimization problem.
Since $w_i^1=1/m$, the iterate $\bw^{t+1}$ remains in the relative interior of $\Delta_m$.
Thus, applying Lemma \ref{lem:kl-three-point} gives
\[\beta_t \langle \boldsymbol{w}-\boldsymbol{w}^{t+1}\!,\!
\boldsymbol{\lambda} \odot
(\widetilde{\boldsymbol V}^{t+1}
+\boldsymbol{\iota}
-{\boldsymbol V}^{t+1})
-\mu \log \widetilde{\boldsymbol{w}}^t
\rangle \! \leq \!
D_{\mathrm{KL}}( \boldsymbol{w}\Vert \widetilde{\boldsymbol{w}}^t )
\!-\!
D_{\mathrm{KL}}( \boldsymbol{w}\Vert \boldsymbol{w}^{t+1})
\!-\!
D_{\mathrm{KL}}( \boldsymbol{w}^{t+1}\Vert \widetilde{\boldsymbol{w}}^t),\]
where $\boldsymbol{w}$ is arbitrary. 
Rearranging the terms leads to
\begin{align}
\begin{aligned}
\label{eq:oco-11}
&\beta_t  \big\langle \boldsymbol{w}-\boldsymbol{w}^t,\,
\boldsymbol{\lambda} \odot
\bigl(\widetilde{\boldsymbol V}^{t+1}
+\boldsymbol{\iota}
-{\boldsymbol V}^{t+1}\bigr)
\big\rangle
-\beta_t \mu\big\langle \boldsymbol{w}, \log \boldsymbol{w}\big\rangle
+\beta_t \mu \big\langle \boldsymbol{w}^t, \log \boldsymbol{w}^t \big\rangle \\
&\qquad \leq
D_{\mathrm{KL}}\big( \boldsymbol{w}\Vert \widetilde{\boldsymbol{w}}^t\big)
-
D_{\mathrm{KL}}\big( \boldsymbol{w}\Vert \boldsymbol{w}^{t+1} \big)
-
D_{\mathrm{KL}}\big( \boldsymbol{w}^{t+1}\Vert \widetilde{\boldsymbol{w}}^t \big) \\
&\qquad \quad
+ \beta_t  \big\langle \boldsymbol{w}^{t+1} - \boldsymbol{w}^t,\,
\boldsymbol{\lambda} \odot
\bigl(\widetilde{\boldsymbol V}^{t+1}
+\boldsymbol{\iota}
-{\boldsymbol V}^{t+1}\bigr)
-\mu \log \widetilde{\boldsymbol{w}}^t
\big\rangle \\
&\qquad \quad
- \beta_t \mu\big\langle \boldsymbol{w}, \log \boldsymbol{w}\big\rangle
+\beta_t \mu \big\langle \boldsymbol{w}^t, \log \boldsymbol{w}^t \big\rangle
+\beta_t \mu \big\langle \boldsymbol{w} - \boldsymbol{w}^t,\,
\log \widetilde{\boldsymbol{w}}^t \big\rangle.
\end{aligned}
\end{align}

For the terms on the right-hand side of~\eqref{eq:oco-11}, by Lemma \ref{lem:mix} and the definition of KL divergence, we have
\begin{align}
\begin{aligned} \label{eq:oco-22}
&\beta_t  \big\langle \boldsymbol{w}-\boldsymbol{w}^t,\,
\boldsymbol{\lambda} \odot \bigl(\widetilde{\boldsymbol V}^{t+1}
+\boldsymbol{\iota}
-{\boldsymbol V}^{t+1}\bigr) \big\rangle
-\beta_t \mu\big\langle \boldsymbol{w}, \log \boldsymbol{w}\big\rangle
+\beta_t  \mu \big\langle \boldsymbol{w}^t, \log \boldsymbol{w}^t \big\rangle \\
&\qquad \leq (1-\beta_t \mu )D_{\mathrm{KL}}\big( \boldsymbol{w}\Vert \widetilde{\boldsymbol{w}}^t\big)
- D_{\mathrm{KL}}\big( \boldsymbol{w}\Vert \widetilde{\boldsymbol{w}}^{t+1} \big)
- D_{\mathrm{KL}}\big( \boldsymbol{w}^{t+1}\Vert \widetilde{\boldsymbol{w}}^t \big)
+ 2 \alpha_t \log m\\
&\qquad \quad + \beta_t  \big\langle \boldsymbol{w}^{t+1} - \boldsymbol{w}^t,\,
\boldsymbol{\lambda} \odot \bigl(\widetilde{\boldsymbol V}^{t+1}
+\boldsymbol{\iota}
-{\boldsymbol V}^{t+1}\bigr)
-\mu \log \widetilde{\boldsymbol{w}}^t \big\rangle,
\end{aligned}
\end{align}
where \(\mu\beta_t=1/t\le 1\). 
In addition, by Pinsker's inequality, we have
$-D_{\mathrm{KL}}\big( \bw^{t+1}\Vert \widetilde \bw^t \big) \leq -\frac{1}{2} \big\|\bw^{t+1} - \widetilde \bw^t\big\|^2_1.$
By H\"older's inequality, the last term in \eqref{eq:oco-22} is bounded as
\begin{align*}
&\beta_t  \big\langle \boldsymbol{w}^{t+1} - \boldsymbol{w}^t,\,
\boldsymbol{\lambda} \odot \bigl(\widetilde{\boldsymbol V}^{t+1}
+\boldsymbol{\iota}
-{\boldsymbol V}^{t+1}\bigr)
-\mu \log \widetilde{\boldsymbol{w}}^t \big\rangle\\
&\qquad \leq
\beta_t  \left(
\big\|\boldsymbol{w}^{t+1} - \widetilde{\boldsymbol{w}}^t\big\|_1
+
\big\|\widetilde{\boldsymbol{w}}^t - \boldsymbol{w}^t\big\|_1
\right)
\big\|\boldsymbol{\lambda} \odot \bigl(\widetilde{\boldsymbol V}^{t+1}
+\boldsymbol{\iota}
-{\boldsymbol V}^{t+1}\bigr)
-\mu \log \widetilde{\boldsymbol{w}}^t \big\|_\infty\\
&\qquad \leq
\frac{1}{2}\big\|\boldsymbol{w}^{t+1} - \widetilde{\boldsymbol w}^t \big\|_1^2
+
\left(3H + \mu\log (m/\alpha_t)\right)^2\beta_t ^2
+
2\alpha_t^2,
\end{align*}
where the last inequality uses Lemma \ref{lem:mix} and  
$\|
\boldsymbol{\lambda} \odot
(\widetilde{\boldsymbol V}^{t+1}
+\boldsymbol{\iota}
-{\boldsymbol V}^{t+1})
-\mu \log \widetilde{\boldsymbol{w}}^t
\|_\infty
\leq 3H + \mu\log (m/\alpha_t)$
with $\iota,\lambda_i\leq 1$.
Substituting the above bounds into \eqref{eq:oco-22}, 
we have
\begin{align*}
&\beta_t  \big\langle \boldsymbol{w}-\boldsymbol{w}^t,\,
\boldsymbol{\lambda} \odot \bigl(\widetilde{\boldsymbol V}^{t+1}
+\boldsymbol{\iota}
-{\boldsymbol V}^{t+1}\bigr) \big\rangle
-\beta_t \mu\big\langle \boldsymbol{w}, \log \boldsymbol{w}\big\rangle
+\beta_t  \mu \big\langle \boldsymbol{w}^t, \log \boldsymbol{w}^t \big\rangle \\
&\quad \leq
(1-\beta_t \mu )D_{\mathrm{KL}}\big( \boldsymbol{w}\Vert \widetilde{\boldsymbol{w}}^t\big)
- D_{\mathrm{KL}}\big( \boldsymbol{w}\Vert \widetilde{\boldsymbol{w}}^{t+1} \big)
+ 2 \alpha_t \log m
+ \left(3H + \mu\log (m/\alpha_t)\right)^2\beta_t ^2
+ 2\alpha_t^2.
\end{align*}
Now we set \(\beta_t=1/(\mu t)\) and \(\alpha_t=1/t^2\), and we divide both sides by \(T\beta_t\).
At \(t=T-1\) we retain \(-D_{\mathrm{KL}}(\boldsymbol w\Vert\boldsymbol w^{T})\) from~\eqref{eq:oco-11} instead of applying Lemma~\ref{lem:mix} to it, since the last execution of Line~\ref{line:stch-w} produces \(\boldsymbol w^{T}\) and performs no further mixing.
The sum over \(t=1,\ldots,T-1\) cancels all intermediate KL terms, leaving only \(-\frac{\mu(T-1)}{T}D_{\mathrm{KL}}(\boldsymbol w\Vert\boldsymbol w^{T})\le0\), which we discard.
Therefore, we obtain
\begin{align*}
&\frac1T\sum_{t=1}^{T-1}
\Big[
\big\langle \boldsymbol{w}-\boldsymbol{w}^t,\,
\boldsymbol{\lambda}\odot
(\widetilde{\boldsymbol V}^{t+1}
+\boldsymbol{\iota}
-\boldsymbol V^{t+1})
\big\rangle
-\mu\big\langle \boldsymbol{w},\log\boldsymbol{w}\big\rangle
+\mu\big\langle \boldsymbol{w}^t,\log\boldsymbol{w}^t\big\rangle
\Big] \\
&\quad\le
\frac1T\sum_{t=1}^{T-1}
\Big[
\frac{2\mu\log m}{t}
+\frac{(3H+2\mu\log(mt))^2}{\mu t}
+\frac{2\mu}{t^3}
\Big] 
\le
\frac{18H^2\log T}{\mu T}
+
\frac{11\mu\log^3(mT)}{T}.
\end{align*}
This completes the proof.
\end{proof}

\begin{lemma}
\label{lem:realized-bonus-sum}
Let $\sigma_h^t$ and $\zeta_h^t$ be defined as in~\eqref{eq:reward_confidence} and~\eqref{eq:transition-width-bonus}, respectively.
For any sequence of trajectories, let $\{(s_h^t,a_h^t)\}_{h=1}^H$ denote the trajectory collected in episode $t\in[T]$, and construct the counts and bonuses from this sequence.
Then we have
\begin{align*}
\sum_{t=1}^{T-1}\sum_{h=1}^H
\sigma_h^t(s_h^{t+1},a_h^{t+1})
&\le
\mathcal O\!\left(
H\sqrt{
T|\mathcal S||\mathcal A|
\ln\!\left(\frac{m|\mathcal S||\mathcal A|HT}{\delta}\right)
}
\right),
\\
\sum_{t=1}^{T-1}\sum_{h=1}^H
\zeta_h^t(s_h^{t+1},a_h^{t+1})
&\le
\mathcal O\!\left(
H^2
\sqrt{
T|\mathcal S|^2|\mathcal A|
\ln\!\left(\frac{m|\mathcal S||\mathcal A|HT}{\delta}\right)
}
\right).
\end{align*}
\end{lemma}

\begin{proof}
The bound~\eqref{eq:bridge-realized-count} applies at every step $h\in[H]$.
By~\eqref{eq:reward_confidence},
we have
$\sigma_h^t(s,a)
\le
\sqrt{
\frac{
2\ln(8m|\mathcal S||\mathcal A|HT/\delta)
}
{\max\{1,N_h^t(s,a)\}}
}.$
Therefore, using~\eqref{eq:bridge-realized-count} and summing over $h\in[H]$, we obtain
\begin{align*}
\sum_{t=1}^{T-1}\sum_{h=1}^H
\sigma_h^t(s_h^{t+1},a_h^{t+1})
\le \mathcal O\!\left(
H\sqrt{
T|\mathcal S||\mathcal A|
\ln\!\left(\frac{m|\mathcal S||\mathcal A|HT}{\delta}\right) }\right).
\end{align*}
By~\eqref{eq:transition-width-bonus}, we have
$\zeta_h^t(s,a)
\le
H
\sqrt{
\frac{
8|\mathcal S|\ln(2mTH|\mathcal S||\mathcal A|/\delta)
}
{\max\{1,N_h^t(s,a)\}}
}.$
Using~\eqref{eq:bridge-realized-count} and summing over $h\in[H]$,
we have
\begin{align*}
\sum_{t=1}^{T-1}\sum_{h=1}^H
\zeta_h^t(s_h^{t+1},a_h^{t+1})
\le \mathcal O\!\left(
H^2
\sqrt{
T|\mathcal S|^2|\mathcal A|
\ln\!\left(\frac{m|\mathcal S||\mathcal A|HT}{\delta}\right)
}
\right).
\end{align*}
This completes the proof.
\end{proof}

\subsection{Proof of Theorem~\ref{thm:subopt}}
\begin{proof}
Let $\bar\theta$ be the averaged occupancy measure returned by Algorithm~\ref{alg:CSTRL}.
We use the optimal occupancy measures $\theta_i^*$, $i\in[m]$, and $\theta_{\boldsymbol\lambda}^*$ introduced in Section~\ref{sec:sketch-subopt}.
Since each $V_i$ is linear in $\theta$ and the log-sum-exp function is convex, $\stchlm(\pi^\theta)$ is convex in $\theta$ on $\Theta(P)$.
Moreover, for any $\theta,\theta'\in\Theta(P)$ and every $i\in[m]$, we have
$\lambda_i(V_i^*+\iota-V_i(\theta))
\le
\lambda_i(V_i^*+\iota-V_i(\theta'))
+
\max_{k\in[m]}\lambda_k\bigl(V_k(\theta')-V_k(\theta)\bigr)$.
Taking the exponential, summing over $i\in[m]$, and taking the logarithm, we obtain
\begin{align}
\stchlm(\pi^\theta)-\stchlm(\pi^{\theta'})
\le
\max_{i\in[m]}\lambda_i\bigl(V_i(\theta')-V_i(\theta)\bigr)
\le
H,
\label{eq:stch-nonexpansive}
\end{align}
where the last inequality uses $\lambda_i\le1$ and $V_i(\theta)\in[0,H]$.
For $T=1$, \eqref{eq:stch-nonexpansive} gives
$\mathrm{SubOpt}(\widehat\pi)\le H$, which is dominated by the stated bound.
We assume $T\ge2$ below and define
$\bar\theta_{2:T}
:=
\frac{1}{T-1}\sum_{t=1}^{T-1}\theta^{P,\pi^{t+1}}$.
Then
$\bar\theta
=
\frac1T\theta^{P,\pi^{1}}
+
\frac{T-1}{T}\bar\theta_{2:T}.$
By the convexity of $\stchlm(\pi^\theta)$ in $\theta$ and~\eqref{eq:stch-nonexpansive}, we obtain
\begin{align}
\mathrm{SubOpt}(\widehat\pi) 
&=
\stchlm(\pi^{\bar\theta})-\stchlm(\pi^{\theta_{\boldsymbol\lambda}^*})\nonumber \\
& \le
\frac{H}{T}
+
\frac{T-1}{T}
\Bigl(
\stchlm(\pi^{\bar\theta_{2:T}})-\stchlm(\pi^{\theta_{\boldsymbol\lambda}^*})
\Bigr)\nonumber \\
&= \frac{H}{T}
+
\frac{T-1}{T}
\max_{\boldsymbol w\in\Delta_m}
\sum_{i=1}^m
\left[
w_i\lambda_i
\left(
V_i^*+\iota-V_i(\bar\theta_{2:T})
\right)
-\mu w_i\log w_i
\right]
\nonumber\\
&\quad -
\frac1T\sum_{t=1}^{T-1}
\sum_{i=1}^m
\left[
w_i^t\lambda_i
\left(
V_i(\theta^{P,\widetilde\nu_i^{t+1}})
+\iota
-
V_i(\theta^{P,\pi^{t+1}})
\right)
-\mu w_i^t\log w_i^t
\right]
\nonumber\\
&\quad +
\frac1T\sum_{t=1}^{T-1}
\sum_{i=1}^m
\left[
w_i^t\lambda_i
\left(
V_i(\theta^{P,\widetilde\nu_i^{t+1}})
+\iota
-
V_i(\theta^{P,\pi^{t+1}})
\right)
-\mu w_i^t\log w_i^t
\right]
\nonumber\\
&\quad -
\frac{T-1}{T}
\max_{\boldsymbol w\in\Delta_m}
\sum_{i=1}^m
\left[
w_i\lambda_i
\left(
V_i^*+\iota-V_i(\theta_{\boldsymbol{\lambda}}^*)
\right)
-\mu w_i\log w_i
\right],
\label{eq:stch-decomp1}
\end{align}
where the last equality uses~\eqref{eq:eq-stch}.
We first bound the following difference.
\begin{align}
&\frac{T-1}{T}
\max_{\boldsymbol w\in\Delta_m}
\sum_{i=1}^m
\left[
w_i\lambda_i
\left(
V_i^*+\iota-V_i(\bar\theta_{2:T})
\right)
-\mu w_i\log w_i
\right]
\nonumber\\
&\,\, -
\frac1T\sum_{t=1}^{T-1}
\sum_{i=1}^m
\left[
w_i^t\lambda_i
\left(
V_i(\theta^{P,\widetilde\nu_i^{t+1}})
+\iota
-
V_i(\theta^{P,\pi^{t+1}})
\right)
-\mu w_i^t\log w_i^t
\right]
\nonumber\\
&\le
\underbrace{
\max_{\boldsymbol w\in\Delta_m}
\sum_{i=1}^m
w_i\lambda_i
\left[
\frac1T\sum_{t=1}^{T-1}
\Bigl(
V_i^*
-
V_i(\theta^{P,\widetilde\nu_i^{t+1}})
\Bigr)
\right]
}_{\mathrm{(I)}}
\nonumber\\
&\,\, +
\underbrace{
\max_{\boldsymbol w\in\Delta_m}
\frac1T\sum_{t=1}^{T-1}
\sum_{i=1}^m
\left[
(w_i-w_i^t)\lambda_i
\left(
V_i(\theta^{P,\widetilde\nu_i^{t+1}})
+\iota
-
V_i(\theta^{P,\pi^{t+1}})
\right)
-\mu w_i\log w_i
+\mu w_i^t\log w_i^t
\right]
}_{\mathrm{(II)}} .
\label{eq:stch-decomp-first}
\end{align}
This inequality uses the definition of \(\bar\theta_{2:T}\), the linearity of \(V_i\), and the fact that the maximum of a sum is at most the sum of the maxima.
The first term is the auxiliary single-objective learning error, and the second term is the smooth Tchebycheff dual optimization error.
We next bound the difference between the last two terms in \eqref{eq:stch-decomp1}. 
On $\mathcal E_{R,C,T}(\delta)$, Theorem~\ref{thm:safety} gives $\theta^{P,\nu_i^{t+1}}\in\Theta_{\mathrm{safe}}$.
Since $\theta_i^*\in\arg\max_{\theta\in\Theta_{\mathrm{safe}}}V_i(\theta)$, we have $V_i(\theta^{P,\nu_i^{t+1}})\le V_i(\theta_i^*)=V_i^*$.
Moreover, the construction of \(\nu_i^{t+1}\) gives
$\theta^{P,\nu_i^{t+1}}
=
(1-\psi_i^t)\theta^{P,\widetilde\nu_i^{t+1}}
+
\psi_i^t\theta^{P,\pi^\dagger}.$
Since \(0\le V_i(\theta)\le H\), 
it follows that
$V_i(\theta^{P,\widetilde\nu_i^{t+1}})
\le
V_i^*+H\psi_i^t.$
Therefore, for any \(\boldsymbol w\in\Delta_m\),
we obtain
\begin{align}
&\frac1T\sum_{t=1}^{T-1}
\sum_{i=1}^m
\left[
w_i^t\lambda_i
\left(
V_i(\theta^{P,\widetilde\nu_i^{t+1}})
+\iota
-
V_i(\theta^{P,\pi^{t+1}})
\right)
-\mu w_i^t\log w_i^t
\right]
\nonumber\\
&\, -
\frac{T-1}{T}
\max_{\boldsymbol w\in\Delta_m}
\sum_{i=1}^m
\left[
w_i\lambda_i
\left(
V_i^*+\iota-V_i(\theta_{\boldsymbol{\lambda}}^*)
\right)
-\mu w_i\log w_i
\right]
\nonumber\\
&\le
\frac1T\sum_{t=1}^{T-1}
\sum_{i=1}^m
\left[
w_i^t\lambda_i
\left(
V_i(\theta^{P,\widetilde\nu_i^{t+1}})
+\iota
-
V_i(\theta^{P,\pi^{t+1}})
\right)
-\mu w_i^t\log w_i^t
\right]
\nonumber\\
&\, -
\frac1T\sum_{t=1}^{T-1}
\max_{\boldsymbol w\in\Delta_m}
\sum_{i=1}^m
\left[
w_i\lambda_i
\left(
V_i(\theta^{P,\widetilde\nu_i^{t+1}})
+\iota
-
V_i(\theta_{\boldsymbol{\lambda}}^*)
\right)
-\mu w_i\log w_i
\right]
\!+\!
\frac{H}{T}\sum_{t=1}^{T-1}
\max_{\boldsymbol w\in\Delta_m}
\sum_{i=1}^m
w_i\lambda_i\psi_i^t
\nonumber\\
&\le
\underbrace{
\frac1T\sum_{t=1}^{T-1}
\sum_{i=1}^m
w_i^t\lambda_i
\left(
V_i(\theta_{\boldsymbol{\lambda}}^*)-V_i(\theta^{P,\pi^{t+1}})
\right)
}_{\mathrm{(III)}}
+
\underbrace{
\frac{H}{T}\sum_{t=1}^{T-1}
\max_{\boldsymbol w\in\Delta_m}
\sum_{i=1}^m
w_i\lambda_i\psi_i^t
}_{\mathrm{(IV)}} .
\label{eq:stch-decomp-second}
\end{align}

Combining the decomposition in~\eqref{eq:stch-decomp1} with~\eqref{eq:stch-decomp-first} and~\eqref{eq:stch-decomp-second}, we obtain
\begin{align}
\frac{T-1}{T}
\Bigl(
\stchlm(\pi^{\bar\theta_{2:T}})
-
\stchlm(\pi^{\theta_{\boldsymbol{\lambda}}^*})
\Bigr)
&\le
\mathrm{(I)}+\mathrm{(II)}+\mathrm{(III)}+\mathrm{(IV)}.
\label{eq:stch-decomp-sum}
\end{align}

\vspace{5pt}
\noindent\textbf{Term \(\mathrm{(I)}\).}
Since \(V_i(\theta)=\langle r_i,\theta\rangle\), 
Lemma~\ref{lem:aux-occ-regret} implies that, 
for every \(i\in[m]\),
$\frac1T\sum_{t=1}^{T-1}
(V_i^*
-
V_i(\theta^{P,\widetilde\nu_i^{t+1}}))
\le
\widetilde{\mathcal O}\!(
H^2|\mathcal S|^{3/2}\sqrt{|\mathcal A|/T})
+
\frac{4H}{T}\sum_{t=1}^{T-1}\psi_i^t .$
Therefore, we have
\begin{align}
\mathrm{(I)}
\le
\widetilde{\mathcal O}\!\left(
H^2|\mathcal S|^{3/2}\sqrt{|\mathcal A|/T}
\right)
+
\frac{4H}{T}
\sum_{i=1}^m
\lambda_i
\sum_{t=1}^{T-1}\psi_i^t .
\label{eq:term-i-bound}
\end{align}
The inequality uses \(w_i\ge 0\), \(\sum_i w_i=1\), and \(\boldsymbol\lambda\in\Delta_m\).

It remains to bound the term involving $\psi_i^t$ in~\eqref{eq:term-i-bound}.
Fix an objective \(i\in[m]\) and an episode \(t\in[T-1]\).
On the confidence event, we have \(P\in\mathcal P^t\), and \(\theta_i^*\in\Theta_{\mathrm{safe}}\) is feasible for the auxiliary projection used to construct the policy for episode \(t+1\).
Recall that
$\phi_{\mathrm{safe}}
=
\max_{j\in[n]}
d^2/(\tau_j-\gamma_j)^2$.
Since $0<\tau_j-\gamma_j\le d$ and $d\ge1$, we have
$1/(\tau_j-\gamma_j)
\le
\sqrt{\phi_{\mathrm{safe}}}
\le
\phi_{\mathrm{safe}}$
for every $j\in[n]$.
Moreover, for every \(P'\in\mathcal P^t\), the bound~\eqref{eq:bridge-transition-bound} also holds with \(P^{t+1}\) replaced by \(P'\), and Lemma~\ref{lem:simulation} gives
$\|
\theta^{P',\widetilde\nu_i^{t+1}}
-
\theta^{P,\widetilde\nu_i^{t+1}}
\|_1
\le
2H|\mathcal S|
\sum_{h=1}^H\sum_{s,a}
\theta_h^{P,\widetilde\nu_i^{t+1}}(s,a)\,
\epsilon_h^t(s,a).$
We apply Lemma~\ref{lem:mixing-bound} with \(\widetilde\theta=\widetilde\theta_i^{t+1}\) and use
$(1-\psi_i^t)\theta_h^{P,\widetilde\nu_i^{t+1}}(s,a)
\le
\theta_h^{P,\nu_i^{t+1}}(s,a)$
for all \(h,s,a\), which follows from the construction of \(\nu_i^{t+1}\).
Then we obtain
\begin{align}
\psi_i^t
\le
\phi_{\mathrm{safe}}
\Bigg[
6H|\mathcal S|
\sum_{h=1}^H\sum_{s,a}
\theta_h^{P,\nu_i^{t+1}}(s,a)\,
\epsilon_h^t(s,a)
+
2\max_{j\in[n]}
\sum_{h=1}^H\sum_{s,a}
\theta_h^{P,\nu_i^{t+1}}(s,a)\,
\xi_{j,h}^t(s,a)
\Bigg],
\label{eq:term-i-eta-one-step}
\end{align}
where \(\epsilon_h^t\) and \(\xi_{j,h}^t\) are computed from the auxiliary trajectories generated by \(\{\nu_i^k\}_{k=1}^t\).
By~\eqref{eq:transition_confidence} and~\eqref{eq:cost_confidence}, both \(\epsilon_h^t(s,a)\) and \(\xi_{j,h}^t(s,a)\) are at most \(\widetilde{\mathcal O}(1/\sqrt{\max\{1,N_h^t(s,a)\}})\).
Summing~\eqref{eq:term-i-eta-one-step} over \(t=1,\ldots,T-1\) and applying~\eqref{eq:bridge-visitation-final} to the auxiliary trajectories, we obtain, for every \(i\in[m]\),
\begin{align}
\sum_{t=1}^{T-1}\psi_i^t
&\le
\widetilde{\mathcal O}\!\left(
\phi_{\mathrm{safe}}H^2|\mathcal S|^{3/2}
\sqrt{|\mathcal A|T}
\right).
\label{eq:aux-eta-sum-bound}
\end{align}

Since \(\boldsymbol\lambda\in\Delta_m\), substituting~\eqref{eq:aux-eta-sum-bound} into~\eqref{eq:term-i-bound} yields
\begin{align}
\mathrm{(I)}
&\le
\widetilde{\mathcal O}\!\left(
H^2|\mathcal S|^{3/2}\sqrt{|\mathcal A|/T}
+
\phi_{\mathrm{safe}}H^3|\mathcal S|^{3/2}
\sqrt{|\mathcal A|/T}
\right).
\label{eq:term-i-final}
\end{align}

\vspace{5pt}
\noindent\textbf{Term \(\mathrm{(II)}\).}
Since
$|w_i-w_i^t|\le w_i+w_i^t\le 2$, 
we obtain
\begin{align}
\mathrm{(II)}
&\le
\max_{\boldsymbol w\in\Delta_m}
\frac1T\sum_{t=1}^{T-1}
\left[
\left\langle
\boldsymbol w-\boldsymbol w^t,\,
\boldsymbol\lambda\odot
\left(
\widetilde{\boldsymbol V}^{t+1}
+\boldsymbol \iota
-{\boldsymbol V}^{t+1}
\right)
\right\rangle
-\mu\langle\boldsymbol w,\log\boldsymbol w\rangle
+\mu\langle\boldsymbol w^t,\log\boldsymbol w^t\rangle
\right]
\nonumber\\
&\quad+
\frac{2}{T}\sum_{t=1}^{T-1}\sum_{i=1}^m
\lambda_i
\left|
V_i(\theta^{P,\widetilde\nu_i^{t+1}})
-
\widetilde V_i^{t+1}
\right|
+
\frac{2}{T}\sum_{t=1}^{T-1}\sum_{i=1}^m
\lambda_i
\left|
V_i^{t+1}
-
V_i(\theta^{P,\pi^{t+1}})
\right|.
\label{eq:term-ii-decomp-new}
\end{align}
Lemma~\ref{lem:oco-stch} bounds the first term on the right-hand side of~\eqref{eq:term-ii-decomp-new} by \(\widetilde{\mathcal O}\!\left(H^2/(\mu T)+\mu/T\right)\).

We next bound the auxiliary estimation term.
For each \(i\in[m]\), using
\(\widetilde\theta_i^{t+1}
=
\theta^{P_i^{t+1},\widetilde\nu_i^{t+1}}\) for some
\(P_i^{t+1}\in\mathcal P^t\), we have
\begin{align*}
\left|
V_i(\theta^{P,\widetilde\nu_i^{t+1}})
-
\widetilde V_i^{t+1}
\right|
=
\left|
\left\langle
r_i,\theta^{P,\widetilde\nu_i^{t+1}}
\right\rangle
-
\left\langle
\bar r_i^t,\widetilde\theta_i^{t+1}
\right\rangle
\right|
\le
\left\|
\theta^{P,\widetilde\nu_i^{t+1}}
-
\theta^{P_i^{t+1},\widetilde\nu_i^{t+1}}
\right\|_1
+
2\left\langle
\sigma^t,\widetilde\theta_i^{t+1}
\right\rangle.
\end{align*}
The unweighted form of Lemma~\ref{lem:aux-weighted-transition-bridge} and the reward bonus bound used in Lemma~\ref{lem:aux-occ-regret} imply
$\frac1T\sum_{t=1}^{T-1}
|V_i(\theta^{P,\widetilde\nu_i^{t+1}})
-
\widetilde V_i^{t+1}|
\le
\widetilde{\mathcal O}\!(
H^2|\mathcal S|^{3/2}\sqrt{|\mathcal A|/T})
+
\frac{4H}{T}\sum_{t=1}^{T-1}\psi_i^t .$
Using $\boldsymbol\lambda\in\Delta_m$ together with~\eqref{eq:aux-eta-sum-bound},
we obtain
\begin{align*}
\frac{2}{T}\sum_{t=1}^{T-1}\sum_{i=1}^m
\lambda_i
\left|
V_i(\theta^{P,\widetilde\nu_i^{t+1}})
-
\widetilde V_i^{t+1}
\right|
\le
\widetilde{\mathcal O}\!\left(
H^2|\mathcal S|^{3/2}\sqrt{|\mathcal A|/T}
+
\phi_{\mathrm{safe}}H^3|\mathcal S|^{3/2}\sqrt{|\mathcal A|/T}
\right).
\end{align*}

It remains to bound the main estimation term.
By construction, we have
\(\widetilde\theta^{t+1}
=
\theta^{P^{t+1},\widetilde\pi^{t+1}}\) for some
\(P^{t+1}\in\mathcal P^t\), and
$\theta^{P,\pi^{t+1}}
=
(1-\psi^t)\theta^{P,\widetilde\pi^{t+1}}
+
\psi^t\theta^{P,\pi^\dagger}.$
Thus, for every \(i\in[m]\),
we have
\begin{align*}
|
V_i^{t+1}
\!-\!
V_i(\theta^{P,\pi^{t+1}})
|
\!=\!
|
\langle
\bar r_i^t,\widetilde\theta^{t+1}
\rangle
\!-\!
\langle
r_i,\theta^{P,\pi^{t+1}}
\rangle|
\!\le\!
2\langle
\sigma^t,\widetilde\theta^{t+1}
\rangle
\!+\!
\|
\theta^{P^{t+1},\widetilde\pi^{t+1}}
-
\theta^{P,\widetilde\pi^{t+1}}
\|_1
\!+\!
H\psi^t .
\end{align*}
The first term is controlled by the same reward confidence argument used in the proof of Lemma~\ref{lem:main-occ-regret}, using the bonus \(\sigma^t\).
The second term is controlled by the unweighted form of Lemma~\ref{lem:weighted-transition-bridge}. 
Hence,
we have
\begin{align}
\frac1T\sum_{t=1}^{T-1}\sum_{i=1}^m
\lambda_i
\left|
V_i^{t+1}
-
V_i(\theta^{P,\pi^{t+1}})
\right|
&\le
\widetilde{\mathcal O}\!\left(
H^2|\mathcal S|^{3/2}\sqrt{|\mathcal A|/T}
\right)
+
\frac{5H}{T}\sum_{t=1}^{T-1}\psi^t .
\label{eq:term-ii-main-error}
\end{align}
The sum $\sum_{t=1}^{T-1}\psi^t$ is bounded by the same argument used in the proof of~\eqref{eq:aux-eta-sum-bound}, with $(\widetilde\nu_i^{t+1},\nu_i^{t+1},\psi_i^t)$ replaced by $(\widetilde\pi^{t+1},\pi^{t+1},\psi^t)$. 
Therefore, we have
\begin{align}
\sum_{t=1}^{T-1}\psi^t
&\le
\widetilde{\mathcal O}\!\left(
\phi_{\mathrm{safe}}H^2|\mathcal S|^{3/2}\sqrt{|\mathcal A|T}
\right).
\label{eq:main-eta-sum-bound}
\end{align}
Combining~\eqref{eq:term-ii-main-error} and
\eqref{eq:main-eta-sum-bound}, we obtain
\begin{align*}
\frac{2}{T}\sum_{t=1}^{T-1}\sum_{i=1}^m
\lambda_i
\left|
V_i^{t+1}
-
V_i(\theta^{P,\pi^{t+1}})
\right|
&\le
\widetilde{\mathcal O}\!\left(
H^2|\mathcal S|^{3/2}\sqrt{|\mathcal A|/T}
+
\phi_{\mathrm{safe}}H^3|\mathcal S|^{3/2}\sqrt{|\mathcal A|/T}
\right).
\end{align*}

Combining~\eqref{eq:term-ii-decomp-new} with the bounds established above for its three terms, we obtain
\begin{align}
\mathrm{(II)}
&\le
\widetilde{\mathcal O}\!\left(
H^2|\mathcal S|^{3/2}\sqrt{|\mathcal A|/T}
+
\phi_{\mathrm{safe}}H^3|\mathcal S|^{3/2}\sqrt{|\mathcal A|/T}
+
\frac{H^2}{\mu T}
+
\frac{\mu}{T}
\right).
\label{eq:term-ii-final}
\end{align}

\vspace{5pt}
\noindent\textbf{Term \(\mathrm{(III)}\).}
Since \(V_i(\theta)=\langle r_i,\theta\rangle\), we have
\begin{align}
\mathrm{(III)}
=
\frac1T\sum_{t=1}^{T-1}
\left\langle
\sum_{i=1}^m w_i^t\lambda_i r_i,\,
\theta_{\boldsymbol{\lambda}}^*-\widetilde\theta^{t+1}
\right\rangle
+
\frac1T\sum_{t=1}^{T-1}
\left\langle
\sum_{i=1}^m w_i^t\lambda_i r_i,\,
\widetilde\theta^{t+1}-\theta^{P,\pi^{t+1}}
\right\rangle .
\label{eq:term-iii-decomp}
\end{align}

The first term in~\eqref{eq:term-iii-decomp} is controlled by
Lemma~\ref{lem:main-occ-regret} with 
\(\theta=\theta_{\boldsymbol{\lambda}}^*\). 
Since
\(\theta_{\boldsymbol{\lambda}}^*\in\Theta_{\mathrm{safe}}\), we have
\begin{align}
\frac1T\sum_{t=1}^{T-1}
\left\langle
\sum_{i=1}^m w_i^t\lambda_i r_i,\,
\theta_{\boldsymbol{\lambda}}^*-\widetilde\theta^{t+1}
\right\rangle
&\le
\widetilde{\mathcal O}\!\left(
H^2|\mathcal S|^{3/2}\sqrt{|\mathcal A|/T}
\right)
+
\frac{2H}{T}\sum_{t=1}^{T-1}\psi^t .
\label{eq:term-iii-main-regret}
\end{align}

We next bound the second term in~\eqref{eq:term-iii-decomp}. 
By construction, we have
$\widetilde\theta^{t+1}
=
\theta^{P^{t+1},\widetilde\pi^{t+1}}$
for some \(P^{t+1}\in\mathcal P^t\).
At the beginning of episode $t+1$,
\texttt{SafeProj} selects $\pi^\dagger$ with probability $\psi^t$
and $\widetilde\pi^{t+1}$ with probability $1-\psi^t$.
Thus, we have
$\theta^{P,\pi^{t+1}}
=
(1-\psi^t)\theta^{P,\widetilde\pi^{t+1}}
+
\psi^t\theta^{P,\pi^\dagger}$.
Moreover,
since $r_{i,h}(s,a)\in[0,1]$ and
$\boldsymbol w^t,\boldsymbol\lambda\in\Delta_m$,
we have
$0\le
\sum_{i=1}^m w_i^t\lambda_i r_{i,h}(s,a)
\le 1,
\, \forall h,s,a,t.$
Therefore, for every $t$, we obtain
\begin{align}
&\left\langle
\sum_{i=1}^m w_i^t\lambda_i r_i,\,
\widetilde\theta^{t+1}-\theta^{P,\pi^{t+1}}
\right\rangle
\nonumber\\
&\quad=
\left\langle
\sum_{i=1}^m w_i^t\lambda_i r_i,\,
\theta^{P^{t+1},\widetilde\pi^{t+1}}
-
\theta^{P,\widetilde\pi^{t+1}}
\right\rangle
+
\left\langle
\sum_{i=1}^m w_i^t\lambda_i r_i,\,
\theta^{P,\widetilde\pi^{t+1}}
-
\theta^{P,\pi^{t+1}}
\right\rangle
\nonumber\\
&\quad\le
\left\|
\theta^{P^{t+1},\widetilde\pi^{t+1}}
-
\theta^{P,\widetilde\pi^{t+1}}
\right\|_1
+
H\psi^t .
\label{eq:term-iii-bridge-one-step}
\end{align}
The last inequality follows from
$\theta^{P,\widetilde\pi^{t+1}}
-
\theta^{P,\pi^{t+1}}
=
\psi^t
(\theta^{P,\widetilde\pi^{t+1}}
-
\theta^{P,\pi^\dagger})$
and the fact that any $H$-step occupancy measure has total mass $H$.
Averaging~\eqref{eq:term-iii-bridge-one-step} over \(t=1,\ldots,T-1\) and using
the unweighted form of Lemma~\ref{lem:weighted-transition-bridge}, we obtain
\begin{align}
\frac1T\sum_{t=1}^{T-1}
\left\langle
\sum_{i=1}^m w_i^t\lambda_i r_i,\,
\widetilde\theta^{t+1}-\theta^{P,\pi^{t+1}}
\right\rangle
&\le
\widetilde{\mathcal O}\!\left(
H^2|\mathcal S|^{3/2}\sqrt{|\mathcal A|/T}
\right)
+
\frac{3H}{T}\sum_{t=1}^{T-1}\psi^t .
\label{eq:term-iii-main-bridge}
\end{align}

Combining~\eqref{eq:term-iii-main-regret} and
\eqref{eq:term-iii-main-bridge}, we obtain
$\mathrm{(III)}
\le
\widetilde{\mathcal O}(
H^2|\mathcal S|^{3/2}\sqrt{|\mathcal A|/T}
)
+
\frac{5H}{T}\sum_{t=1}^{T-1}\psi^t .$
Using the bound in~\eqref{eq:main-eta-sum-bound}, namely
$\sum_{t=1}^{T-1}\psi^t
\le \allowbreak
\widetilde{\mathcal O}\!(
\phi_{\mathrm{safe}}H^2|\mathcal S|^{3/2}\sqrt{|\mathcal A|T}),$
we conclude that
\begin{align}
\mathrm{(III)}
&\le
\widetilde{\mathcal O}\!\left(
H^2|\mathcal S|^{3/2}\sqrt{|\mathcal A|/T}
+
\phi_{\mathrm{safe}}H^3|\mathcal S|^{3/2}\sqrt{|\mathcal A|/T}
\right).
\label{eq:term-iii-final}
\end{align}

\vspace{5pt}
\noindent\textbf{Term \(\mathrm{(IV)}\).}
Using \(w_i\ge 0\), \(\sum_{i=1}^m w_i=1\), and \(\boldsymbol\lambda\in\Delta_m\), we have
$\mathrm{(IV)}
\le
\frac{H}{T}
\sum_{i=1}^m
\lambda_i
\sum_{t=1}^{T-1}\psi_i^t .$
By~\eqref{eq:aux-eta-sum-bound}, we obtain
\begin{align}
\mathrm{(IV)}
\le
\widetilde{\mathcal O}\!\left(
\phi_{\mathrm{safe}}H^3|\mathcal S|^{3/2}
\sqrt{|\mathcal A|/T}
\right).
\label{eq:term-iv-final}
\end{align}

Consequently, by combining \eqref{eq:stch-decomp1}, \eqref{eq:stch-decomp-sum}, \eqref{eq:term-i-final}, \eqref{eq:term-ii-final}, \eqref{eq:term-iii-final}, and \eqref{eq:term-iv-final}, and noting that $H/T$ is dominated by $H^2|\mathcal S|^{3/2}\sqrt{|\mathcal A|/T}$, we have, with probability at least $1-\delta$,
\begin{align*}
\mathrm{SubOpt}(\widehat \pi) 
\leq \widetilde{\mathcal O}\Biggl(
H^2|\mathcal S|^{3/2}\sqrt{\frac{|\mathcal A|}{T}}
+
\phi_{\mathrm{safe}}H^3|\mathcal S|^{3/2}\sqrt{\frac{|\mathcal A|}{T}} 
+
\frac{H^2}{\mu T}
+
\frac{\mu}{T} \Biggr).
\end{align*}
This completes the proof.
\end{proof}

\section{Proof of Theorem~\ref{thm:safety}}
\label{app:safety}
\begin{proof}
We work on the event $\mathcal E_{R,C,T}(\delta)$, which holds with probability at least $1-\delta$.
On this event, we have $|\widehat C^t-C|\le\Xi^t$ and $P\in\mathcal P^t$ for every $t\in[T-1]$.

We first verify the safety of the initial executed policy $\pi^1$.
Algorithm~\ref{alg:CSTRL} sets $\psi^0=\max_{j\in[n]}(d-\tau_j)/(d-\gamma_j)$.
Since $0\le\gamma_j<\tau_j\le d$, we have $\psi^0\in[0,1]$.
At the beginning of the first episode, the algorithm selects $\pi^\dagger$ with probability $\psi^0$ and $\widetilde\pi^1$ with probability $1-\psi^0$.
Thus, by the law of total probability, we have
$\theta^{P,\pi^1}
=
(1-\psi^0)\theta^{P,\widetilde\pi^1}
+
\psi^0\theta^{P,\pi^\dagger}.$
For every $j\in[n]$, the constraint cost of $\widetilde\pi^1$ is at most $d$, while that of
$\pi^\dagger$ equals $\gamma_j$.
Therefore, we have
$(C^\top\theta^{P,\pi^1})_j
\le
(1-\psi^0)d+\psi^0\gamma_j
=
d-\psi^0(d-\gamma_j)
\le
\tau_j,$
where the last inequality follows from $\psi^0\ge(d-\tau_j)/(d-\gamma_j)$.
Thus, we have $C^\top\theta^{P,\pi^1}\le\boldsymbol\tau$.

It remains to show that, for every \(t\in[T-1]\), the policy \(\pi^{t+1}\) constructed after episode \(t\) and executed in episode \(t+1\) is safe.
Fix any $t\in[T-1]$ and $j\in[n]$.
We prove $(C^\top\theta^{P,\pi^{t+1}})_j\le\tau_j$.
When the projection is feasible, the cost confidence bound, $P\in\mathcal P^t$, and the fact that every constraint cost is at most $d$ give
\begin{equation}
(C^\top\theta^{P,\widetilde\pi^{t+1}})_j
\le
\min\{
\max_{P'\in\mathcal P^t}
\sum_{h=1}^H
(\widehat c_{j,h}^t+\xi_{j,h}^t)^\top
\theta_h^{P',\widetilde\pi^{t+1}}, d\}.
\label{eq:safety-pessimistic-bound}
\end{equation}
We consider three cases.

\emph{Case 1:
$\texttt{Proj}(\widehat\theta^{t+1},
\mathcal P^t,\widehat C^t,\Xi^t)$ is infeasible.}
In this case, \texttt{SafeProj} sets $\psi^t=1$, and hence
$\pi^{t+1}=\pi^\dagger$.
Therefore, we have
$(C^\top\theta^{P,\pi^{t+1}})_j =\gamma_j<\tau_j$.

\emph{Case 2:
$\texttt{Proj}(\widehat\theta^{t+1},
\mathcal P^t,\widehat C^t,\Xi^t)$ is feasible and $\psi^t=0$.}
In this case, we have $\pi^{t+1}=\widetilde\pi^{t+1}$.
Since the denominator of every ratio in~\eqref{eq:eta_t} is at least $\tau_j-\gamma_j>0$, the definition of $\psi^t$ implies that the right-hand side of~\eqref{eq:safety-pessimistic-bound} is at most $\tau_j$.
It follows that $(C^\top\theta^{P,\pi^{t+1}})_j\le\tau_j$.

\emph{Case 3:
$\texttt{Proj}(\widehat\theta^{t+1}, \mathcal P^t,\widehat C^t,\Xi^t)$ is feasible and $\psi^t\in(0,1)$.}
At the beginning of episode $t+1$, \texttt{SafeProj} selects $\pi^\dagger$ with probability $\psi^t$ and $\widetilde\pi^{t+1}$ with probability $1-\psi^t$.
By the law of total probability, we have
$\theta^{P,\pi^{t+1}}
=
(1-\psi^t)\theta^{P,\widetilde\pi^{t+1}}
+
\psi^t\theta^{P,\pi^\dagger}.$
Consequently, we have
\begin{align}
\bigl(C^\top\theta^{P,\pi^{t+1}}\bigr)_j
\le
(1-\psi^t)
\min\left\{
\max_{P'\in\mathcal P^t}
\sum_{h=1}^H
\bigl(\widehat c_{j,h}^t+\xi_{j,h}^t\bigr)^\top
\theta_h^{P',\widetilde\pi^{t+1}},
d
\right\}
+
\psi^t\gamma_j.
\label{eq:safety-case3-mix}
\end{align}
If the minimum in~\eqref{eq:safety-case3-mix} is at most $\tau_j$, then the right-hand side of~\eqref{eq:safety-case3-mix} is at most $(1-\psi^t)\tau_j+\psi^t\gamma_j\le\tau_j$.
Otherwise, the definition of $\psi^t$ gives
\[
\psi^t
\ge
\frac{
\min\left\{
\max_{P'\in\mathcal P^t}
\sum_{h=1}^H
\bigl(\widehat c_{j,h}^t+\xi_{j,h}^t\bigr)^\top
\theta_h^{P',\widetilde\pi^{t+1}},
d
\right\}
-\tau_j
}{
\min\left\{
\max_{P'\in\mathcal P^t}
\sum_{h=1}^H
\bigl(\widehat c_{j,h}^t+\xi_{j,h}^t\bigr)^\top
\theta_h^{P',\widetilde\pi^{t+1}},
d
\right\}
-\gamma_j
}.
\]
Rearranging this inequality shows that the right-hand side of~\eqref{eq:safety-case3-mix} is at most $\tau_j$.
Since $t\in[T-1]$ and $j\in[n]$ were arbitrary, every generated policy $\pi^{t+1}$ is safe on $\mathcal E_{R,C,T}(\delta)$.

The same analysis applies to each auxiliary policy
$\nu_i^{t+1}$ for $i\in[m]$, with
$\widetilde\pi^{t+1}$ and $\psi^t$ replaced by
$\widetilde\nu_i^{t+1}$ and $\psi_i^t$, respectively.
The initial auxiliary policies $\nu_i^1$ are safe by the argument used
for $\pi^1$.
Therefore, every policy executed by Algorithm~\ref{alg:CSTRL} is safe.

Finally, consider the averaged occupancy measure returned by the algorithm, $\bar\theta=\frac1T\sum_{t=1}^T\theta^{P,\pi^t}$.
Since \(\theta^{P,\pi^t}\in\Theta_{\mathrm{safe}}\) for every \(t\in[T]\) and \(\Theta_{\mathrm{safe}}\) is convex, \(\bar\theta\) also belongs to \(\Theta_{\mathrm{safe}}\).
Let \(\widehat\pi=\pi^{\bar\theta}\) be the Markov policy induced by \(\bar\theta\), so that \(\theta^{P,\widehat\pi}=\bar\theta\).
Therefore, we have \(\mathrm{Violation}(\widehat\pi)=0\).
This completes the proof.
\end{proof}

\section{Proof of Theorem~\ref{thm:pfcstrl_subopt}}
\label{app:pfcstrl-subopt}
The exploration stage of Algorithm~\ref{alg:PFCSTRL} executes one policy per episode and collects one trajectory $\{(s_h^{t},a_h^{t},s_{h+1}^{t})\}_{h=1}^H$ in each episode $t$.
The counts and estimators in~\eqref{eq:rollout-counts}--\eqref{eq:P} are therefore constructed from these exploration trajectories.
The confidence widths in~\eqref{eq:reward_confidence}, \eqref{eq:cost_confidence}, and~\eqref{eq:transition_confidence} remain unchanged, and the transition bonus in~\eqref{eq:transition-width-bonus} uses the same counts.
Lemmas~\ref{lem:reward_concentration}, \ref{lem:gen_confidence}, and~\ref{lem:confidenceset} apply directly to this trajectory sequence.
Retaining these confidence widths affects the resulting bounds only through logarithmic factors.
In the planning stage of Algorithm~\ref{alg:pfcstrl-plan-rf}, the episode superscript is omitted and $(\widehat r,\widehat C,\mathcal P,\Xi,\sigma,\zeta)$ denote these quantities evaluated at the final counts $N^T$ of the exploration dataset $\mathcal D$.

\subsection{Lemmas}
Throughout this subsection, we use the fixed quantities described above. 
For the final exploration dataset, we define the confidence event
$\mathcal E_{\mathrm{fix}}(\delta)
:=
\{|\widehat r_{i,h}(s,a)-r_{i,h}(s,a)|
\le
\sigma_h(s,a),
\,
\forall i\in[m], h\in[H],
(s,a)\in\mathcal S\times\mathcal A,
|\widehat C-C|
\le
\Xi,
\text{and } P\in\mathcal P\}.$
By Lemmas~\ref{lem:reward_concentration}, \ref{lem:gen_confidence}, and~\ref{lem:confidenceset}, applied to this dataset with \(\delta\) replaced by \(\delta/3\), each of the three requirements fails with probability at most \(\delta/3\).
Therefore, \(\mathcal E_{\mathrm{fix}}(\delta)\) holds with probability at least \(1-\delta\).
The analysis in this appendix also uses the online cost and transition events
\(\mathcal E_C(\delta)\) and \(\mathcal E_T(\delta)\) of Appendix~\ref{app:cstrl-subopt},
instantiated for the single-rollout estimates of Algorithm~\ref{alg:PFCSTRL},
together with two martingale events.
The first is the event in the proof of Lemma~\ref{lem:weighted-transition-bridge}. 
Its martingale differences depend only on the inverse square-root visit counts, and the same bound holds for the single exploration rollout.
The second follows by repeating the martingale argument in the proof of Lemma~\ref{lem:main-occ-regret} with the exploration reward \(\overline r_h^t\) in place of \(\sum_{i=1}^m w_i^t\lambda_i\sigma_h^t\).
Since \(\overline r^t\) is determined before episode \(t+1\) and satisfies \(0\le\overline r_h^t(s,a)\le1\), the resulting sequence remains a martingale difference sequence with increments bounded by \(H\).
The same Azuma-Hoeffding bound therefore applies.
Throughout this appendix, we work on the intersection of \(\mathcal E_C(\delta/5)\), \(\mathcal E_T(\delta/5)\), \(\mathcal E_{\mathrm{fix}}(\delta/5)\), and the two martingale events above, each taken with failure probability \(\delta/5\).
By a union bound, this intersection holds with probability at least \(1-\delta\), and this allocation affects only logarithmic factors, which are already absorbed into the \(\widetilde{\mathcal O}(\cdot)\) notation.
We refer to this intersection as the joint event.
The planning-stage argument in the proof of Theorem~\ref{thm:safety-pfcstrl} uses only the fixed cost and transition confidence bounds.
Therefore, on the joint event, every policy \(\pi^{k+1}\) is safe, and Lemma~\ref{lem:pf-exploration-coverage} applies to each \(\pi^{k+1}\).

On the joint event, the transition estimate enters the analysis only through the following comparison.
Let \(\pi\) be any policy and let \(P'\in\mathcal P\).
Since \(P'\in\mathcal P\) and \(P\in\mathcal P\), both kernels satisfy the transition confidence bounds in~\eqref{eq:transition_confidence}.
Therefore, for every \(h\in[H]\) and
\((s,a)\in\mathcal S\times\mathcal A\), we have
$\bigl\|P_h'(\cdot|s,a)-P_h(\cdot|s,a)\bigr\|_1
\le
\min\{2,\,2|\mathcal S|\epsilon_h(s,a)\}.$
The bound \(2\) is valid since both kernels are probability distributions.
Comparing this quantity with the transition bonus in~\eqref{eq:transition-width-bonus} and using
\(\ln(2mTH|\mathcal S|^2|\mathcal A|/\delta)
\le
2\ln(2mTH|\mathcal S||\mathcal A|/\delta)\),
we obtain
$H\min\{2,\,2|\mathcal S|\epsilon_h(s,a)\}
\le
2\sqrt{|\mathcal S|}\,\zeta_h(s,a)$.
Therefore, Lemma~\ref{lem:simulation} yields
\begin{align}
\sum_{h=1}^H
\bigl\|
\theta_h^{P',\pi}-\theta_h^{P,\pi}
\bigr\|_1
\le
2\sqrt{|\mathcal S|}
\bigl\langle
\zeta,\theta^{P,\pi}
\bigr\rangle .
\label{eq:pf-transition-transfer}
\end{align}

\begin{lemma}
\label{lem:pf-exploration-coverage}
On the joint event, with \(\eta=T^{-1/2}\), for every policy \(\pi\) with \(\theta^{P,\pi}\in\Theta_{\mathrm{safe}}\), we have
\begin{align*}
\left\langle \sigma,\theta^{P,\pi}\right\rangle
\le
\widetilde{\mathcal O}\!\left(
\sqrt{\phi_{\mathrm{safe}}}\,H^2|\mathcal S|^{3/2}\sqrt{\frac{|\mathcal A|}{T}}
\right),
\quad
\left\langle \zeta,\theta^{P,\pi}\right\rangle
\le
\widetilde{\mathcal O}\!\left(
\sqrt{\phi_{\mathrm{safe}}}\,H^3|\mathcal S|^{3/2}\sqrt{\frac{|\mathcal A|}{T}}
\right).
\end{align*}
\end{lemma}

\begin{proof}
We work on the joint event defined at the beginning of this appendix and fix any policy \(\pi\) with \(\theta^{P,\pi}\in\Theta_{\mathrm{safe}}\).
When \(T=1\), both bounds are immediate, since \(\sigma_h\le1\), \(\zeta_h\le H\), and \(\theta^{P,\pi}\) has total mass \(H\).
We therefore assume \(T\ge2\), so that \(1/(T-1)\le2/T\) and the bounds of
Lemmas~\ref{lem:weighted-transition-bridge} and~\ref{lem:main-occ-regret}, which are normalized by \(1/T\), can be applied directly.

As described at the beginning of this appendix, 
the fixed bonuses are the expressions in~\eqref{eq:reward_confidence} and~\eqref{eq:transition-width-bonus} evaluated at the final counts \(N_h^T\), 
whereas the online bonuses are the same expressions evaluated at \(N_h^t\).
Since these expressions are nonincreasing in the count and \(N_h^T(s,a)\ge N_h^t(s,a)\) for every \(t\le T\), 
we have
\(\sigma_h(s,a)\le\sigma_h^t(s,a)\) and
\(\zeta_h(s,a)\le\zeta_h^t(s,a)\) for all \(h\in[H]\) and \((s,a)\in\mathcal S\times\mathcal A\).
Moreover, Line~\ref{line:ex_r} of Algorithm~\ref{alg:PFCSTRL} gives
$\sigma_h^t(s,a)\le\overline r_h^t(s,a)$
and
${\zeta_h^t(s,a)}/{H}\le\overline r_h^t(s,a)$
for all \(h,s,a,t\).
It therefore suffices to bound
$\frac{1}{T-1}\sum_{t=1}^{T-1}\langle\overline r^t,\theta^{P,\pi}\rangle$,
since the bound for \(\zeta\) follows after multiplication by \(H\).

We first compare \(\theta^{P,\pi}\) with the projected occupancy measures.
On the joint event, we have \(P\in\mathcal P^t\) and \(\widehat C^t-\Xi^t\le C\). 
Since \(C^\top\theta^{P,\pi}\le\boldsymbol\tau\), we have
$(\widehat C^t-\Xi^t)^\top\theta^{P,\pi}
\le
C^\top\theta^{P,\pi}
\le
\boldsymbol\tau.$
Thus, \(\theta^{P,\pi}\) is feasible for the projection.
In particular, the projection is feasible, and \(\psi^t\) is computed by~\eqref{eq:eta_t}.
Applying~\eqref{eq:main-omd-one-step} to the exponential update in Line~\ref{alg3:theta} of Algorithm~\ref{alg:PFCSTRL} and summing over \(t=1,\ldots,T-1\), we obtain
\begin{align}
\frac{1}{T-1}\sum_{t=1}^{T-1}
\bigl\langle\overline r^t,\theta^{P,\pi}-\widetilde\theta^{t+1}\bigr\rangle
\le
\frac{D_{\mathrm{KL}}(\theta^{P,\pi}\Vert\widetilde\theta^1)}{\eta(T-1)}
\le
\widetilde{\mathcal O}\!\left(\frac{H}{\sqrt T}\right),
\label{eq:pf-coverage-regret}
\end{align}
where the last inequality uses
$\widetilde\theta_h^1(s,a,s')=1/(|\mathcal S|^2|\mathcal A|)$ and \(\eta=T^{-1/2}\).

It remains to bound the average exploration reward under the projected occupancy measures.
If \(\psi^t>0\), let \(j\in[n]\) attain the maximum in the definition of \(\psi^t\).
Then we have
\begin{align*}
1-\psi^t
&=
\frac{\tau_j-\gamma_j}{
\min\left\{
\max_{P'\in\mathcal P^t}
\sum_{h=1}^H
(\widehat c_{j,h}^t+\xi_{j,h}^t)^\top
\theta_h^{P',\widetilde\pi^{t+1}},
d
\right\}
-\gamma_j
}
\ge
\frac{\tau_j-\gamma_j}{d}.
\end{align*}
If \(\psi^t=0\), we directly have \(1-\psi^t=1\ge(\tau_j-\gamma_j)/d\) for every \(j\in[n]\), since \(\tau_j-\gamma_j\le d\).
Thus, in both cases, we have
$\frac{1}{1-\psi^t}
\le
\max_{j\in[n]}
\frac{d}{\tau_j-\gamma_j}
=
\sqrt{\phi_{\mathrm{safe}}},
\,
\forall t\in[T-1].$
By construction, we have
$\widetilde\theta^{t+1}=\theta^{P^{t+1},\widetilde\pi^{t+1}}$
for some \(P^{t+1}\in\mathcal P^t\), 
and \(0\le\overline r_h^t(s,a)\le1\).
Using
$(1-\psi^t)\theta_h^{P,\widetilde\pi^{t+1}}(s,a)\le\theta_h^{P,\pi^{t+1}}(s,a)$ together with \(\overline r^t\ge0\), we obtain
$(1-\psi^t)
\langle\overline r^t,\widetilde\theta^{t+1}\rangle
\le
\langle\overline r^t,\theta^{P,\pi^{t+1}}\rangle
+
(1-\psi^t)
\|
\theta^{P^{t+1},\widetilde\pi^{t+1}}-\theta^{P,\widetilde\pi^{t+1}}\|_1 .$
Dividing by \(1-\psi^t\) and using \(\frac{1}{1-\psi^t}\le\sqrt{\phi_{\mathrm{safe}}}\), we have \(1\le\sqrt{\phi_{\mathrm{safe}}}(1-\psi^t)\) for the second term. 
Averaging over \(t=1,\ldots,T-1\), we obtain
\begin{align}
&\frac{1}{T-1}\sum_{t=1}^{T-1}
\bigl\langle\overline r^t,\widetilde\theta^{t+1}\bigr\rangle \nonumber \\
&\quad \le
\sqrt{\phi_{\mathrm{safe}}}
\left[
\frac{1}{T-1}\sum_{t=1}^{T-1}
\bigl\langle\overline r^t,\theta^{P,\pi^{t+1}}\bigr\rangle
+
\frac{1}{T-1}\sum_{t=1}^{T-1}(1-\psi^t)
\bigl\|
\theta^{P^{t+1},\widetilde\pi^{t+1}}-\theta^{P,\widetilde\pi^{t+1}}
\bigr\|_1
\right].
\label{eq:pf-coverage-mixing}
\end{align}
The second sum in~\eqref{eq:pf-coverage-mixing} is controlled by the argument in the proof of Lemma~\ref{lem:weighted-transition-bridge}, applied to the exploration trajectories of Algorithm~\ref{alg:PFCSTRL}, and is at most $\widetilde{\mathcal O}(H^2|\mathcal S|^{3/2}\sqrt{|\mathcal A|/T})$.
For the first sum, the exploration stage executes \(\pi^{t+1}\) in episode \(t+1\).
Since both bonuses are nonnegative, Line~\ref{line:ex_r} of Algorithm~\ref{alg:PFCSTRL} gives
$\overline r_h^t(s,a)
\le
\sigma_h^t(s,a)+\zeta_h^t(s,a)/H$.
The martingale event for \(\overline r^t\) and the two bounds in Lemma~\ref{lem:realized-bonus-sum}, applied to the same exploration trajectories, therefore give
$\frac{1}{T-1}\sum_{t=1}^{T-1}
\langle\overline r^t,\theta^{P,\pi^{t+1}}\rangle
\le
\widetilde{\mathcal O}(H|\mathcal S|\sqrt{|\mathcal A|/T})$.

Combining these two bounds with~\eqref{eq:pf-coverage-regret} and~\eqref{eq:pf-coverage-mixing} proves the bound for \(\sigma\), 
and multiplying by \(H\) proves the bound for \(\zeta\).
\end{proof}

\begin{lemma}
\label{lem:pf-fixed-aux-benchmark}
For each $i\in[m]$, let $\widetilde\theta_i$ and $\widetilde V_i
=
\langle\bar r_i,\widetilde\theta_i\rangle$
be computed in Lines~\ref{line:aux-proj-start}--\ref{line:aux-proj-end} of Algorithm~\ref{alg:pfcstrl-plan-rf}.
On the joint event, with \(\eta=K^{-1/2}\), we have
$V_i^*-\widetilde V_i
\le
\widetilde{\mathcal O}(H/\sqrt K)$
for every $i\in[m]$.
\end{lemma}

\begin{proof}
Consider any $i\in[m]$, and let
$\theta_i^*
\in
\argmax_{\theta\in\Theta_{\mathrm{safe}}}
V_i(\theta).$
Since \(\theta_i^*\in\Theta_{\mathrm{safe}}\) and \(P\in\mathcal P\), we have \(\theta_i^*\in\Theta(\mathcal P)\). 
The cost confidence bound \(\widehat C-\Xi\le C\) also gives
$(\widehat C-\Xi)^\top\theta_i^*
\le
C^\top\theta_i^*
\le
\boldsymbol\tau.$
Thus, on the joint event, the projection \(\texttt{Proj}(\widehat\theta_i,\mathcal P,\widehat C,\Xi)\) is feasible.
By the optimality condition of the KL projection with
$\widehat\theta_{i,h}(s,a,s')
=
\widetilde\theta_h^1(s,a,s')
\exp\{K^{1/2}\bar r_{i,h}(s,a)\},$
we have
$\langle
\bar r_i,\theta_i^*-\widetilde\theta_i
\rangle
\le
\frac{
D_{\mathrm{KL}}(\theta_i^*\Vert\widetilde\theta^1)
-
D_{\mathrm{KL}}(\widetilde\theta_i\Vert\widetilde\theta^1)}{\sqrt K}.$
Since \(\widetilde\theta^1\) is uniform,
we obtain
$D_{\mathrm{KL}}(\theta_i^*\Vert\widetilde\theta^1)
\le
H\log(|\mathcal S|^2|\mathcal A|).$
Therefore, we have
$\langle
\bar r_i,\theta_i^*-\widetilde\theta_i
\rangle
\le
\widetilde{\mathcal O}(H/\sqrt K).$
Using \(\bar r_i\ge r_i\), we obtain
$V_i^*-\widetilde{V}_i
=
\langle r_i,\theta_i^*\rangle
-
\langle \bar r_i,\widetilde\theta_i\rangle
\le
\langle
\bar r_i,\theta_i^*-\widetilde\theta_i
\rangle
\le
\widetilde{\mathcal O}(H/\sqrt K).$
This completes the proof.
\end{proof}

\begin{lemma}
\label{lem:pf-planning-occupancy}
Let \(\{\widetilde\theta^{k+1}\}_{k=1}^K\) be generated by the planning-stage occupancy measure update in Algorithm~\ref{alg:pfcstrl-plan-rf}. 
For any safe occupancy measure \(\theta\in\Theta_{\mathrm{safe}}\), on the joint event, with \(\eta=K^{-1/2}\), we have
\begin{align*}
\frac1K\sum_{k=1}^K
\left\langle
\sum_{i=1}^m w_i^k\lambda_i r_i,\,
\theta-\widetilde\theta^{k+1}
\right\rangle
\le
\widetilde{\mathcal O}\!\left(
H/\sqrt{K}
+
\phi_{\mathrm{safe}}\,H^3|\mathcal S|^{2}\sqrt{|\mathcal A|/T}
\right).
\end{align*}
\end{lemma}

\begin{proof}
For each planning iteration $k$, we define the true scalarized reward $R_h^k(s,a):=\sum_{i=1}^m w_i^k\lambda_i r_{i,h}(s,a)$.
Since \(\theta\in\Theta_{\mathrm{safe}}\), \(P\in\mathcal P\), and \(\widehat C-\Xi\le C\), we have
$(\widehat C-\Xi)^\top\theta
\le
C^\top\theta
\le
\boldsymbol\tau.$
Thus, \(\theta\) is feasible for the projection.
Applying the KL projection argument in the proof of Lemma~\ref{lem:main-occ-regret} to the $K$ planning updates and summing over $k$ gives
$\frac1K\sum_{k=1}^K
\langle\bar R^k,\theta-\widetilde\theta^{k+1}\rangle
\le
\frac{H\log(|\mathcal S|^2|\mathcal A|)}{\sqrt K}
=
\widetilde{\mathcal O}(H/\sqrt K)$.

We next express the preceding inequality using the true scalarized reward \(R^k\).
On the event \(\mathcal E_{\mathrm{fix}}(\delta)\),
\(\sigma\) is the reward bonus computed from the final exploration data.
Thus, by Lemma~\ref{lem:reward_concentration} and the definition of
\(\bar r_i\),
we obtain
$0\le
\bar r_{i,h}(s,a)-r_{i,h}(s,a)
\le \allowbreak
2\sigma_h(s,a),
\,\allowbreak \forall i,h,s,a.$
Using this bound, we have
$0\le
\bar R_h^k(s,a)-R_h^k(s,a)
\le
2\sum_{i=1}^m w_i^k\lambda_i\sigma_h(s,a).$
Hence, we obtain
$\langle
R^k,\theta-\widetilde\theta^{k+1}
\rangle
\le
\langle
\bar R^k,\theta-\widetilde\theta^{k+1}
\rangle
+
2\langle
\sum_{i=1}^m w_i^k\lambda_i\sigma,\,
\widetilde\theta^{k+1}\rangle.$
It remains to control the last term.
Since \(\boldsymbol w^k,\boldsymbol\lambda\in\Delta_m\), we have $\sum_{i=1}^m w_i^k\lambda_i\le1$ and therefore
$\langle
\sum_{i=1}^m w_i^k\lambda_i\sigma,\,
\widetilde\theta^{k+1}
\rangle
\le
\langle
\sigma,\widetilde\theta^{k+1}
\rangle .$
For each \(k\in[K]\), we have \(\widetilde\theta^{k+1} =\theta^{P^{k+1},\widetilde\pi^{k+1}}\) for some \(P^{k+1}\in\mathcal P\).
Since \(0\le\sigma_h(s,a)\le1\), inequality~\eqref{eq:pf-transition-transfer} gives
$\langle
\sigma,\widetilde\theta^{k+1}
\rangle
\le
\langle
\sigma,\theta^{P,\widetilde\pi^{k+1}}
\rangle
+
2\sqrt{|\mathcal S|}
\langle
\zeta,\theta^{P,\widetilde\pi^{k+1}}
\rangle .$
By the construction of \(\pi^{k+1}\), we have
$\theta_h^{P,\widetilde\pi^{k+1}}(s,a)
\le
\frac{1}{1-\psi^k}\theta_h^{P,\pi^{k+1}}(s,a)$ for all \(h,s,a\).
The argument in the proof of Lemma~\ref{lem:pf-exploration-coverage} gives \(1/(1-\psi^k)\le\sqrt{\phi_{\mathrm{safe}}}\).
Since \(\sigma\) and \(\zeta\) are nonnegative, we therefore obtain
$\langle
\sigma,\widetilde\theta^{k+1}
\rangle
\le
\sqrt{\phi_{\mathrm{safe}}}
[\langle
\sigma,\theta^{P,\pi^{k+1}}
\rangle
+
2\sqrt{|\mathcal S|}
\langle
\zeta,\theta^{P,\pi^{k+1}}
\rangle].$
Averaging over \(k\in [K]\) and applying Lemma~\ref{lem:pf-exploration-coverage} gives
$\frac1K\sum_{k=1}^K
\langle
\sum_{i=1}^m w_i^k\lambda_i\sigma,\,
\widetilde\theta^{k+1}
\rangle
\le
\widetilde{\mathcal O}\!(
\phi_{\mathrm{safe}}H^3|\mathcal S|^2
\sqrt{\frac{|\mathcal A|}{T}}
).$
Combining this bound with the preceding KL projection bound proves the lemma.
\end{proof}

\begin{lemma}
\label{lem:pf-planning-mixing}
Let \(\{\psi^k\}_{k=1}^K\) be the probabilities of selecting \(\pi^\dagger\) during the planning stage of Algorithm~\ref{alg:pfcstrl-plan-rf}.
On the joint event, we have
\begin{align*}
\frac1K\sum_{k=1}^K\psi^k
\le
\widetilde{\mathcal O}\!\left(
\phi_{\mathrm{safe}}H^3|\mathcal S|^{2}
\sqrt{|\mathcal A|/T}
\right).
\end{align*}
\end{lemma}

\begin{proof}
On the joint event, we have \(P\in\mathcal P\) and \(\widehat C-\Xi\le C\).
Therefore, we have
$(\widehat C-\Xi)^\top
\theta^{P,\pi^\dagger}
\le
C^\top\theta^{P,\pi^\dagger}
=
\boldsymbol\gamma
<
\boldsymbol\tau.$
The occupancy measure \(\theta^{P,\pi^\dagger}\) is thus feasible for the fixed planning projection, and the projection is feasible at every planning iteration.
Fix \(k\in[K]\).
Since $0<\tau_j-\gamma_j\le d$ and $d\ge1$, we have $1/(\tau_j-\gamma_j)\le\sqrt{\phi_{\mathrm{safe}}}$ for every $j\in[n]$.
Lemma~\ref{lem:mixing-bound} with \(\widetilde\theta=\widetilde\theta^{k+1}\) therefore gives
\[
\psi^k
\le
\sqrt{\phi_{\mathrm{safe}}}\,(1-\psi^k)
\biggl[
3\max_{P'\in\mathcal P}
\bigl\|
\theta^{P',\widetilde\pi^{k+1}}
-
\theta^{P,\widetilde\pi^{k+1}}
\bigr\|_1
+
2\max_{j\in[n]}
\sum_{h=1}^H
\xi_{j,h}^\top
\theta_h^{P,\widetilde\pi^{k+1}}
\biggr].
\]
By~\eqref{eq:pf-transition-transfer}, the transition term is at most \(6\sqrt{|\mathcal S|}\langle\zeta,\theta^{P,\widetilde\pi^{k+1}}\rangle\).
Since the cost bonus in~\eqref{eq:cost_confidence} and the transition bonus in~\eqref{eq:transition-width-bonus} differ only in logarithmic factors and in the factor \(H\sqrt{|\mathcal S|}\), we have \(\xi_{j,h}(s,a)\le\widetilde{\mathcal O}(\zeta_h(s,a)/H)\).
The cost bonus term is therefore at most \(\widetilde{\mathcal O}(\langle\zeta,\theta^{P,\widetilde\pi^{k+1}}\rangle/H)\).
Since $\theta^{P,\pi^{k+1}}=(1-\psi^k)\theta^{P,\widetilde\pi^{k+1}}+\psi^k\theta^{P,\pi^\dagger}$, we have $(1-\psi^k)\theta_h^{P,\widetilde\pi^{k+1}}(s,a)\le\theta_h^{P,\pi^{k+1}}(s,a)$ for all \(h,s,a\).
Combining these bounds, we obtain
\begin{equation}
\psi^k
\le
\widetilde{\mathcal O}\!\left(
\sqrt{\phi_{\mathrm{safe}}|\mathcal S|}\,
\bigl\langle\zeta,\theta^{P,\pi^{k+1}}\bigr\rangle
\right).
\label{eq:pf-planning-mixing-one-step}
\end{equation}
Averaging~\eqref{eq:pf-planning-mixing-one-step} over \(k=1,\ldots,K\) and applying Lemma~\ref{lem:pf-exploration-coverage} to each safe policy \(\pi^{k+1}\) proves the lemma.
\end{proof}

\begin{lemma}
\label{lem:pf-fixed-aux-excess}
Let \(\widetilde\theta_i\) and
$\widetilde{V}_i
=
\langle \bar r_i,\widetilde\theta_i\rangle$
be computed in Lines~\ref{line:aux-proj-start}--\ref{line:aux-proj-end} of Algorithm~\ref{alg:pfcstrl-plan-rf}. 
For any safe single-objective optimal occupancy measure \(\theta_i^*\in\Theta_{\mathrm{safe}}\), on the joint event, we have
\[
-\widetilde{\mathcal O}(H/\sqrt K)
\le
\widetilde{V}_i
-
V_i^*
\le
\widetilde{\mathcal O}\!\left(
\phi_{\mathrm{safe}}H^4|\mathcal S|^{2}
\sqrt{|\mathcal A|/T}
\right).
\]
\end{lemma}

\begin{proof}
Fix \(i\in[m]\). 
The lower bound
$\widetilde{V}_i
-
V_i^*
\ge
-\widetilde{\mathcal O}(H/\sqrt K)$
follows from Lemma~\ref{lem:pf-fixed-aux-benchmark}.
It remains to prove the upper bound.

Let \(\widetilde\nu_i=\pi^{\widetilde\theta_i}\) be the policy induced by the projected occupancy measure \(\widetilde\theta_i\).
We define \(\psi_i\) by~\eqref{eq:eta_t} with \((\mathcal P^t,\widehat C^t,\Xi^t,\widetilde\pi^{t+1})\) replaced by \((\mathcal P,\widehat C,\Xi,\widetilde\nu_i)\).
Let \(\nu_i\) equal \(\pi^\dagger\) with probability \(\psi_i\) and \(\widetilde\nu_i\) with probability \(1-\psi_i\).
Since \(\theta_i^*\) is feasible for the auxiliary projection on the joint event, the projection is feasible.
The argument used in the proof of Lemma~\ref{lem:pf-exploration-coverage} then gives $\frac{1}{1-\psi_i}\le\sqrt{\phi_{\mathrm{safe}}}$.
The argument that considers the three cases separately in the proof of Theorem~\ref{thm:safety}, applied to \((\widetilde\nu_i,\nu_i,\psi_i)\), shows that \(\nu_i\) is safe.

Repeating the derivation of~\eqref{eq:pf-planning-mixing-one-step} for \((\widetilde\theta_i,\widetilde\nu_i,\nu_i,\psi_i)\) gives
$\psi_i
\le
\widetilde{\mathcal O}(
\sqrt{\phi_{\mathrm{safe}}|\mathcal S|}\,
\langle\zeta,\theta^{P,\nu_i}\rangle)$.
Applying Lemma~\ref{lem:pf-exploration-coverage} to the safe policy \(\nu_i\), we obtain
$\psi_i
\le
\widetilde{\mathcal O}\!(
\phi_{\mathrm{safe}}H^3|\mathcal S|^{2}
\sqrt{|\mathcal A|/T}).$

Since \(\theta_i^*\) is a safe single-objective optimal occupancy measure, we have
$V_i(\theta^{P,\nu_i})
\le
V_i^*$.
Moreover, the definition of \(\nu_i\) gives
$\theta^{P,\nu_i}
=
(1-\psi_i)\theta^{P,\widetilde\nu_i}
+
\psi_i\theta^{P,\pi^\dagger}.$
Since the cumulative reward over one episode is bounded by \(H\), this implies
\[V_i(\theta^{P,\widetilde\nu_i})
\le
V_i(\theta^{P,\nu_i})
+
H\psi_i
\le
V_i^*
+
H\psi_i.\]

We now compare the fixed optimistic single-objective value \(\widetilde V_i\) with the true value of \(\widetilde\nu_i\).
Using \(0\le \bar r_i-r_i\le 2\sigma\), we have
$\langle \bar r_i,\widetilde\theta_i\rangle
\le
\langle r_i,\widetilde\theta_i\rangle
+
2\langle \sigma,\widetilde\theta_i\rangle$.
Since \(\widetilde\theta_i\in\Theta(\mathcal P)\) and \(r_{i,h}\le1\), inequality~\eqref{eq:pf-transition-transfer} applied to \(\widetilde\nu_i\) gives $\langle r_i,\widetilde\theta_i\rangle
\le
V_i(\theta^{P,\widetilde\nu_i})
+
2\sqrt{|\mathcal S|}\langle \zeta,\theta^{P,\widetilde\nu_i}\rangle.$
Then we have
\[\widetilde{V}_i
\le V_i^*
+ H\psi_i
+ 2\langle \sigma,\widetilde\theta_i\rangle
+ 2\sqrt{|\mathcal S|}\langle \zeta,\theta^{P,\widetilde\nu_i}\rangle.\]

Inequality~\eqref{eq:pf-transition-transfer} bounds the difference between \(\widetilde\theta_i\) and the occupancy measure \(\theta^{P,\widetilde\nu_i}\) of \(\widetilde\nu_i\) under \(P\).
By the construction of $\nu_i$, we have $\theta_h^{P,\widetilde\nu_i}(s,a)\le\frac{1}{1-\psi_i}\theta_h^{P,\nu_i}(s,a)$ for all \(h,s,a\), and $\frac{1}{1-\psi_i}\le\sqrt{\phi_{\mathrm{safe}}}$ by the argument used in the proof of Lemma~\ref{lem:pf-exploration-coverage}.
The resulting \(\sigma\)- and \(\zeta\)-weighted occupancy measures are therefore bounded by \(\sqrt{\phi_{\mathrm{safe}}}\) times the corresponding terms under \(\nu_i\), which are controlled by Lemma~\ref{lem:pf-exploration-coverage} applied to the safe policy \(\nu_i\).
Thus, we obtain
\[2\langle \sigma,\widetilde\theta_i\rangle
+
2\sqrt{|\mathcal S|}\langle \zeta,\theta^{P,\widetilde\nu_i}\rangle
\le
\widetilde{\mathcal O}\!(
\phi_{\mathrm{safe}}\,H^3|\mathcal S|^{2}\sqrt{|\mathcal A|/T}).\]
Together with the bound on \(\psi_i\), we obtain
$\widetilde{V}_i
-
V_i^*
\le
\widetilde{\mathcal O}\!(
\phi_{\mathrm{safe}}H^4|\mathcal S|^{2}
\sqrt{|\mathcal A|/T}).$
This completes the proof.
\end{proof}

\subsection{Proof of Theorem~\ref{thm:pfcstrl_subopt}}
\begin{proof}
We work on the joint event defined at the beginning of this appendix, which holds with probability at least \(1-\delta\).
The bounds of Lemmas~\ref{lem:pf-exploration-coverage}--\ref{lem:pf-fixed-aux-excess} hold simultaneously on this event.
Let $\bar\theta$ be the averaged occupancy measure returned by Algorithm~\ref{alg:pfcstrl-plan-rf}.
We use the optimal occupancy measures $\theta_i^*$, $i\in[m]$, and $\theta_{\boldsymbol\lambda}^*$ introduced in Section~\ref{sec:sketch-subopt}.
\begin{align}
\mathrm{SubOpt}(\widehat\pi)
&=
\stchlm(\pi^{\bar\theta})-\stchlm(\pi^{\theta_{\boldsymbol\lambda}^*})\nonumber \\
&=
\max_{\boldsymbol w\in\Delta_m}
\sum_{i=1}^m
\left[
w_i\lambda_i
\left(
V_i^*+\iota-V_i(\bar \theta)
\right)
-\mu w_i\log w_i
\right] \nonumber\\
& \quad -\max_{\boldsymbol w\in\Delta_m}
\sum_{i=1}^m
\left[
w_i\lambda_i
\left(
V_i^*+\iota-V_i(\theta^*_{\boldsymbol{\lambda}})
\right)
-\mu w_i\log w_i
\right] \nonumber\\
&\le
\max_{\boldsymbol w\in\Delta_m}
\sum_{i=1}^m
\left[
w_i\lambda_i
\left(
\widetilde{V}_i+\iota-V_i(\bar \theta)
\right)
-\mu w_i\log w_i
\right] \nonumber\\
& \quad -\max_{\boldsymbol w\in\Delta_m}
\sum_{i=1}^m
\left[
w_i\lambda_i
\left(
\widetilde{V}_i+\iota-V_i(\theta^*_{\boldsymbol{\lambda}})
\right)
-\mu w_i\log w_i
\right]\nonumber \\
& \quad +
\widetilde{\mathcal O}\!\left(
\phi_{\mathrm{safe}}H^4|\mathcal S|^{2}
\sqrt{|\mathcal A|/T}
+
H/\sqrt K
\right).
\label{eq:pf-reduce-optimistic}
\end{align}
The first equality holds since \(\widehat\pi=\pi^{\bar\theta}\) and \(\pi^{\theta_{\boldsymbol\lambda}^*}\) is an optimal safe policy.
The second equality follows from~\eqref{eq:eq-stch}, since the inner maximum over \(\{\theta_i\in\Theta_{\mathrm{safe}}\}_{i=1}^m\) is attained at \(\theta_i=\theta_i^*\).
For every \(\boldsymbol w\in\Delta_m\), replacing \(V_i^*\) by \(\widetilde{V}_i\) changes the objective inside each maximum by
$\sum_{i=1}^m w_i\lambda_i(\widetilde{V}_i-V_i^*)$,
whose absolute value is at most \(\max_{i\in[m]}|\widetilde{V}_i-V_i^*|\) since \(\lambda_i\le1\) and \(\sum_{i=1}^m w_i=1\).
Thus, each of the two maxima changes by at most this amount, and the inequality follows from Lemma~\ref{lem:pf-fixed-aux-excess}.

We now bound the remaining objective difference on the right-hand side of~\eqref{eq:pf-reduce-optimistic}.
By the definition of \(\bar\theta\), we have
$V_i(\bar\theta)
=
\frac1K\sum_{k=1}^K
V_i(\theta^{P,\pi^{k+1}})$.
Therefore, we obtain
\begin{align}
&\max_{\boldsymbol w\in\Delta_m}
\sum_{i=1}^m
\left[
w_i\lambda_i
\left(
\widetilde{V}_i+\iota-V_i(\bar\theta)
\right)
-\mu w_i\log w_i
\right]  \nonumber \\
& \quad-
\max_{\boldsymbol w\in\Delta_m}
\sum_{i=1}^m
\left[
w_i\lambda_i
\left(
\widetilde{V}_i+\iota-V_i(\theta_{\boldsymbol{\lambda}}^*)
\right)
-\mu w_i\log w_i
\right] \nonumber  \\
&\le
\underbrace{\frac1K\sum_{k=1}^K\sum_{i=1}^m
\lambda_i
\left|
V_i^{k+1}
-
V_i(\theta^{P,\pi^{k+1}})
\right|}_{\mathrm{(I)}} \nonumber\\
&\quad +
\underbrace{
\max_{\boldsymbol w\in\Delta_m}
\frac1K\sum_{k=1}^K
\sum_{i=1}^m
\bigg[
(w_i-w_i^k)\lambda_i
\left(
\widetilde{V}_i+\iota
-V_i^{k+1}
\right) 
-\mu w_i\log w_i
+\mu w_i^k\log w_i^k
\bigg]
}_{\mathrm{(II)}} \nonumber \\
& \quad +\underbrace{\frac1K\sum_{k=1}^K
\sum_{i=1}^m
w_i^k\lambda_i
\left(
V_i(\theta_{\boldsymbol{\lambda}}^*)
-
V_i^{k+1}
\right)}_{\mathrm{(III)}}.
\label{eq:pf-main-decomp}
\end{align}

We first bound Term \(\mathrm{(I)}\). 
For each \(k\), by construction, we have
$V_i^{k+1}
=
\langle \bar r_i,\widetilde\theta^{k+1}\rangle$.
Using \(0\le \bar r_i-r_i\le 2\sigma\),
inequality~\eqref{eq:pf-transition-transfer}, and
$\theta^{P,\pi^{k+1}}
=
(1-\psi^k)\theta^{P,\widetilde\pi^{k+1}}
+
\psi^k\theta^{P,\pi^\dagger},$
we have
\begin{align}
\left|
V_i^{k+1}
-
V_i(\theta^{P,\pi^{k+1}})
\right|
&\le
2\langle \sigma,\widetilde\theta^{k+1}\rangle
+
2\sqrt{|\mathcal S|}\langle \zeta,\theta^{P,\widetilde\pi^{k+1}}\rangle
+
H\psi^k .
\label{eq:pf-main-est-one}
\end{align}
Averaging~\eqref{eq:pf-main-est-one} over \(k=1,\ldots,K\), we use \eqref{eq:pf-transition-transfer} together with
$(1-\psi^k)
\theta_h^{P,\widetilde\pi^{k+1}}(s,a)
\le
\theta_h^{P,\pi^{k+1}}(s,a)$
to bound the first two terms by the corresponding fixed bonus terms evaluated under \(\theta^{P,\pi^{k+1}}\).
This comparison incurs the factor \(\frac{1}{1-\psi^k}\le\sqrt{\phi_{\mathrm{safe}}}\), as established in the proof of Lemma~\ref{lem:pf-exploration-coverage}.
The last term is controlled by Lemma~\ref{lem:pf-planning-mixing} and dominates the first two.
Thus, we have
\begin{align}
\mathrm{(I)}
&\le
\widetilde{\mathcal O}\!\left(
\phi_{\mathrm{safe}}H^4|\mathcal S|^{2}
\sqrt{|\mathcal A|/T}
\right).
\label{eq:pf-term-i}
\end{align}

We next bound Term \(\mathrm{(II)}\). 
The update of \(\boldsymbol w^k\) in the planning stage is the same entropic mirror-ascent update as in Lemma~\ref{lem:oco-stch}.
Its proof applies with the $T-1$ online updates replaced by the $K$ planning
updates and the normalization $1/T$ replaced by $1/K$.
Therefore, we obtain
\begin{align}
\mathrm{(II)}
\le
\widetilde{\mathcal O}\!\left(
\frac{H^2}{\mu K}
+
\frac{\mu}{K}
\right).
\label{eq:pf-term-ii}
\end{align}

It remains to bound Term \(\mathrm{(III)}\). 
Since \(\bar r_i\ge r_i\) and \(\widetilde\theta^{k+1}\ge 0\), 
we have
$V_i(\theta_{\boldsymbol{\lambda}}^*)-V_i^{k+1}
=
\langle r_i,\theta_{\boldsymbol{\lambda}}^*\rangle
-
\langle \bar r_i,\widetilde\theta^{k+1}\rangle
\le
\langle r_i,\theta_{\boldsymbol{\lambda}}^*-\widetilde\theta^{k+1}\rangle$.
Therefore, we have
$\mathrm{(III)}
\le
\frac1K\sum_{k=1}^K
\langle
\sum_{i=1}^m w_i^k\lambda_i r_i,\,
\theta_{\boldsymbol{\lambda}}^*
-
\widetilde\theta^{k+1}
\rangle .$
Since \(\theta_{\boldsymbol{\lambda}}^*\in\Theta_{\mathrm{safe}}\), applying
Lemma~\ref{lem:pf-planning-occupancy} with \(\theta=\theta_{\boldsymbol{\lambda}}^*\) gives
\begin{align}
\mathrm{(III)}
&\le
\widetilde{\mathcal O}\!\left(
{H}/{\sqrt K}
+
\phi_{\mathrm{safe}}\,H^3|\mathcal S|^{2}\sqrt{|\mathcal A|/T}\right).
\label{eq:pf-term-iii}
\end{align}

Combining~\eqref{eq:pf-reduce-optimistic},
\eqref{eq:pf-main-decomp},
\eqref{eq:pf-term-i},
\eqref{eq:pf-term-ii}, and
\eqref{eq:pf-term-iii}, we obtain
\[
\mathrm{SubOpt}(\widehat\pi)
\le
\widetilde{\mathcal O}\!\left(
\phi_{\mathrm{safe}}H^4|\mathcal S|^{2}
\sqrt{\frac{|\mathcal A|}{T}}
+
\frac{H}{\sqrt K}
+
\frac{H^2}{\mu K}
+
\frac{\mu}{K}
\right).
\]
This completes the proof.
\end{proof}

\section{Proof of Theorem~\ref{thm:safety-pfcstrl}}
\label{app:safety-pfcstrl}

\begin{proof}
For the exploration stage, let \(\mathcal E_C(\delta)\) and \(\mathcal E_T(\delta)\) be the online cost and transition confidence events defined at the beginning of Appendix~\ref{app:cstrl-subopt}, applied to the single-rollout estimates constructed by Algorithm~\ref{alg:PFCSTRL}.
For the planning stage, let \(\mathcal E_{\mathrm{fix}}(\delta)\) be the confidence event defined at the beginning of Appendix~\ref{app:pfcstrl-subopt}.
We define the joint confidence event
$\mathcal E_{C,T,\mathrm{fix}}(\delta)
:=
\mathcal E_C(\delta/3)
\cap
\mathcal E_T(\delta/3)
\cap
\mathcal E_{\mathrm{fix}}(\delta/3).$
By the confidence guarantees established in Appendices~\ref{app:cstrl-subopt} and~\ref{app:pfcstrl-subopt}, \(\mathcal E_{C,T,\mathrm{fix}}(\delta)\) holds with probability at least \(1-\delta\).
Throughout the proof, we work on this event.
In particular, we have
$|\widehat C^t-C|\le\Xi^t,
\,
P\in\mathcal P^t,
\,
\forall t\in[T-1],$
and
$|\widehat C-C|\le\Xi,
\,
P\in\mathcal P.$

We first consider the exploration stage.
Algorithm~\ref{alg:PFCSTRL} constructs its initial policy as in Algorithm~\ref{alg:CSTRL}.
The initial policy argument in the proof of Theorem~\ref{thm:safety}
therefore gives
$C^\top\theta^{P,\pi^1}\le\boldsymbol\tau$.
For each \(t\in[T-1]\), the exploration reward in Algorithm~\ref{alg:PFCSTRL} affects only the update of \(\widehat\theta^{t+1}\).
Therefore, the argument in the proof of Theorem~\ref{thm:safety}, applied with the online confidence bounds, gives
$C^\top\theta^{P,\pi^t}
\le
\boldsymbol\tau,
\,
\forall t\in[T].$

We next consider the planning stage.
For every $k\in[K]$, the main planning update uses \texttt{SafeProj} with the fixed quantities $(\widehat C,\Xi,\mathcal P)$.
On $\mathcal E_{C,T,\mathrm{fix}}(\delta)$, $|\widehat C-C|\le\Xi$ and $P\in\mathcal P$ hold at every planning iteration.
The argument in the proof of Theorem~\ref{thm:safety} therefore applies to each main planning policy.
We obtain
$C^\top\theta^{P,\pi^{k+1}}
\le
\boldsymbol\tau,
\,
\forall k\in[K].$

Finally, we show that the returned policy is safe.
Since \(\theta^{P,\pi^{k+1}}\in\Theta_{\mathrm{safe}}\) for every \(k\in[K]\) and \(\Theta_{\mathrm{safe}}\) is convex, the returned occupancy measure satisfies
$\bar\theta
=
\frac1K\sum_{k=1}^K
\theta^{P,\pi^{k+1}}
\in
\Theta_{\mathrm{safe}}.$
Let \(\widehat\pi=\pi^{\bar\theta}\) be the induced Markov policy.
By the correspondence between policies and occupancy measures established in Appendix~\ref{app:safe_value_polytope}, we have \(\theta^{P,\widehat\pi}=\bar\theta\), and therefore \(\mathrm{Violation}(\widehat\pi)=0\).
When \(d=1\), these inequalities guarantee that each constraint is satisfied at the time step to which it applies.
When \(d=H\), they guarantee that each cumulative constraint is satisfied over the entire episode.
This completes the proof.
\end{proof}

\section{Proofs for the Unknown Constraint Cost Extensions}
\label{app:unknown-gamma}

\begin{lemma}
\label{lem:uc-estimation}
Suppose Assumptions~\ref{ass:slater} and~\ref{ass:unknown-gamma} hold and $T>1+8d^2\ln(4nT/\delta)/(\varrho^\dagger)^2$.
Then, with probability at least $1-\delta/2$, we have $t_s\le1+8d^2\ln(4nT/\delta)/(\varrho^\dagger)^2<T$, $\gamma_j\le\widehat\gamma_j\le(\gamma_j+\tau_j)/2<\tau_j$ for every $j\in[n]$, and $\widehat{\phi}_{\mathrm{safe}}\le4\phi_{\mathrm{safe}}$.
\end{lemma}

\begin{proof}
Recall the empirical costs \(\bar\gamma_j^t\), the stopping time \(t_s\), and the estimate \(\widehat{\boldsymbol\gamma}\) defined in Section~\ref{sec:unknown-gamma}.
By Assumption~\ref{ass:slater}, we have $\varrho^\dagger>0$.
For each $j\in[n]$, the total cost observed in one episode belongs to $[0,d]$.

By Hoeffding's inequality and a union bound over $j\in[n]$ and $t\in[T]$, with probability at least $1-\delta/2$, we have
\begin{align}
|\bar\gamma_j^t-\gamma_j|
\le
d\sqrt{\frac{\ln(4nT/\delta)}{2t}}
\quad
\text{for all } j\in[n] \text{ and } t\in[T].
\label{eq:uc-concentration}
\end{align}
We work on this event throughout the proof.

To bound $t_s$, we may assume $t_s\ge2$, since the stated bound is immediate otherwise.
By the definition of $t_s$, the stopping rule is not triggered at episode $t_s-1$.
Then, for some $j\in[n]$, we have $\tau_j-\bar\gamma_j^{t_s-1}<3d\sqrt{\frac{\ln(4nT/\delta)}{2(t_s-1)}}$.
Using~\eqref{eq:uc-concentration} at $t=t_s-1$, we obtain
\[
\varrho^\dagger
\le
\tau_j-\gamma_j
\le
\tau_j-\bar\gamma_j^{t_s-1}
+
d\sqrt{\frac{\ln(4nT/\delta)}{2(t_s-1)}}
<
4d\sqrt{\frac{\ln(4nT/\delta)}{2(t_s-1)}}.
\]
Solving this inequality for $t_s-1$ yields $t_s\le1+8d^2\ln(4nT/\delta)/(\varrho^\dagger)^2$.
By the assumed lower bound on $T$, this bound implies $t_s<T$.
The stopping condition therefore holds at $t_s$.

By the definition of $\widehat\gamma_j$, the stopping condition at $t_s$, and~\eqref{eq:uc-concentration} at $t=t_s$, we have
\[
\gamma_j
\le
\widehat\gamma_j
=
\frac12\Biggl(\bar\gamma_j^{t_s}-d\sqrt{\frac{\ln(4nT/\delta)}{2t_s}}\Biggr)
+
\frac12\Biggl(\bar\gamma_j^{t_s}+3d\sqrt{\frac{\ln(4nT/\delta)}{2t_s}}\Biggr)
\le
\frac{\gamma_j+\tau_j}{2}.
\]
This implies $\tau_j-\widehat{\gamma}_j\ge(\tau_j-\gamma_j)/2$, and therefore $\widehat{\phi}_{\mathrm{safe}}\le4\phi_{\mathrm{safe}}$.
\end{proof}

The proofs of Theorems~\ref{thm:subopt}-\ref{thm:safety-pfcstrl} use $\boldsymbol\gamma$ only as an upper bound on $C^\top\theta^{P,\pi^\dagger}$ that is strictly below $\boldsymbol\tau$.
They therefore remain valid with $\widehat{\boldsymbol\gamma}$ in place of $\boldsymbol\gamma$ and $\widehat{\phi}_{\mathrm{safe}}$ in place of $\phi_{\mathrm{safe}}$ whenever $\boldsymbol\gamma\le\widehat{\boldsymbol\gamma}<\boldsymbol\tau$.

\subsection{Proof of Theorem~\ref{thm:unknown-gamma-subopt}}

\begin{proof}
By Lemma~\ref{lem:uc-estimation}, with probability at least $1-\delta/2$, we have $t_s<T$, $\boldsymbol\gamma\le\widehat{\boldsymbol\gamma}<\boldsymbol\tau$, and $\widehat{\phi}_{\mathrm{safe}}\le4\phi_{\mathrm{safe}}$.
Conditioned on any estimation-phase history in this event, $t_s$ and $\widehat{\boldsymbol\gamma}$ are fixed, and the proof of Theorem~\ref{thm:subopt} applies to the remaining $T-t_s$ episodes with $\widehat{\boldsymbol\gamma}$ in place of $\boldsymbol\gamma$ and $\delta$ replaced by $\delta/2$.
This gives the bound in Theorem~\ref{thm:unknown-gamma-subopt} with conditional probability at least $1-\delta/2$, and a union bound completes the proof.
\end{proof}

\subsection{Proof of Theorem~\ref{thm:unknown-gamma-safety}}

\begin{proof}
By Lemma~\ref{lem:uc-estimation}, with probability at least $1-\delta/2$, we have $t_s<T$ and $\boldsymbol\gamma\le\widehat{\boldsymbol\gamma}<\boldsymbol\tau$.
Every policy executed during the estimation phase is $\pi^\dagger$ and is therefore safe.
Conditioned on any estimation-phase history in this event, $t_s$ and $\widehat{\boldsymbol\gamma}$ are fixed, and the proof of Theorem~\ref{thm:safety} applies to the remaining $T-t_s$ episodes with $\widehat{\boldsymbol\gamma}$ in place of $\boldsymbol\gamma$ and $\delta$ replaced by $\delta/2$.
Thus, with conditional probability at least $1-\delta/2$, every policy executed in \texttt{CSTRL} is safe and $\mathrm{Violation}(\widehat\pi)=0$, and a union bound completes the proof.
\end{proof}

\subsection{Proof of Theorem~\ref{thm:pfcstrl-uc-subopt}}

\begin{proof}
By Lemma~\ref{lem:uc-estimation}, with probability at least $1-\delta/2$, we have $t_s<T$, $\boldsymbol\gamma\le\widehat{\boldsymbol\gamma}<\boldsymbol\tau$, and $\widehat{\phi}_{\mathrm{safe}}\le4\phi_{\mathrm{safe}}$.
Conditioned on any estimation-phase history in this event, $t_s$ and $\widehat{\boldsymbol\gamma}$ are fixed, and the proof of Theorem~\ref{thm:pfcstrl_subopt} applies to the remaining $T-t_s$ exploration episodes and $K$ planning iterations with $\widehat{\boldsymbol\gamma}$ in place of $\boldsymbol\gamma$ and $\delta$ replaced by $\delta/2$.
This gives the bound in Theorem~\ref{thm:pfcstrl-uc-subopt} with conditional probability at least $1-\delta/2$, and a union bound completes the proof.
\end{proof}

\subsection{Proof of Theorem~\ref{thm:pfcstrl-uc-safety}}

\begin{proof}
By Lemma~\ref{lem:uc-estimation}, with probability at least $1-\delta/2$, we have $t_s<T$ and $\boldsymbol\gamma\le\widehat{\boldsymbol\gamma}<\boldsymbol\tau$.
Every policy executed during the estimation phase is $\pi^\dagger$ and is therefore safe.
Conditioned on any estimation-phase history in this event, $t_s$ and $\widehat{\boldsymbol\gamma}$ are fixed, and the proof of Theorem~\ref{thm:safety-pfcstrl} applies with $\widehat{\boldsymbol\gamma}$ in place of $\boldsymbol\gamma$ and $\delta$ replaced by $\delta/2$.
Thus, with conditional probability at least $1-\delta/2$, every policy executed during preference-free exploration and every main policy generated during planning is safe, and $\mathrm{Violation}(\widehat{\pi})=0$.
A union bound completes the proof.
\end{proof}

\section{Proofs for the Unknown Safe Policy Extensions}
\label{app:us}
We give the procedure for \texttt{CSTRL-US} in Algorithm~\ref{alg:cstrl_us}.
\begin{algorithm}[htbp]
\caption{\texttt{CSTRL-US}: \texttt{CSTRL} with an Unknown Safe Policy}
\label{alg:cstrl_us}
\begin{algorithmic}[1]
\Require $\boldsymbol{\lambda}$, $\mu$, $\mathtt{Alg}^P$, $\mathtt{Alg}^D$
\State Set $t=0$
\While{$t<t_s$}
\State Set $t=t+1$
\State Obtain $\pi^t$ from $\mathtt{Alg}^P$ and $\phi^t\in\Delta_n$ from $\mathtt{Alg}^D$
\State Execute $\pi^t$ and observe $\{(s_h^t,a_h^t,\{c_{j,h}^t\}_{j=1}^n)\}_{h=1}^H$
\State Execute $\mathtt{Alg}^P$ with $\{(s_h^t,a_h^t,\ell_h^t)\}_{h=1}^H$
\State Execute $\mathtt{Alg}^D$ with $\varrho^t$
\EndWhile
\State Compute
$\widehat{\varrho}:=
\frac{1}{t_s}
\min_{j\in[n]}
\sum_{t=1}^{t_s}\varrho_j^t
-\frac{2d}{t_s}
\sqrt{2t_s\ln\frac{24nt_s^2}{\delta}}$
\State Let $\widehat{\pi}^{\dagger}$ select $\pi^t$ w.p. $1/t_s$ for each $t\in[t_s]$
\State Set
$\widehat{\boldsymbol{\gamma}} :=(\widehat\gamma_1,\ldots,\widehat\gamma_n)$, where $\widehat\gamma_j:=\tau_j-\widehat\varrho$ for all $j\in[n]$
\State Run \texttt{CSTRL} for $T-t_s$ episodes using
$\widehat\pi^\dagger$ and $\widehat{\boldsymbol{\gamma}}$
\State {\bfseries Return:} $\bar\theta$
\end{algorithmic}
\end{algorithm}

\subsection{Lemmas}

\begin{lemma}
\label{lem:us_martingale}
Let $\mathcal E_M(\delta)$ denote the event that the following inequalities hold simultaneously for every $t\ge 1$.
\begin{align}
&\left|
\sum_{k=1}^t
\sum_{j=1}^n
\phi_j^k
\left[
\sum_{h=1}^H c_{j,h}^k
-
\bigl(C^\top\theta^{P,\pi^k}\bigr)_j
\right]
\right|
\le
d\sqrt{
2t\ln\frac{8nt^2}{\delta}
},
\label{eq:us_weighted_concentration}
\\
&\left|
\sum_{k=1}^t
\left[
\sum_{h=1}^H c_{j,h}^k
-
\bigl(C^\top\theta^{P,\pi^k}\bigr)_j
\right]
\right|
\le
d\sqrt{
2t\ln\frac{8nt^2}{\delta}
},
\quad
\forall j\in[n].
\label{eq:us_coordinate_concentration}
\end{align}
Then $\Pr(\mathcal E_M(\delta))\ge1-\delta$.
\end{lemma}

\begin{proof}
Let $(\mathcal F_k)_{k\ge 0}$ be the filtration generated by the history up to the end of episode $k$.
For each $k\ge1$, we define
$X^k
:=
\sum_{j=1}^n
\phi_j^k
[
\sum_{h=1}^H c_{j,h}^k
-
(C^\top\theta^{P,\pi^k})_j].$
Since $\phi^k\in\Delta_n$ and $\pi^k$ are $\mathcal F_{k-1}$-measurable, the law of iterated expectations and the observation model give
$\mathbb E[
\sum_{h=1}^H c_{j,h}^k|
\mathcal F_{k-1}]
=
(C^\top\theta^{P,\pi^k})_j.$
Moreover, since $c_{j,h}^k\in[0,1]$ and $\phi^k$ lies in the simplex, we have $|X^k|\le d$ almost surely.
For each fixed $t\ge1$, the Azuma-Hoeffding inequality gives
$\Pr(
|\sum_{k=1}^t X^k|
>
d\sqrt{2t\ln\frac{8nt^2}{\delta}})
\le
\frac{\delta}{4nt^2}.$
Taking a union bound over all $t\ge1$, we conclude that~\eqref{eq:us_weighted_concentration} holds simultaneously for every $t\ge1$ with probability at least
$1-\sum_{t=1}^{\infty}\frac{\delta}{4nt^2}
\ge
1-\frac{\delta}{2n},$
where we use $\sum_{t=1}^{\infty}t^{-2}\le2$.

Next, fix $j\in[n]$ and define
$Y_j^k
:=
\sum_{h=1}^H c_{j,h}^k
-
(C^\top\theta^{P,\pi^k})_j.$
By the same argument, $(Y_j^k)_{k\ge1}$ is a martingale difference sequence with respect to $(\mathcal F_k)_{k\ge0}$, and $|Y_j^k|\le d$.
For each fixed $t\ge1$, applying the Azuma-Hoeffding inequality to $(Y_j^k)_{k\ge1}$ gives
$\Pr(
|\sum_{k=1}^t Y_j^k|
>
d\sqrt{2t\ln\frac{8nt^2}{\delta}})
\le
\frac{\delta}{4nt^2}.$
Taking a union bound over all \(j\in[n]\) and \(t\ge1\) shows that~\eqref{eq:us_coordinate_concentration} holds simultaneously for every \(j\in[n]\) and \(t\ge1\) with probability at least
$1-
\sum_{j=1}^n
\sum_{t=1}^{\infty}
\frac{\delta}{4nt^2}
\ge
1-\frac{\delta}{2},$
where we use \(\sum_{t=1}^{\infty}t^{-2}\le2\).

By the definition of $\mathcal E_M(\delta)$, the preceding probability bounds and a union bound give $\Pr(\mathcal E_M(\delta))\ge1-\delta$.
This completes the proof.
\end{proof}

\begin{lemma}
\label{lem:us_empirical_surplus}
Let $\mathcal E_P(\delta)$ be the event on which the regret bound of $\mathtt{Alg}^P$ in Section~\ref{sec:cstrl-us} holds.
On $\mathcal E_P(\delta/3)\cap\mathcal E_M(\delta/3)$, for every $t\ge 1$, we have
\begin{equation}
\min_{j\in[n]}
\sum_{k=1}^t\varrho_j^k
\ge
t\varrho
-
\mathcal O\Bigl(H^2|\mathcal S|\sqrt{|\mathcal A|t\ln(2H|\mathcal S||\mathcal A|t/\delta)}+d\sqrt{nt}\Bigr)
-
d\sqrt{
2t\ln\frac{24nt^2}{\delta}
},
\label{eq:us_empirical_surplus}
\end{equation}
where the $\mathcal O(\cdot)$ term is the sum of the regret bounds of $\mathtt{Alg}^P$ and $\mathtt{Alg}^D$.
\end{lemma}

\begin{proof}
Fix any \(t\ge1\) and \(\pi\in\Pi\).
For each episode \(k\), we define \(\bar\varrho^k\in\mathbb R^n\) by
\(\bar\varrho_j^k
:=
\tau_j-(C^\top\theta^{P,\pi^k})_j\) for every \(j\in[n]\).
By the definition of $\bar\ell^k$ in Section~\ref{sec:cstrl-us} and $\phi^k\in\Delta_n$, we have $\langle\phi^k,\bar\varrho^k\rangle=-\bar\ell^k(\pi^k)$ and $-\bar\ell^k(\pi)\ge\min_{j\in[n]}\{\tau_j-(C^\top\theta^{P,\pi})_j\}$.
The primal regret bound on $\mathcal E_P(\delta/3)$ therefore gives
\[
\sum_{k=1}^{t}\langle \phi^k,\bar\varrho^k\rangle
\ge
t\min_{j\in[n]}\{\tau_j-(C^\top\theta^{P,\pi})_j\}
-\mathcal O\bigl(H^2|\mathcal S|\sqrt{|\mathcal A|t\ln(2H|\mathcal S||\mathcal A|t/\delta)}\bigr).
\]
Since $\pi\in\Pi$ is arbitrary, the definition of $\varrho$ yields
\begin{equation}
\sum_{k=1}^{t}\langle \phi^k,\bar\varrho^k\rangle
\ge
t\varrho-\mathcal O\bigl(H^2|\mathcal S|\sqrt{|\mathcal A|t\ln(2H|\mathcal S||\mathcal A|t/\delta)}\bigr).
\label{eq:us_expected_dual_lb}
\end{equation}
Next, on $\mathcal E_M(\delta/3)$, Lemma~\ref{lem:us_martingale} gives
\begin{equation}
\sum_{k=1}^{t}\langle \phi^k,\varrho^k\rangle
\ge
\sum_{k=1}^{t}\langle \phi^k,\bar\varrho^k\rangle
-d\sqrt{2t\ln\tfrac{24nt^2}{\delta}}.
\label{eq:us_realized_dual_lb}
\end{equation}
Combining~\eqref{eq:us_expected_dual_lb} and~\eqref{eq:us_realized_dual_lb} gives
\begin{equation}
\sum_{k=1}^{t}\langle \phi^k,\varrho^k\rangle
\ge
t\varrho
-\mathcal O\bigl(H^2|\mathcal S|\sqrt{|\mathcal A|t\ln(2H|\mathcal S||\mathcal A|t/\delta)}\bigr)
-d\sqrt{2t\ln\tfrac{24nt^2}{\delta}}.
\label{eq:us_dual_progress_lb}
\end{equation}
We next derive a lower bound on $\min_{j\in[n]}\sum_{k=1}^t\varrho_j^k$.
By the dual regret bound and linearity over the simplex, we have
\[
\sum_{k=1}^{t}\langle \phi^k,\varrho^k\rangle
\le
\min_{\phi\in\Delta_n}
\sum_{k=1}^{t}\langle \phi,\varrho^k\rangle
+
\mathcal O(d\sqrt{nt})
=
\min_{j\in[n]}
\sum_{k=1}^{t}\varrho_j^k
+
\mathcal O(d\sqrt{nt}).
\]
Combining this inequality with~\eqref{eq:us_dual_progress_lb} gives~\eqref{eq:us_empirical_surplus}.
This completes the proof.
\end{proof}

\begin{lemma}
\label{lem:us_stop_time}
Suppose that $T$ satisfies the condition in Section~\ref{sec:cstrl-us}.
On $\mathcal E_P(\delta/3)\cap\mathcal E_M(\delta/3)$, we have $t_s<T$ and $t_s=\widetilde{\mathcal O}((H^4|\mathcal S|^2|\mathcal A|+nd^2)/\varrho^2)$.
\end{lemma}

\begin{proof}
Let $t\ge1$ be the smallest integer for which $t\varrho$ is at least twice the stopping threshold.
The condition on $T$ in Section~\ref{sec:cstrl-us} means that $T>t$.
By Lemma~\ref{lem:us_empirical_surplus}, we have
\[
\begin{aligned}
\min_{j\in[n]}\sum_{k=1}^t\varrho_j^k
&\ge
t\varrho
-\mathcal O\Bigl(H^2|\mathcal S|\sqrt{|\mathcal A|t\ln(2H|\mathcal S||\mathcal A|t/\delta)}+d\sqrt{nt}\Bigr)
-d\sqrt{2t\ln\frac{24nt^2}{\delta}}
\\
&\ge
t\varrho
-\mathcal O\Bigl((H^2|\mathcal S|\sqrt{|\mathcal A|}+d\sqrt n)\sqrt{t\ln(2nH|\mathcal S||\mathcal A|t/\delta)}\Bigr)
\\
&\ge
\mathcal O\Bigl((H^2|\mathcal S|\sqrt{|\mathcal A|}+d\sqrt n)\sqrt{t\ln(2nH|\mathcal S||\mathcal A|t/\delta)}\Bigr),
\end{aligned}
\]
where the $\mathcal O(\cdot)$ terms in the last two lines denote the stopping threshold at episode $t$.
The second inequality holds since the constant hidden in the stopping threshold is chosen sufficiently large, and the last inequality follows from the choice of $t$.
Therefore, the inequality in the definition of $t_s$ holds at episode $t$, which gives $t_s\le t<T$.
If $t\ge2$, the minimality of $t$ gives
\[
\sqrt{t-1}\,\varrho
\le
\mathcal O\Bigl((H^2|\mathcal S|\sqrt{|\mathcal A|}+d\sqrt n)\sqrt{\ln(2nH|\mathcal S||\mathcal A|t/\delta)}\Bigr).
\]
Solving this inequality for $t$ yields $t=\widetilde{\mathcal O}((H^4|\mathcal S|^2|\mathcal A|+nd^2)/\varrho^2)$, which also holds trivially when $t=1$.
Since $t_s\le t$, this completes the proof.
\end{proof}

\begin{lemma}
\label{lem:us_delta_hat}
Suppose that $T$ satisfies the condition in Section~\ref{sec:cstrl-us}.
On $\mathcal E_P(\delta/3)\cap\mathcal E_M(\delta/3)$, the estimate $\widehat{\varrho}$ satisfies
\begin{equation}
\widehat{\varrho}
\le
\min_{j\in[n]}
\left\{
\tau_j
-
\bigl(C^\top
\theta^{P,\widehat\pi^\dagger}\bigr)_j
\right\}
\le
\varrho.
\label{eq:us_delta_sandwich_appendix}
\end{equation}
Moreover, on the same event, 
we have $\widehat{\varrho}\ge \varrho/2>0$ and
$C^\top\theta^{P,\widehat\pi^\dagger}
\le
\widehat{\boldsymbol\gamma}
<
\boldsymbol\tau$.
\end{lemma}

\begin{proof}
By definition, we have
$\widehat{\varrho}
=
\frac{1}{t_s}
\min_{j\in[n]}
\sum_{t=1}^{t_s}\varrho_j^t
-
\frac{2d}{t_s}
\sqrt{
2t_s\ln\frac{24nt_s^2}{\delta}}.$
Applying~\eqref{eq:us_coordinate_concentration} on $\mathcal E_M(\delta/3)$ to each \(j\in[n]\), we obtain
$\sum_{t=1}^{t_s}
(\sum_{h=1}^H c_{j,h}^t-\tau_j)
\ge
\sum_{t=1}^{t_s}
[(C^\top\theta^{P,\pi^t})_j-\tau_j]
-
d\sqrt{
2t_s\ln\frac{24nt_s^2}{\delta}}.$
Combining this inequality with the definition of $\widehat\varrho$ gives
$\widehat{\varrho}
\le
\min_{j\in[n]}
\{\tau_j
-
\frac{1}{t_s}
\sum_{t=1}^{t_s}
(C^\top\theta^{P,\pi^t})_j\}.$
Since $\widehat{\pi}^\dagger$ selects $\pi^t$ with probability $1/t_s$ for each $t\in[t_s]$, its occupancy measure satisfies
$\theta^{P,\widehat\pi^\dagger}
=
\frac{1}{t_s}
\sum_{t=1}^{t_s}
\theta^{P,\pi^t}.$
Therefore, we obtain
$\widehat{\varrho}
\le
\min_{j\in[n]}
\{
\tau_j-
(C^\top
\theta^{P,\widehat\pi^\dagger})_j\}.$
The second inequality in~\eqref{eq:us_delta_sandwich_appendix} follows immediately from the definition of $\varrho$.
It remains to prove that $\widehat{\varrho}\ge\varrho/2$.
By Lemma~\ref{lem:us_stop_time}, we have $t_s<T$, and thus the inequality in the definition of $t_s$ holds at episode $t_s$.
Lemma~\ref{lem:us_empirical_surplus} and the definition of $\widehat\varrho$ therefore give
\[
\begin{aligned}
t_s\varrho
&\le
\min_{j\in[n]}\sum_{k=1}^{t_s}\varrho_j^k
+\mathcal O\Bigl(H^2|\mathcal S|\sqrt{|\mathcal A|t_s\ln(2H|\mathcal S||\mathcal A|t_s/\delta)}+d\sqrt{nt_s}\Bigr)
+d\sqrt{2t_s\ln\frac{24nt_s^2}{\delta}}
\\
&\le
2\min_{j\in[n]}\sum_{k=1}^{t_s}\varrho_j^k
-4d\sqrt{2t_s\ln\frac{24nt_s^2}{\delta}}
=
2t_s\widehat\varrho.
\end{aligned}
\]
The second inequality follows from the stopping condition at $t_s$, since the constant hidden in the stopping threshold is chosen sufficiently large.
Dividing both sides by $2t_s$ yields $\widehat{\varrho}\ge\varrho/2$.
Since $\widehat\gamma_j=\tau_j-\widehat\varrho$ for every \(j\in[n]\), the first inequality in~\eqref{eq:us_delta_sandwich_appendix} gives
$C^\top\theta^{P,\widehat\pi^\dagger}
\le
\widehat{\boldsymbol\gamma}$.
Moreover, $\widehat{\varrho}>0$ implies $\widehat{\boldsymbol\gamma}<\boldsymbol\tau$.
This completes the proof.
\end{proof}

\subsection{Proof of Theorem~\ref{thm:us_subopt}}
\begin{proof}
Consider the event
$\mathcal E
:=
\mathcal E_P(\delta/3)
\cap
\mathcal E_M(\delta/3)$.
Since $\mathtt{Alg}^P$ is run with $\delta/3$ in place of $\delta$, a union bound gives $\Pr(\mathcal E)\ge1-2\delta/3$.
By the condition on $T$ in Section~\ref{sec:cstrl-us} and Lemma~\ref{lem:us_stop_time}, we have $T>t_s$ on $\mathcal E$.
Thus, $T-t_s\ge1$.
On the same event, Lemma~\ref{lem:us_delta_hat} yields \(\widehat{\varrho}\ge\varrho/2>0\) and
$C^\top\theta^{P,\widehat\pi^\dagger}
\le
\widehat{\boldsymbol\gamma}
<
\boldsymbol\tau.$
Moreover, since \(\widehat\gamma_j=\tau_j-\widehat{\varrho}\) for every \(j\in[n]\), we have
$\widehat\phi_{\mathrm{safe}}
=
\max_{j\in[n]}
\frac{d^2}{(\tau_j-\widehat\gamma_j)^2}
=
\frac{d^2}{\widehat{\varrho}^2}
\le
\frac{4d^2}{\varrho^2}.$
Fix any estimation-phase history on which \(\mathcal E\) holds.
Then the pair \((\widehat\pi^\dagger,\widehat{\boldsymbol\gamma})\) is fixed.
After episode \(t_s\), Algorithm~\ref{alg:cstrl_us} runs \texttt{CSTRL} for \(T-t_s\) episodes using this pair.
The proof of Theorem~\ref{thm:subopt} therefore applies with $\delta$ replaced by $\delta/3$, $T$ by $T-t_s$, and $\phi_{\mathrm{safe}}$ by $\widehat{\phi}_{\mathrm{safe}}$.
This gives the bound in Theorem~\ref{thm:us_subopt} with conditional probability at least $1-\delta/3$.
Combining this conditional guarantee with the estimation event gives an overall probability of at least $1-\delta$.
This proves Theorem~\ref{thm:us_subopt}.
\end{proof}

\subsection{Proof of Theorem~\ref{thm:us_safety}}

\begin{proof}
Let
$\mathcal E
:=
\mathcal E_P(\delta/3)
\cap
\mathcal E_M(\delta/3)
\cap
\mathcal E_{R,C,T}(\delta/3)$.
Since $\mathtt{Alg}^P$ is run with $\delta/3$ in place of $\delta$, the first two events each hold with probability at least $1-\delta/3$.
For any fixed history of the estimation phase, the subsequent \texttt{CSTRL} run uses $\delta/3$ in place of $\delta$, and thus $\mathcal E_{R,C,T}(\delta/3)$ holds with conditional probability at least $1-\delta/3$.
A union bound therefore gives $\Pr(\mathcal E)\ge1-\delta$.
It thus suffices to prove the result on $\mathcal E$.

By the condition on $T$ in Section~\ref{sec:cstrl-us} and Lemma~\ref{lem:us_stop_time}, we have $T>t_s$ on $\mathcal E$.
On the same event, Lemma~\ref{lem:us_delta_hat} gives
\(C^\top\theta^{P,\widehat\pi^\dagger}
\le
\widehat{\boldsymbol\gamma}
<
\boldsymbol\tau\).
We now bound the cumulative violation over all episodes.
During the estimation phase, the policy executed in each episode may be unsafe.
Since every constraint value is at most \(d\) and \(\tau_j\ge 0\) for every \(j\in[n]\), we have \(\mathrm{Violation}(\pi^t)\le d\).
Therefore, we have \(\sum_{t=1}^{t_s}\mathrm{Violation}(\pi^t) \le d t_s\).

After episode \(t_s\), the algorithm runs \texttt{CSTRL} for the remaining $T-t_s$ episodes with \(\widehat\pi^\dagger\) and \(\widehat{\boldsymbol\gamma}\).
Given the estimation-phase history, the pair \((\widehat\pi^\dagger,\widehat{\boldsymbol\gamma})\) is fixed.
On \(\mathcal E\), the argument in the proof of Theorem~\ref{thm:safety} therefore applies with \(\widehat{\boldsymbol\gamma}\) as a conservative upper bound on the costs of \(\widehat\pi^\dagger\).
Consequently, every policy executed in \texttt{CSTRL} is safe, which gives $\sum_{t=t_s+1}^{T}\mathrm{Violation}(\pi^t)=0$.
Combining the preceding bounds gives
$\sum_{t=1}^{T}\mathrm{Violation}(\pi^t)
=
\sum_{t=1}^{t_s}\mathrm{Violation}(\pi^t)
\le
d t_s.$
Moreover, Lemma~\ref{lem:us_stop_time} shows that $t_s=\widetilde{\mathcal O}((H^4|\mathcal S|^2|\mathcal A|+nd^2)/\varrho^2)$ on $\mathcal E$.
This upper bound is independent of $T$.

Finally, we have \(\theta^{P,\pi^t}\in\Theta_{\mathrm{safe}}\) for every \(t\in[t_s+1,T]\).
Since \(\Theta_{\mathrm{safe}}\) is convex, the $\bar\theta$ returned by \texttt{CSTRL} also belongs to \(\Theta_{\mathrm{safe}}\).
Let \(\widehat\pi=\pi^{\bar\theta}\).
By the correspondence between policies and occupancy measures established in Appendix~\ref{app:safe_value_polytope}, we have \(\theta^{P,\widehat\pi}=\bar\theta\).
Therefore, we obtain \(\mathrm{Violation}(\widehat\pi)=0\).
This completes the proof.
\end{proof}

\subsection{Proof of Theorem~\ref{thm:pfcstrl-us-subopt}}

\begin{proof}
Let
$\mathcal E
:=
\mathcal E_P(\delta/3)
\cap
\mathcal E_M(\delta/3)$.
Since $\mathtt{Alg}^P$ is run with $\delta/3$ in place of $\delta$, a union bound gives $\Pr(\mathcal E)\ge1-2\delta/3$.
By the condition on $T$ in Section~\ref{sec:cstrl-us} and Lemma~\ref{lem:us_stop_time}, we have $T>t_s$ on $\mathcal E$.
On the same event, Lemma~\ref{lem:us_delta_hat} gives \(\widehat{\varrho}\ge\varrho/2>0\) and
\(C^\top\theta^{P,\widehat\pi^\dagger}
\le
\widehat{\boldsymbol\gamma}
<
\boldsymbol\tau\).
Fix any estimation-phase history on which \(\mathcal E\) holds.
Then the pair \((\widehat\pi^\dagger,\widehat{\boldsymbol\gamma})\) is fixed.
The algorithm runs Algorithms~\ref{alg:PFCSTRL} and~\ref{alg:pfcstrl-plan-rf} for \(T-t_s\) exploration episodes and \(K\) planning iterations, respectively, using this pair.
The proof of Theorem~\ref{thm:pfcstrl_subopt} therefore applies with $\delta$ replaced by $\delta/3$, $T$ by $T-t_s$, and $\phi_{\mathrm{safe}}$ by $\widehat{\phi}_{\mathrm{safe}}=d^2/\widehat{\varrho}^2\le 4d^2/\varrho^2$.
This gives the bound in Theorem~\ref{thm:pfcstrl-us-subopt} with conditional probability at least \(1-\delta/3\).
Combining this conditional guarantee with the estimation event gives an overall probability of at least \(1-\delta\).
This proves Theorem~\ref{thm:pfcstrl-us-subopt}.
\end{proof}

\subsection{Proof of Theorem~\ref{thm:pfcstrl-us-safety}}

\begin{proof}
Let
$\mathcal E
:=
\mathcal E_P(\delta/3)
\cap
\mathcal E_M(\delta/3)
\cap
\mathcal E_{C,T,\mathrm{fix}}(\delta/3)$.
Since $\mathtt{Alg}^P$ is run with $\delta/3$ in place of $\delta$, the first two events each hold with probability at least $1-\delta/3$.
For any fixed history of the estimation phase for which both events hold, \texttt{PF-CSTRL} is run with $\delta/3$ in place of $\delta$, and thus $\mathcal E_{C,T,\mathrm{fix}}(\delta/3)$ holds with conditional probability at least $1-\delta/3$ for the subsequent $T-t_s$ exploration episodes and $K$ planning iterations.
A union bound therefore gives $\Pr(\mathcal E)\ge1-\delta$.
Lemma~\ref{lem:us_stop_time} gives $t_s=\widetilde{\mathcal O}((H^4|\mathcal S|^2|\mathcal A|+nd^2)/\varrho^2)$ on $\mathcal E$.
Each policy executed during the estimation phase has violation at most \(d\), and therefore
$\sum_{t=1}^{t_s}
\mathrm{Violation}(\pi^t)
\le
d t_s.$
Lemma~\ref{lem:us_delta_hat} further gives
$C^\top\theta^{P,\widehat\pi^\dagger}
\le
\widehat{\boldsymbol\gamma}
<
\boldsymbol\tau.$
Given the estimation-phase history, the pair \((\widehat\pi^\dagger,\widehat{\boldsymbol\gamma})\) is fixed.
On \(\mathcal E\), the proof of Theorem~\ref{thm:safety-pfcstrl} applies with this pair.
Consequently, every policy executed during preference-free exploration and every main policy generated during planning is safe.
In particular, we have
$\sum_{t=t_s+1}^{T}
\mathrm{Violation}(\pi^t)
=
0.$
Combining the violation bounds for the estimation phase and the exploration stage gives
$\sum_{t=1}^{T}
\mathrm{Violation}(\pi^t)
\le
d t_s.$
The returned policy \(\widehat{\pi}\) also satisfies $\mathrm{Violation}(\widehat{\pi}) = 0.$
This proves Theorem~\ref{thm:pfcstrl-us-safety}.
\end{proof}

\section{Other Supporting Lemmas}
\label{app:aux}

\begin{lemma}[\citet{tseng2008accelerated,wei2020online}]
\label{lem:kl-three-point}  
Let $f: \Delta_m \mapsto \mathbb{R}$ be a convex function, and let $D_{\mathrm{KL}}(\boldsymbol{x}\Vert\boldsymbol{y}):=\sum_{i=1}^m x_i \log\frac{x_i}{y_i}$ for $\boldsymbol{x}\in\Delta_m$ and $\boldsymbol{y}\in\Delta_m^o$.
For any $\beta>0$, $\boldsymbol{z} \in \Delta_m$, and $\boldsymbol{y} \in \Delta_m^o$, if
$\boldsymbol{x}^{\mathrm{opt}}
\in
\argmin_{\boldsymbol{x}\in\Delta_m}
\{f(\boldsymbol{x}) + \beta D_{\mathrm{KL}}(\boldsymbol{x}\Vert\boldsymbol y)\},$
then we have $f(\boldsymbol{x}^{\mathrm{opt}}) + \beta D_{\mathrm{KL}} (\boldsymbol{x}^{\mathrm{opt}}\Vert \boldsymbol{y}) \leq f(\boldsymbol{z}) + \beta D_{\mathrm{KL}} (\boldsymbol{z}\Vert \boldsymbol{y}) - \beta D_{\mathrm{KL}} (\boldsymbol{z}\Vert \boldsymbol{x}^{\mathrm{opt}}).$
\end{lemma}

\begin{lemma}[\citet{wei2020online,qiu2023gradient}]
\label{lem:mix} 
Suppose that $\boldsymbol{x}\in\Delta_m$ and $\widetilde x_i = (1-\alpha) x_i + \alpha/m$.
Then for all $\boldsymbol{y} \in\Delta_m$,
we have $D_{\mathrm{KL}}(\boldsymbol{y}\Vert\widetilde {\boldsymbol{x}}) - D_{\mathrm{KL}}(\boldsymbol{y}\Vert\boldsymbol{x}) \leq \alpha \log m,$
and $\|\widetilde {\boldsymbol{x}} - \boldsymbol{x}\|_1 \leq 2\alpha$.
\end{lemma}

\newpage
\bibliography{main}

\end{document}